\documentclass[opre,sglanonrev]{informs4}
\RequirePackage{tgtermes}
\RequirePackage{newtxtext}
\RequirePackage{newtxmath}
\RequirePackage{bm}
\RequirePackage{endnotes}
\OneAndAHalfSpacedXII 

\usepackage{algorithm}
\usepackage{algpseudocode}
\usepackage{tikz}

\usepackage[backend=biber,natbib=true]{biblatex}
\EquationsNumberedThrough    

\TheoremsNumberedThrough     
\ECRepeatTheorems  %

\MANUSCRIPTNO{MOOR-0001-2024.00}
\usepackage{times}
\usepackage{tikz}
\usepackage{ulem}
\usepackage[english]{babel}
\usepackage{amssymb}
\usepackage[margin=1in]{geometry}
\usepackage{verbatim}
\usepackage{lineno}
\usepackage{calligra}
\usepackage{xcolor}
\usepackage{amssymb}
\usepackage{graphicx}
\usepackage[colorlinks=true, allcolors=blue]{hyperref}
\usepackage{verbatim}
\usepackage{ulem}
 \usepackage{cleveref} 
\usepackage{enumitem}
\usepackage{bbm}
\usepackage{csquotes}

\numberwithin{equation}{section} 
\newcommand{\pow}[1]{^{(#1)}}
\newcommand{\lp}{\left(}
\newcommand{\rp}{\right)}
\newcommand{\lc}{\left\{}
\newcommand{\rc}{\right\}}
\newcommand{\lb}{\left[}
\newcommand{\rb}{\right]}
\newcommand{\lv}{\left|}
\newcommand{\rv}{\right|}
\newcommand{\lV}{\left\|}
\newcommand{\rV}{\right\|}
\newcommand{\ldot}{\left.}
\newcommand{\rdot}{\right.}

\newcommand{\E}{\mathrm{E}}
\newcommand{\Var}{\mathrm{Var}}
\newcommand{\Cov}{\mathrm{Cov}}

\newcommand{\datas}{\{x_0,\dots,x_m\}}
\newcommand{\datay}{\{y_0,\dots,y_m\}}

\newcommand{\History}{\mathcal{H}}
\newcommand{\history}{\hbar}

\newcommand{\Cbb}{\mathbb{C}}

\newcommand{\Fbb}{\mathbb{F}}

\newcommand{\Hbb}{\mathbb{H}}
\newcommand{\Ibb}{\mathbb{I}}

\newcommand{\Mbb}{\mathbb{M}}
\newcommand{\Nbb}{\mathbb{N}}

\newcommand{\Pbb}{\mathbb{P}}

\newcommand{\Rbb}{\mathbb{R}}

\newcommand{\Tbb}{\mathbb{T}}

\newcommand{\Xbb}{\mathbb{X}}

\newcommand{\Fcal}{\mathcal{F}}

\newcommand{\Hcal}{\mathcal{H}}

\newcommand{\Mcal}{\mathcal{M}}

\newcommand{\Pcal}{\mathcal{P}}

\newcommand{\Scal}{\mathcal{S}}
\newcommand{\Rcal}{\mathcal{R}}

\newcommand{\Xcal}{\mathcal{X}}
\newcommand{\Ycal}{\mathcal{Y}}

\DeclareMathAlphabet{\mathdutchcal}{U}{dutchcal}{m}{n}

\newcommand{\constant}{\Cbb}

\newcommand{\statespace}{\{1,\dots, d\}}

\newcommand{\indexeddata}{\left\{(X_0,a_0),\dots,(X_m,a_m)\right\}}

\newcommand{\naturalset}{\Nbb}
\newcommand{\hatmls}{\Hat{M}^{(l)}(s,\cdot)}
\newcommand{\mls}{{M}^{(l)}(s,\cdot)}
\newcommand{\tildemls}{\tilde{M}^{(l)}(s,\cdot)}
\newcommand{\tildemlst}{\tilde{M}_{s,t}^{(l)}}

\newcommand{\deltam}{\lVert \Delta_m \rVert}
\newcommand\numberthis{\addtocounter{equation}{1}\tag{\theequation}}

\newcommand{\whiteqed}{\hfill$\square$\par\bigskip}
\renewcommand{\mc}{\mathfrak{M}}
\newcommand{\prob}{\Pbb}
\newcommand{\expec}{\mathbb{E}}
\newcommand{\probl}{\prob}
\newcommand{\initialD}{D_0}

\usepackage{newtxmath}
\newcommand{\indicator}{\vmathbb{1}}

\newcommand{\nrm}[1]{\left\Vert #1 \right\Vert}

\newcommand{\abs}[1]{\left| #1 \right|}

\newcommand{\ben}{\begin{enumerate}}
\newcommand{\een}{\end{enumerate}}
\newcommand{\bit}{\begin{itemize}}
\newcommand{\eit}{\end{itemize}}
\newcommand{\hide}[1]{}

\newcommand{\diag}{\operatorname{diag}}

\newcommand{\gn}{\, | \,}

\newcommand{\etab}{\boldsymbol{\eta}}

\newcommand{\mmrisk}{\mathcal{R}_{m}}

\newcommand{\eps}{\varepsilon}
\newcommand{\vertiii}[1]{{\left\vert\kern-0.25ex\left\vert\kern-0.25ex\left\vert #1 
    \right\vert\kern-0.25ex\right\vert\kern-0.25ex\right\vert}}

\renewcommand{\epsilon}{\eps}

\bibliography{biblio}
\begin{document}

\RUNAUTHOR{Banerjee, Honnappa, and Rao}

\RUNTITLE{Offline Estimation of Controlled Markov Chains}

\TITLE{Offline Estimation of Controlled Markov Chains: Minimaxity and Sample Complexity}

\ARTICLEAUTHORS{%
\AUTHOR{Imon Banerjee\footnote{Work completed while author was at the Department of Statistics, Purdue University}}
\AFF{Department of Industrial Engineering and Management Sciences, Northwestern University \EMAIL{imon.banerjee@northwestern.edu}} 
\AUTHOR{Harsha Honnappa}
\AFF{Edwardson School of Industrial Engineering, Purdue University \EMAIL{honnappa@purdue.edu}}
\AUTHOR{Vinayak Rao}
\AFF{Department of Statistics, Purdue University \EMAIL{varao@purdue.edu}} 
} 

\ABSTRACT{%
   In this work, we study a
   natural nonparametric estimator of the transition probability matrices of a finite controlled Markov chain. We consider an offline setting with a fixed dataset of size $m$, collected using a so-called logging policy. We develop sample complexity bounds for the estimator and establish conditions for minimaxity. Our statistical bounds depend on the logging policy through its mixing properties. We show that achieving a particular statistical risk bound involves a subtle and interesting trade-off between the strength of the mixing properties and the number of samples. We demonstrate the validity of our results under various examples, such as ergodic Markov chains, weakly ergodic inhomogeneous Markov chains, and controlled Markov chains with non-stationary Markov, episodic, and greedy controls. Lastly, we use these sample complexity bounds to establish concomitant ones for offline evaluation of stationary Markov control policies.
}%


\KEYWORDS{reinforcement learning, controlled Markov chains, stochastic processes, policy evaluation, nonparametric statistics} 

\date{\today}

\maketitle

{
\section{Introduction}~\label{sec:intro}
This paper presents probably approximately correct (PAC)-style minimax sample complexity results for statistical estimation of transition matrices of discrete-time, finite-state controlled Markov chains (CMCs)~\citep{borkar1991topics}. 
%
%
%
We model a controlled Markov chain as a discrete-time pair process $\{X_i,a_i\}$ where  $\{a_i\}$ is a sequence of controls and $\{X_i\}$ is the state sequence that, conditioned on $a_i$, follows a Markov transition kernel.
{ In this paper, we answer the following question:

\begin{quote}
\centering
        \textit{What is the minimum number of samples required to estimate the transition matrices of a discrete-time finite-state controlled Markov chain to any given degree of precision?}
\end{quote}
}
We answer this question by showing that a particular nonparametric estimator (see below) of the transition matrix is minimax optimal. The control sequence can be viewed as generated by a `logging' or `behavior' policy. { Assuming} the policy is stationary Markovian, $\{X_i,a_i\}$ is jointly Markovian  \citep[Chapter 2.3]{hernandez1991recurrence}. 
This simplifies the problem to one of estimating the transition kernel of a Markov chain, and opens the door to a number of results under suitable ergodicity and mixing assumptions. 
Some very recent ones include frequentist \citep{wolfer2021statistical} and Bayesian \citep{banerjee2021pac} PAC bounds. 
On the other hand, logged data in OR and machine learning tasks like hospital emergency scheduling \citep{lee_markov_2018}, minimum system entropy explorations  \citep{mutti2022importance}, reward machines \citep{icarte_using_2018}, adversarial Markov games \citep{wang2023foundation} and others are often non-Markovian. 
    The following passage from \cite{mutti2022importance} helps to elucidate the importance of non-Markovian policies:
    \begin{quote}
    ``In this finite-sample formulation non-Markovian strategies are crucial, and we believe they can benefit a significant range of relevant applications. For example, collecting task-specific samples might be costly in some real-world domains, and a pre-trained non-Markovian strategy is essential to guarantee quality exploration even in a single-trial setting [\dots] A non-Markovian strategy could exploit the history of interactions to swiftly identify the structure of the environment, then employing the environment-specific optimal strategy thereafter."
    \end{quote}
    While
    ~\cite{mutti2022importance} focuses on non-Markovian policies in the online setting,~\cite{laroche2022non,laroche23a} observe that non-Markovianity of policies is particularly an issue in the offline setting where logging policies can have an arbitrary structure. However, if the logging policy is non-Markovian, finite sample statistical inference results are quite sparse.
Consequently, an understanding of the sample complexity of estimating the transition kernel for  non-Markovian controls is an important open problem. In this work, all we assume is that the $a_i$'s are adapted and `mixing' in a sense defined in Assumptions \ref{assume:eta-mix} and \ref{assume:eta-exp}.
\subsubsection*{Contributions.}

{
We focus on the following nonparametric estimator of the transition matrices. For any state-control-state (SCS) tuple $(s,l,t)$, define $N_s\pow l$ as the number of visits to the (state, control) pair $(s,l)$, and $N_{s,t}\pow l$ as the number of transitions from state $s$ to state $t$ under $l$. Thus, with $\indicator[\cdot]$ as the indicator function, 
\begin{align}
    N_s^{(l)} & :=\sum_{i=1}^m \indicator[X_i=s,a_i=l]\nonumber\\
    N_{s,t}^{(l)} & :=\sum_{i=1}^m \indicator[X_i=s,X_{i+1}=t,a_{i}=l]. \label{eq:pair-visit-count}
\end{align}
Then our estimator of the transition probability from state $s$ to $t$ conditioned on control $l$, $M_{s,t}^{(l)}\ $, is
\begin{align}
    \Hat{M}_{s,t}^{(l)} := \frac{ N_{s,t}^{(l)} }{ N_s^{(l)}}~\label{eq:lestimate}.
\end{align}
We prove that this estimator is minimax optimal under suitable conditions. }

Although this particular non-parametric estimator of transition matrices has been previously used in the context of model-based RL studies \citep{li2022settling,mannor2005empirical}, its statistical guarantees under non-Markovian controls are not known in the literature. This paper fills this obvious gap. The main contributions of this paper can be summarized as follows.

\begin{enumerate}
\item Our main result (\Cref{thm:minimax}) shows that the non-parametric estimator is minimax optimal if the number of samples is large enough, and identifies an explicit lower bound on the required sample size. Informally, we prove that the sample complexity of estimating the transition matrices in a CMC with $d$ states and $k$ controls is $\Theta(d^2\,k)$ if the CMC is geometrically fast mixing.
As we argue in \Cref{sec:examples}, a geometrically ergodic Markov chain with $d\,k$ states can be thought of as a special case of a $d$-state, $k$-control CMC with  ``stationary'' controls. Thus, our result (\Cref{thm:minimax}) recovers the optimal sample complexity of estimating Markov chains from \cite{wolfer2021statistical} as a special case.

\item 
We prove in \Cref{thm:sample-complexity} that the transition probabilities can be estimated even under a weaker mixing assumption than is required for minimaxity (\Cref{thm:minimax}). However, this involves a trade-off, requiring more samples (roughly $\Omega(d^2k^2)$ in place of $\Omega(d^2k)$) to achieve the same level of estimation error.

\item A useful implication of our sample complexity results is that they yield error bounds for offline policy evaluation (OPE).
Theorem ~\ref{thm:value-thm} evaluates stationary Markov policies from data logged using non Markovian controls. 
The resulting sample complexity recovers minimax optimal rates in the literature, which typically assume Markovian logging policies. 
Furthermore, Theorem \ref{thm:opt-pol1} demonstrates a sample complexity of estimating the optimal policy under sufficient regularity conditions on the model class. We also demonstrate how to use our theory to derive the sample complexity bounds for learning the demand distribution of an inventory control problem in Proposition \ref{prop:ic-control} \citep{goldberg2021survey}.

\end{enumerate}

From a methodological perspective, our analysis reveals two principles that are broadly useful in establishing sample efficiency results for learning CMCs and other controlled stochastic models. First, is a `Goldilocks principle' that no state-control pair must be visited too many or too few times in a single observed sample path. This can be achieved by ensuring that the control sequence is such that the time-to-return to a particular pair is uniformly bounded over the state-control space (see Assumptions \ref{assume:sets} and~\ref{assume:return_time}). 

Second, the effect of history on the probability of under- or over-visiting any part of the state-control space is controlled by the mixing properties of the control sequence, as defined in Assumption~\ref{assume:control-mixing} and Assumption~\ref{assume:control-exp}. 
Roughly speaking, weaker mixing properties imply looser bounds on these probabilities, in turn implying that estimators are possibly sample inefficient. 
The bulk of the existing literature on offline estimation of CMCs focuses on the setting where the control sequence forms a stationary ergodic Markov sequence and, under this condition, the nonparametric estimator is minimax optimal, as implied by~\Cref{thm:minimax} and~\Cref{prop:markov-mdp}. 
However, if the control sequence mixes comparatively slowly (say, polynomially), then~\Cref{thm:sample-complexity} yields a loose sample complexity bound. 

As we prove in Sections \ref{sec:example-stationary} and \ref{sec:exam-Markov},  it is relatively straightforward to verify the geometric mixing properties of the control sequences when the controls are Markovian. 
However, when the controls are non-Markovian, there is no general result to demonstrate geometric mixing. Thus, a practitioner must be cautious of erroneously assuming the logging policy to be Markovian when it is not. 
If the controls are not Markovian, then one needs $\Omega(d^2k^2)$ samples instead of $\Omega(d^2k)$ to control the probability of large estimation errors on the transition probabilities. 

As a final note on the methodological implications, while we focus on finite state-control spaces, we believe that these principles, and our analysis, yield a broad framework for proving sample efficiency results for offline estimation of CMCs, and potentially other controlled stochastic models, under more general  model assumptions. For instance, if one uses a histogram or a density estimator of a transition kernel on continuous state spaces and compact control spaces, then our results are directly applicable, although the optimal sample complexity would depend upon the smoothness properties of the transition function. 

\subsubsection*{Related Literature.}~\label{sec:literature-review}
We divide the review of the literature into three parts. In the first part, we place our results in the context of the existing literature on non-parametric estimation for stochastic processes. In the second and third parts, we relate our sample complexity results to existing relevant ones in the literature on offline RL, and system identification, respectively.

\paragraph{Non-parametric estimation:}  The foundations of non-parametric estimation~\citep{tsybakov2009introduction} of finite ergodic Markov chains were laid by \cite{billingsley1961statistical}. Subsequently, \cite{yakowitz1979nonparametric} presented an important extension to infinite state spaces, with follow-up work on applications to regression \citep{yakowitz1989nonparametric}. There is also extensive literature establishing laws of large numbers (LLNs) \citep{geyer1998markov} and central limit theorems (CLTs) \citep{jones2004markov} for a range of time-homogeneous Markov chains. However, somewhat surprisingly, minimax sample complexity bounds for finite ergodic Markov chains were only established recently in~\cite{wolfer2021statistical}. However, barring some results on LLNs and CLTs \citep{rosenblatt1963some,rosenblatt1964some,dobrushin1956central,dobrushin1956centralII}, results on statistical inference for time-inhomogenous Markov chains remain sparse.  Furthermore, such properties when the controls are stochastic in nature are even less understood. A crucial complication in our setting is that the state-control pair process need not be Markovian, complicating the application of existing results. Nonetheless, we recover rates similar to those of \cite{wolfer2021statistical} as a special case in \cref{sec:example-stationary}, demonstrating the generality of our results.

\paragraph{System Identification:} The problem in our paper is analogous to system identification in optimal control, \citep{vidyasagar2006learning,ljung2010perspectives,tangirala2018principles} where the parameters to be estimated are the transition matrices. 
There is a growing body of work that revolves around so-called ``active learning" for system identification \citep{mania2020active,chin2020active,raychaudhuri1996active}.  
However, in our work, the logging policy does not necessarily aid active learning.
Furthermore, the former settings are online in nature, and system identification in the offline setting traditionally does not involve a controlling policy \citep{ljung1987theory}.
Our work recognizes the obvious utility of being able to use offline data. 

\paragraph{Model-Based Offline Reinforcement Learning:}
Our results are also of importance to model-based offline reinforcement learning (RL)~\citep{levine2020offline,kidambi2020morel,yu2020mopo} which is highly relevant to operations and managerial decision-making problems. For instance, data sets on prognosis, diagnosis, and treatment decisions made by physicians have been proposed to be used to train RL agents to potentially identify new (superior) paths to achieving the same (better) outcomes for patients~\citep{shortreed2011informing,liu2020reinforcement,yu2021reinforcement,chen2021probabilistic}. 
Analogously, in manufacturing and service operations management settings, as implied by~\cite{armony2015patient} in the hospital flow setting, data collected using pre-existing flow control and routing policies can be mined to discover new/better protocols and policies. Offline RL is a natural learning framework to achieve this.  

The model-based setting involves constructing a model for the transition probability matrix and then using it to solve the expected Bellman equation. Notable works in this regard are the trio of papers \cite{li2022settling,li2022settlingh,yan2022model} which prove, in the limited setting of discounted or finite-horizon problems under Markovian policies, that the model-based offline RL is minimax optimal.
Our results show that it continues to be an optimal estimator in the non-Markovian regime under suitable mixing conditions.

\subsubsection*{Outline.} The rough outline of the paper is as follows. 
In \Cref{sec:preliminaries}, we introduce some notation and the concepts from uniform mixing and weak mixing. 
In \Cref{sec:setup}, we construct the empirical estimator $\hat{M}^{(l)}$ for the transition matrix $M^{(l)}$ for any control $l$ and formally introduce our assumptions. 
We then illustrate the trade-off discussed previously by producing weaker PAC bounds for the estimation error $\sup_{l}\|\hat{M}^{(l)}-M^{(l)}\|_{\infty}$ under weaker mixing assumptions, and a stronger minimax PAC bound under stronger mixing assumptions. 
In Section \ref{sec:examples}, we apply our main result to derive statistical guarantees for various reward-free offline RL tasks under a range of settings, such as stationary controls, Markov controls, and episodic controls. 
Finally, in Section \ref{sec:value}, we use our estimator to obtain estimation guarantees for learning the value function. 
We end with a summary and discussion of the open questions in Section~\ref{sec:conclusions}.

\section{Preliminaries}~\label{sec:preliminaries}
\noindent\textit{Notations.} Let $\naturalset$ and $\mathbb{R}$ denote the natural and real numbers, and the symbol $\lfloor\cdot\rfloor$,  the floor function. 
All random variables in this paper will be defined with respect to a filtered probability sample space $(\Omega, \mathcal F, \mathbb F, \mathbb P)$, where $\mathcal F$ is a $\sigma$-algebra and $\mathbb F := \{\mathcal F_i\}$, with $\mathcal F_i \subset \mathcal F$, is a given filtration. 
Let $\{X_i\}$ represent a discrete-time stochastic process adapted to $\mathbb F$, with finite state space $\chi$. 
Overloading notation, we also denote by $\probl(X)$ the law of the random variable $X$. Let $\expec[X]$ be the expectation and $\sigma(X)$ the $\sigma$-algebra induced by $X$. 
A $d\times d$ matrix $Q$ is a \emph{stochastic matrix} if the rows of $Q$ add up to 1 and $Q_{s,t}$ denotes the $(s,t)$'th entry of $Q$. 
Let $\Ibb$ be a finite set representing the {\it control space}, and  $\{a_i\}$ represent the (not necessarily Markovian) sequence of controls where $a_i\in \Ibb\enspace\forall\enspace i$. 

\paragraph{Definitions.}
Following~\cite{borkar1991topics}, we define a {\it controlled Markov chain} (CMC) as an $\mathbb F$-adapted pair process $\{(X_i,a_i)\}$, where the process $\{X_i\}$ satisfies 
\begin{small}
\[
     \mathbb P \left( X_{i+1} = s_{i+1} | \mathcal F_i \right) = \mathbb P(X_{i+1} = s_{i+1} | X_i = s_i, a_i = l) =: M_{s_i,s_{i+1}}^{(l)}.
\]    
\end{small}
Let $\Mbb:=\{M^{(1)},\dots,M^{(k)}\}$ represent the set of transition probability matrices where $M^{(l)} = \left[  M^{(l)}_{s,t} \right]$ for all $s,t \in \chi$ and $l \in \Ibb$. Since $|\Ibb|=k$ the number of possible transition matrices for any given CMC is finite. The control sequence $\{a_i\}$ is assumed to satisfy $a_i \in \mathcal F_i$ for each $i \geq 0$ (i.e., $\{a_i\}$ is an adapted sequence of controls). { We emphasize that our theory holds even when $a_i$ is non-Markovian and non-time-homogenous. }
Let $\Mcal_{\chi,\Ibb}$ be the class of all probability measures over state-control pairs for a CMC with an initial distribution $\initialD$. This constitutes the class of data generating measures that we consider. 
In the case where $\{a_i\}$ is deterministic, $\{X_i\}$ forms a time inhomogeneous Markov chain, where the transition matrix changes at time step $i$ according to the control $a_i$.
Observe, in particular, that if $a_i = f(X_i)$, for some given function $f : \mathcal S \to \Ibb$, then $\{a_i\}$ is a Markov control sequence.  

{
Let $\hat M_{s,t}\pow l$ be the normalized state-control visitation frequencies (defined in eq. \ref{eq:lestimate}) for the triplet $(s,l,t)\in\chi\times\Ibb\times\chi$ and $\hat M\pow l$ be the matrix $[\hat M_{s,t}\pow l]$. Our objective is to find the sample complexity $m_{opt}$ such that
$
\ \prob\lp \sup_l\|\hat M\pow l-M\pow l\|>\eps \rp<\delta \ \text{ whenever } m\geq c_1 m_{opt}   
$, and there exists no estimator $\tilde M$ such that $\ \prob\lp \sup_l\|\tilde M\pow l-M\pow l\|>\eps \rp<\delta \ \text{ whenever } m\leq c_2 m_{opt}$ (for positive universal constants $c_1, c_2$). {Our findings (Theorem \ref{thm:minimax}) indicate $m_{opt}$ to be roughly of order $\Omega(dk)$. } Therefore, the empirical estimator achieves the minimax risk $\Rcal_m$ (as defined below) over the class of data-generating models $\Mcal_{\chi,\Ibb}$ whenever the number of samples is{ $m$ is of the order $\Omega(dk)$}. 
\begin{definition}~\label{defn:minimax}
Any element $\Pcal\in \Mcal_{\chi,\Ibb}$ has an associated set of transition matrices $\lc M\pow 1,\dots,M\pow k\rc$ and a conditional distributions over the control $\{a_i\}$ conditional on the history until time $i$.  Then the minimax risk of an estimator $\hat \Mbb :=(\hat M\pow 1,\dots, \hat M\pow k)$ of $\Mbb:=(M\pow 1,\dots,M\pow k)$
is given by
\[
\Rcal_m = \inf_{\hat \Mbb}\sup_{\Pcal\in \Mcal_{\chi\times\Ibb}}\prob\lp \sup_{l\in\Ibb}\| \hat M\pow l-M\pow l \|_{\infty} >\eps \rp.
\]    
\end{definition} 
\begin{remark}
    Observe that we have defined the conditional probability distributions over $a_i$ implicitly, as we never explicitly require it for our analysis. 
\end{remark}
}
{ Intuitively, to enable fast learning, we need bounds on how frequently the CMC visits all the state-control pairs, and for how long it retains its past memory. To formalize these notions, we define \emph{return times} and \emph{mixing coefficients} of a CMC. }
For two time points $i<j$, we define the \emph{history} $\History_i^j$ to be 
$\History_i^j:= \sigma(X_j,a_j,\dots,X_i,a_i) \subset \mathcal F_j$ and \emph{sample history} $\history_i^j\in\lp\chi\times \Ibb\rp^{(j-i+1)}$ to be a fixed sequence of states and controls $\history_i^j:=\lp s_j,l_j,\dots,s_i,a_i \rp$. 
%
We recursively define the `time to return' for every pair of states and controls $(s,l)$ as follows.

\begin{definition}~\label{def:return-time}
The first \emph{hitting time} $(s,l)$ is defined as
\begin{align*}
&\tau_{s,l}\pow{1}:=\min\lc n: (X_n=s,a_n=l),\rdot\\
&\qquad \ldot(X_j\neq s,a_j\neq l)\  \forall \ 0<j<n \rc.
\end{align*}
When $i\geq2$ the $i$-th {\emph{time to return}} (or {\emph{return time}}) of the state-control pair $(s,l)$ is recursively defined as 
\begin{small}
\begin{align*}
  &\tau^{(i)}_{s,l}:= \min\{n:(X_{\sum_{k=1}^{i-1}\tau^{(k)}_{s,l}+n}=s,a_{\sum_{k=1}^{i-1}\tau^{(k)}_{s,l}+n}=l),\\
  &\qquad (X_{j}\neq s\cup a_{j}\neq l)\enspace\forall \enspace \sum_{k=1}^{i-1}\tau^{(k)}_{s,l}<j<\sum_{k=1}^{i-1}\tau^{(k)}_{s,l}+n\}.  
\end{align*}
\end{small}
\end{definition}
{
\subsection{Mixing Coefficients}
In this subsection we define the weak and uniform mixing coefficients and related lemmas. Let  $\lc (X_i,a_i)\rc$ be an CMC. 
For any $i < j \leq m \in \mathbb N$, let $\Tbb\in(\chi\times\Ibb)^{m-j+1}$, $s_1,s_2\in\chi$, and $l_1,l_2\in\Ibb$.
Let $\history_0^{i-1}$ be an element of $(\chi\times\Ibb)^{i}$. 
Define the map $(\mathbb{T},s_1,s_2,l_1,l_2,\history_0^{i-1}) \mapsto \eta_{i,j}(\mathbb{T},s_1,s_2,l_1,l_2,\history_0^{i-1})$ as $\eta_{i,j}(\mathbb{T},s_1,s_2,l_1,l_2,\history_0^{i-1}) :=|A-B|$ where
$A=\prob\left((X_m,a_m,\dots, X_j,a_j)\in \mathbb{T}|X_i=s_1,a_i=l_1,\History_0^{i-1}=\history_0^{i-1}\right)$ and $B=\prob\left((X_m,a_m,\dots,X_j,a_j)\in\mathbb{T}|X_i=s_2,a_i=l_2,\History_0^{i-1}=\history_0^{i-1}\right)$.
Then the \emph{weak mixing} coefficient $\bar\eta_{i,j}$ is 
\begin{small}
    \begin{align*}~\label{def:weak-mixing}
    \Bar{\eta}_{i,j}:= \underset{\color{black}\substack{\mathbb{T},s_1,s_2,l_1,l_2,\history_0^{i-1},\\\prob\lp X_i=s_1,a_i=l_1,\History_0^{i-1}=\history_0^{i-1}\rp>0,\\\prob\lp  X_i=s_2,a_i=l_2,\History_0^{i-1}=\history_0^{i-1}\rp>0}}{\sup} \eta_{i,j}(\mathbb{T},s_1,s_2,l_1,l_2,\history_0^{i-1}).\numberthis
\end{align*}
\end{small}
With $\Tbb\in(\chi\times\Ibb)^{m-j+1}$ and $\history_0^{i}$ an element of $(\chi\times\Ibb)^{i+1}$ as before, the \emph{uniform-mixing} coefficient is 
\begin{align*}~\label{def:uniform-mixing}
   & \phi_{i,j}:= \ \sup_{\substack{\Tbb,\history_0^i,\\ \prob\lp \History_0^{i}=\history_0^{i} \rp>0}} \lv \prob((X_m,a_m,\dots, X_j,a_j)\in \mathbb{T})\rdot\\&\qquad \ldot-\prob\lp(X_m,a_m,\dots, X_j,a_j)\in  \mathbb{T}|\History_0^{i}=\history_0^{i}\rp\rv .\numberthis
\end{align*}
The following lemma relates the two mixing coefficients. Its proof can be found in Section \ref{sec:prf-mxlemm}
\begin{lemma}~\label{lemma:mixing-lemm}
The uniform and weak mixing coefficients in equations \ref{def:weak-mixing} and \ref{def:uniform-mixing} satisfy 
 $   \phi_{i,j}\leq \bar\eta_{i,j}\leq 2\phi_{i,j}.$
\end{lemma}
\begin{remark}
    We would like to point out that as defined, both $\bar \eta$ and $\phi$ are dependent on $m$. This dependence does not affect the analysis. Therefore, we will follow the convention in literature \citep{kontorovich2008concentration} and make the dependence of $\bar \eta$ on $m$ implicit. 
\end{remark}
\begin{remark}
    We point the interested reader to the classic text \citep[Theorem 3.1]{bradley} for the relationship between the uniform mixing coefficient and the rate of convergence in total variation distance to the stationary distribution. 
\end{remark}
}

\section{Empirical Estimation of Transition Probability Matrices}~\label{sec:setup}
{ As mentioned in the introduction, our objective is to estimate the transition matrices of the CMC from a single, finite sample path of length $m \gg 1$. Recall $\hat M_{s,t}\pow l$ from \cref{eq:lestimate} and define
\begin{align*}
&\hatmls:=\left(\Hat{M}_{s,1}^{(l)},\Hat{M}_{s,2}^{(l)},\dots,\Hat{M}_{s,d}^{(l)} \right), \text{ and }\\
&\mls:=\left({M}_{s,1}^{(l)},{M}_{s,2}^{(l)},\dots,{M}_{s,d}^{(l)}\right).\numberthis\label{eq:proofsketch-eq1prime}    
\end{align*}
$M\pow l(s,\cdot)$ is the $s$-th row of the transition matrix $M\pow l$, and $\hat M\pow l (s,\cdot)$ is the corresponding estimate. }
{\noindent
\Cref{prop:err-bnd} shows the needs to control the number of visits to a state-control pair $N_s^{(l)}$ to find theoretical guarantees for $\hat M_{s,t}\pow{l}$.  }
{
\begin{proposition}~\label{prop:err-bnd}
Consider a sample $\indexeddata$ from a controlled Markov chain.
%
Let $0<n_{low,s}<n_{high,s}<m$ be any two integers. Then we have
\begin{small}
\begin{align*}
     & \prob\Bigg(\bigg\|\hatmls  -\mls\bigg\|_1 > \epsilon,\ n_{low,s}\leq N_s^{(l)}\leq n_{high,s} \Bigg) \\
     &\qquad 
     \leq m\exp\left({-\frac{n_{low,s}}{2} \max \left\{0,\epsilon-\sqrt{\frac{d}{n_{high,s}}}\right\}^2}\right)\numberthis~\label{eq:prop-t1bound}.
\end{align*}
\end{small}
\end{proposition}
}
{ The count statistics $N_s\pow l$ are well studied \citep{billingsley1961statistical} when the process is a stationary ergodic Markov chain.} 
We list three challenges while moving from Markov chains to controlled Markov chains. 
\begin{enumerate}
    \item \textbf{Question of Aperiodicity.} 
    Consider the following three transition probability matrices
    \begin{small}
    \begin{align*}
        &M\pow{1} = \begin{bmatrix}
        0 & 1 & 0 \\ 
        1/3 & 1/3 & 1/3 \\ 
        1/3 & 1/3 & 1/3
        \end{bmatrix} \ M\pow{2} = \begin{bmatrix}
        1/3 & 1/3 & 1/3 \\ 
        0 & 0 & 1 \\ 
        1/3 & 1/3 & 1/3
        \end{bmatrix}\\
        &\ M\pow{3} = \begin{bmatrix}
        1/3 & 1/3 & 1/3 \\ 
        1/3 & 1/3 & 1/3 \\ 
        1 & 0 & 0 
        \end{bmatrix}.
    \end{align*}        
    \end{small}
    It can be verified easily that each of the transition probability matrices is aperiodic (and, in fact, ergodic). However, consider a time-inhomogenous Markov chain on state-space $\{1,2,3\}$ where the transition matrices arrive in a sequence $(M\pow{1},M\pow{2},M\pow{3},M\pow{1},M\pow{2},M\pow{3},\dots)$. Not only is it periodic if the initial state is $1$, {it is guaranteed to eventually become periodic.}
    \item \textbf{Question of Irreducibility.} 
    \cite[Theorem 13.0.1]{meyntweedie} show that an aperiodic and irreducible Markov chain admits a stationary distribution, 
    a key consequence of which is Kac's theorem \citep[Theorem 10.2.2]{meyntweedie} establishing the finiteness of the return times of every state in a Markov chain. However, such notions do not translate to a controlled Markov chain. 
    \item \textbf{Question of Mixing.} 
    An ergodic Markov chain on a finite state space is uniformly mixing. However, no equivalent result exists for controlled Markov chains. 
\end{enumerate}
Our first two assumptions address the questions of aperiodicity and irreducibility by ensuring that no part of the chain is deterministic in nature, and every state-control pair $(s,l)$ is visited sufficiently often. 
\begin{assumption}~\label{assume:sets}
For all times $i$, there exist constants $\zeta_1$ and $\zeta_2$ 
and a set $\mathcal{S}_i\subset \statespace\times\Ibb$ such that \begin{align*}
    0 < \zeta_2\leq \prob[(X_i,a_i)\in \mathcal{S}_i]\leq\zeta_1 < 1.
\end{align*}
\end{assumption}
{\color{black}\begin{remark}
If the controlled Markov chain satisfies the previous assumptions on all but a finite number of time points, our results continue to hold by discarding data.
However, it would lead to more cumbersome (but very similar) calculations. 
\end{remark}
}

\begin{assumption}[Return Time]~\label{assume:return_time}
There exists an integer $T>0$ such that the return time $\tau^{(i)}_{s,l}$ satisfies
\begin{align*}
    \sup_{s,l,i}\E[\tau^{(i)}_{s,l}|{\Fcal_{\sum_{p=0}^{i-1}\tau_{s,l}\pow{p}}}]<T \textit{ almost everywhere.} 
\end{align*}
\end{assumption}
{ The following lemma on the expected count statistics follows as a consequence of the previous assumptions. The main theorem is proved by controlling deviations around this expectation.} Its proof is in Section \ref{sec:prf-expbd}.
{
\begin{lemma}~\label{lemma:expec-bound}
  For any controlled Markov chain that satisfies Assumptions \ref{assume:sets} and \ref{assume:return_time}, 
  \begin{align}
      \frac{m}{T}-1 < \E\left[ N_s^{(l)} \right] \leq m \max\{ \zeta_1, 1-\zeta_2\},
  \end{align}
  where $\zeta_1, \zeta_2 \in (0,1)$ are defined in Assumption~\ref{assume:sets}.
  In particular, if $m\geq 2T$, then
  \begin{align}
      \frac{m}{2T} < \E\left[ N_s^{(l)} \right] \leq m \max\{ \zeta_1, 1-\zeta_2\}.
  \end{align}

\end{lemma}
}
\begin{remark}
Observe a parallel between this lemma and the minorization property described in texts such as \citet[Chapter 5.1.1]{meyntweedie}, \citet{rosenthal1995minorization}, etc. In particular, when $m=k=1$, and $\initialD$ is uniform over $\chi$, this lemma recovers the minorization condition described in \citet[Equation 5.3]{meyntweedie} for a uniform measure. Furthermore, taking summation over all $l$ and $s$ in the lower bound,
    $\sum_{s,l}\frac{m}{2T} < \sum_{s,l}\E\lb N_s^{(l)}\rb=\E\lb\sum_{s,l} N_s^{(l)}\rb=m$. Therefore,
    $dk\frac{m}{2T} < m$,
which in turn implies that
\begin{align}
T>\frac{dk}{2}.\label{eq:expec-bd-rem1}
\end{align}
\end{remark}
\noindent { Lemma \ref{lemma:expec-bound} provides upper and lower bounds for $\expec[N_s\pow l]$. As mentioned, our next objective is to control the deviations of $N_s\pow l$ from its expectation. For that, we require } the following two assumptions on the decay of $\bar \eta$-mixing coefficients of $\{X_i,a_i\}$:
\begin{assumption}[$\bar \eta$-mixing]~\label{assume:eta-mix}
There exists a constant $\constant_\Delta>1$ independent of $m$ such that, 
\begin{align*}
    \|\Delta_m\|:=\underset{1\leq i\leq m}{\max} (1+ \Bar{\eta}_{i,i+1}+ \Bar{\eta}_{i,i+2}+\dots \Bar{\eta}_{i,m})\leq \constant_\Delta.
\end{align*}
\end{assumption}
\begin{assumption}[Exponential $\bar \eta$-mixing]~\label{assume:eta-exp}
There exists a constant $\constant_\Delta>1$ independent of $m$ such that, 
\begin{align*}
    \Bar\eta_{i,j}\leq \exp\lp -(j-i)\log\lp\frac{\constant_\Delta}{\constant_\Delta-1}\rp\rp.
\end{align*}
\end{assumption}
  {
A standard assumption in the offline RL literature is the finiteness of the clipped concentrability coefficient defined in \cite{li2022settling} as
    \begin{align*}
        C_{clipped}^* := \max_{i,s,l} \frac{\inf\lc\prob\lp X_i = s, a_i\pow o =l\rp, d^{-1}\rc }{\prob\lp X_i = s, a_i\pow {\ell} =l\rp},
    \end{align*}
 where $a_i\pow o$ and $a_i\pow \ell$ are controls generated by the optimal and logging policies respectively. 
 Our Assumptions \ref{assume:sets}, \ref{assume:return_time}, and \ref{assume:eta-exp} are more general: consider a controlled Markov chain where $l\in\lc  1, 2\rc$, $M\pow 1$ and $M\pow 2$ are positive stochastic matrices, and $a_i\pow o = i\mod 2$ and $a_i \pow \ell = (i+1) \mod 2$. It is easy to see that in this case, $C_{clipped}^*=\infty$. However, without any assumption on the optimal policy we can still recover a sample complexity of learning the transition matrices (see Proposition \ref{prop:inhomogenous-mc}), and the corresponding optimal policy value (Theorem \ref{thm:opt-pol1}). 

}
It is obvious that if Assumption \ref{assume:eta-exp} is satisfied, so is Assumption \ref{assume:eta-mix} with the same constant. It also follows from Lemma \ref{lemma:mixing-lemm} that if the $\bar \eta$-mixing coefficients satisfy the previous assumptions, so does the $\phi$-mixing coefficients with appropriately adjusted constants. 
Depending on which of the assumptions we make, we have the following increasingly strong concentration inequalities. { First, Proposition \ref{prop:tail-bound-k&R} below provides a Hoeffding bound on the tails of the count statistics $N_s\pow l$. }
\begin{proposition}~\label{prop:tail-bound-k&R} Consider a sample $\indexeddata$ from a controlled Markov chain that satisfies Assumptions \ref{assume:sets}, \ref{assume:return_time}, and \ref{assume:eta-mix}. 
Let $N_s\pow{l}$ be the number of visits to state-control pair $(s,l)$ as defined in \cref{eq:pair-visit-count}. Then for all integers $n_{high,s}> \expec[N_s\pow{l}]> n_{low,s}$, we have 
\begin{align*}
    &\prob(N_s^{(l)}\notin [n_{low,s} ,n_{high,s}])\leq\\
    & \qquad 2 \exp\left({-{\frac{\lp n_{low,s}-\frac{m}{2T}\rp^2}{2m\lp\constant_\Delta\rp^2}}}\right)\\
    & \qquad + 2 \exp\left({-{\frac{\lp n_{high,s}-m\max\{\zeta_1,1-\zeta_2\}\rp^2}{2m\lp\constant_\Delta\rp^2}}}\right).
\end{align*}
\end{proposition}
\begin{proof} \
The proof of this proposition 
follows from the fact that
\begin{align*}
    &\prob(N_s^{(l)}\notin [n_{low,s} ,n_{high,s}]) \\ 
    & \qquad = \prob(N_s^{(l)}-\expec[N_s\pow{l}]< n_{low,s}-\expec[N_s\pow{l}])\\
    & \qquad +\prob(N_s^{(l)}-\expec[N_s\pow{l}]>n_{high,s}-\expec[N_s\pow{l}]),
\end{align*}
and then applying Assumption \ref{assume:eta-mix} on Lemma \ref{lemma:kontorovich} from Section \ref{sec:tech-des}.
\end{proof}
Next, define $\rho_{s}\pow{l}:=\sup_{1\leq i\leq m} \prob\lp X_i=s,a_i=l\rp$. { Then under Assumptions \ref{assume:sets}, \ref{assume:return_time}, and \ref{assume:eta-exp}, Proposition \ref{prop:tail-bound-peligrad} produces a Bernstein inequality for controlling the tail probability of $N_s\pow{l}$.}
\begin{proposition}~\label{prop:tail-bound-peligrad} Consider a sample $\indexeddata$ from a controlled Markov chain that satisfies Assumptions \ref{assume:sets}, \ref{assume:return_time}, and \ref{assume:eta-exp}. 
Let $N_s\pow{l}$ be the number of visits to state-control pair $(s,l)$ as defined in \cref{eq:pair-visit-count}. 
Then there exists a positive constant $\constant_{pel}$ depending only upon $\constant_\Delta$ such that for all integers $n_{low,s}<\expec[N_s\pow{l}]<n_{high,s}$, we have 
\begin{footnotesize}
\begin{align*}
    &\prob(N_s^{(l)}\notin [n_{low,s},n_{high,s}]) \\
    &\leq 2 \exp\left(- \; \ \frac{\constant_{pel}\lp n_{low,s}-\frac{m}{2T}\rp^2}{4m\constant_\Delta\rho_s\pow{l}+1+\lp \frac{m}{2T}-n_{low,s}\rp\lp\log m\rp^2}  \right)\\
    & + 2\exp\left(- \; \ \frac{\constant_{pel}\lp n_{high,s}-m\zeta\pow\max\rp^2}{4m\constant_\Delta\rho_s\pow{l}+1+\lp n_{high,s}-m\zeta\pow\max\rp\lp\log m\rp^2}  \right)
\end{align*}
\end{footnotesize}
where $\zeta^{(\max)}:=\max\{\zeta_1,1-\zeta_2\}$.
\end{proposition}
{ Although Proposition \ref{prop:tail-bound-peligrad} requires a stronger assumption vs. Proposition \ref{prop:tail-bound-k&R} (geometric vs. arithmetic mixing), it provides a tighter concentration which can be used to derive a minimax sample complexity}. It is proved similarly to Proposition \ref{prop:tail-bound-k&R}, but by using Lemma \ref{lemma:peligrad} (also found in Section \ref{sec:tech-des}) instead of Lemma \ref{lemma:kontorovich}. In many practical examples, $\constant_{pel}$ is a universal constant. We discuss this in greater detail in the remark following Lemma \ref{lemma:markov-assume4} and in Section \ref{sec:minmax-ex}.

We can now state our first theorem regarding the sample complexity of estimating the transition probability matrices of controlled Markov chain.

\begin{theorem}~\label{thm:sample-complexity}
Consider a sample $\indexeddata$ from a controlled Markov chain that satisfies Assumptions \ref{assume:sets}, \ref{assume:return_time}, and \ref{assume:eta-mix}.
Let, $\{\hat M_{s,t}^{(l)} : l \in \Ibb\}$ be the empirical estimators as defined in \cref{eq:lestimate} with $\hat M\pow{l}$ being the corresponding estimated transition matrix.
There exists a universal constant $c>1$, such that for any $\epsilon>0$, and $\delta \in (0,1)$,  and with $d = |\chi|$ and $k = |\Ibb|$,
if it holds that
\begin{align*}
  & m>c \; \max\left\{\frac{T}{\eps^2}\log\lp \frac{dkT}{\epsilon^2\delta} \rp,\rdot\\  &\ldot\constant_\Delta^2\max\lc {T}^2, \frac{1}{\lp 1-\max\{\zeta_1,1-\zeta_2\}\rp^2}\rc \log\lp \frac{dk}{\delta} \rp\right\},
\end{align*} 
then the empirical estimator satisfies,
\begin{align}
  \prob\left(\underset{l\in \Ibb}{\sup} \left\|\Hat{M}^{(l)}-M^{(l)}\right\|_{\infty}>\epsilon  \right)< \delta.
\end{align}
\end{theorem}
 As we see in \Cref{thm:sample-complexity}, assuming that the mixing coefficients are summable (Assumption \ref{assume:eta-mix}) allows us to compute the sample complexity. 
However, in \Cref{thm:minimax} we will see that if we further assume the mixing coefficients to be geometrically decaying (Assumption \ref{assume:eta-exp}), then we have a reduced sample complexity that is also minimax.


\subsection{Sketch of Proof of Theorem \ref{thm:sample-complexity}}\label{sec:prf-skt-sampcomp} 
\textbf{STEP 1.} As in \cref{eq:proofsketch-eq1prime} let,
\begin{align*}
    \hatmls & =\left(\Hat{M}_{s,1}^{(l)},\Hat{M}_{s,2}^{(l)},\dots,\Hat{M}_{s,d}^{(l)} \right), \text{ and }\\
    \mls & =\left({M}_{s,1}^{(l)},{M}_{s,2}^{(l)},\dots,{M}_{s,d}^{(l)}\right).\numberthis\label{eq:proofsketch-eq1}
\end{align*}
By an application of the union bound, we get $\prob\left(\underset{l\in \Ibb}{\sup} \left\|\Hat{M}^{(l)}-M^{(l)} \right\|_{\infty}>\epsilon  \right)\leq \sum_{l\in\Ibb}\sum_{s\in\chi} \prob\left(\left\|\hatmls-\mls\right\|_1 > \epsilon\right)$.

\noindent \textbf{STEP 2.} For each $s\in \chi$ and $l\in \Ibb$, we use the law of total probability \citep[Proposition 4.1]{gut2005probability} to decompose $\prob\left(\left\|\hatmls-\mls\right\|_1 > \epsilon\right)$ into a high probability region and a low probability region. 
To be precise, for two integers $n_{high,s}$ and $n_{low,s}$ chosen appropriately by Lemma \ref{lemma:expec-bound}, we write
\begin{small}
\begin{align*}
    &\sum_{s,l}\prob\left(\left\|\hatmls-\mls\right\|_1 > \epsilon\right) \\ 
    & \leq \sum_{s,l}\sum_{n=n_{low,s}}^{n_{high,s}} \prob\left(\left\|\hatmls-\mls\right\|_1 > \epsilon, N_s^{(l)}=n \right)\\
    & +\sum_{s,l}\prob(N_s^{(l)}\notin [n_{low,s} ,n_{high,s}]).
\end{align*}
\end{small}
\textbf{STEP 3.}
In this step we observe that if
$
    m>\max\lc\frac{d}{\epsilon^2\left({1+\max\{\zeta_1,1-\zeta_2\}}\right)},\frac{64T}{\eps^2}\log\lp \frac{6dk}{\delta} \rp\rc,
$
Proposition \ref{prop:err-bnd} gives us $\sum_{s,l}\sum_{n=n_{low,s}}^{n_{high,s}} \prob\left(\left\|\hatmls-\mls\right\|_1 > \epsilon, N_s^{(l)}=n \right) \leq \delta/3.$

\noindent \textbf{STEP 4.}
In this step, we use Proposition \ref{prop:tail-bound-k&R} to upper bound $\sum_{s,l}\prob(N_s^{(l)}\notin [n_{low,s} ,n_{high,s}])$: 
\begin{align*}
    &\sum_{s,l}\prob(N_s^{(l)}\notin [n_{low,s} ,n_{high,s}])\\
    &\ \leq dk\left(2 \exp\left({-{\frac{\lp n_{low,s}-\frac{m}{2T}\rp^2}{2m\lp\constant_\Delta\rp^2}}}\right)
      \rdot\\
      & \ \ldot+ 2 \exp\left({-{\frac{\lp n_{high,s}-m\max\{\zeta_1,1-\zeta_2\}\rp^2}{2m\lp\constant_\Delta\rp^2}}}\right)\right).
\end{align*}
It follows that for a universal constant $c$, as long as
$m>c\max\lc \constant_\Delta^2\log\lp \frac{dk}{\delta} \rp\max\lc {T}^2, \frac{1}{\lp 1-\max\{\zeta_1,1-\zeta_2\}\rp^2}\rc\rc
$, we have
$ \sum_{s,l}\prob(N_s^{(l)}\notin [n_{low,s} ,n_{high,s}])\leq 2\delta/3
$ completing the sketch (details in Appendix~\ref{sec:prf-sampcomp}).
{
\subsection{Minimax Sample Complexity}~\label{sec:minimax-sc}
In Theorem \ref{thm:minimax}, we show that under the extra assumption of geometric mixing, our estimator is minimax optimal.
Before proceeding to Theorem \ref{thm:minimax}, we introduce some notation. Consider a sample $\indexeddata$ from a controlled Markov chain that satisfies Assumptions \ref{assume:sets}, \ref{assume:return_time}, and \ref{assume:eta-exp}. Let $\rho_\star:=\sup_{s,l} \sup_{1\leq i\leq m} \prob\lp X_i=s,a_i=l\rp$, and with $\constant_{pel}$ as in \Cref{prop:tail-bound-peligrad}, define
}
{
\begin{align*}
    &\constant_\zeta := \frac{8\lp2\constant_\Delta\rho_\star \lp1-\zeta\pow\max\rp^{-2}+\lp1-\zeta\pow\max\rp^{-1} \rp}{\constant_{pel}},\\
    &\constant_T := \frac{64\lp \constant_\Delta\rho_\star T^2+2T \rp}{\constant_{pel}},\\
    &\constant_{T,\delta}  :=\constant_T\log\lp\frac{6dk}{\delta}\rp,\quad  \constant_{\zeta,\delta} := \constant_\zeta\log\lp\frac{6dk}{\delta}\rp.
\end{align*}
\begin{theorem}~\label{thm:minimax}
Consider the setting of \Cref{thm:sample-complexity} and suppose that Assumptions \ref{assume:sets}, \ref{assume:return_time}, and  \ref{assume:eta-exp} are satisfied, and let $\rho_\star=\max_{s,l}\rho_s\pow{l}$. Then, there exists a universal constant $c>1$ such that if
\begin{align*}
    & m> c \ \max\lc\frac{4d}{\epsilon^2\lp1+\max\{\zeta_1,1-\zeta_2\}\rp},\rdot\\&\ldot\ldot\ \constant_{T,\delta}\lp\log\constant_{T,\delta}\rp^2,\rdot \ \constant_{\zeta,\delta}\lp\log \constant_{\zeta,\delta}\rp^2
    \rc,
\end{align*}
then the empirical estimator satisfies
$  \prob\left(\underset{l\in \Ibb}{\sup} \left\|\Hat{M}^{(l)}-M^{(l)}\right\|_{\infty}>\epsilon  \right)< \delta$,
for all $\eps,\delta>0$ and is minimax upto $\log$ and $\log\log$ terms whenever $0<\eps<1/32$. 
\end{theorem}
{
{We point out that this result differs from the previous one by a factor of $\constant_\Delta\rho_\star/\constant_{pel}$. In Section \ref{sec:minmax-ex} we present an example where  $\rho_\star$ is $O(1/T)$ and $\constant_{pel}$ is independent of $T$. Therefore, assuming exponential mixing improves the sample complexity by a factor of $1/T$ and is minimax optimal. 

}
\subsection{Sketch of Proof of Theorem \ref{thm:minimax}}\label{sec:prf-skt-minmax} 
{The proof of the Theorem is divided into two parts: (1) the sample complexity, and (2) the minimaxity.} The proof of sample complexity proceeds similarly to the proof of \Cref{thm:sample-complexity}. The key difference is in Step 4, where instead of using Proposition \ref{prop:tail-bound-k&R}, we use Proposition \ref{prop:tail-bound-peligrad}. {The intuition is to use a tighter Chernoff concentration inequality that is available for exponentially mixing random variables, instead of a weaker Hoeffding's inequality. This produces a tighter sample complexity that is provably minimax.} The details of the first part can be found in \Cref{sec:proof-scmm}. 
For this sketch, we focus on the minimaxity. Let $\Mcal_{\chi,\Ibb}$ be the class of all controlled Markov chain on state space $\chi$ with control space $\Ibb$. 
For two collection of stochastic matrices $\Mbb_1:=\lc M_1\pow l\rc_{l\in \Ibb},\Mbb_2:=\lc M_2\pow l\rc_{l\in \Ibb}$, define  
$
\lV \Mbb_1-\Mbb_2 \rV_\infty^*:=\sup_{l\in\Ibb}\lV M_1\pow{l}-M_2\pow{l}\rV_\infty
$.
Observe the minimax risk satisfies
\begin{align*}
\mmrisk & = \inf_{\hat\Mbb} \sup_{\Pcal\in\Mcal_{\chi,\Ibb}} \prob\lp{\nrm{\hat \Mbb-\Mbb}_\infty^* > \eps}\rp\\
& \geq \inf_{\hat\Mbb} \sup_{\Pcal\in\Mcal'} \prob\lp{\nrm{\hat \Mbb-\Mbb}_\infty^* > \eps}\rp,
\end{align*}
for any subclass of controlled Markov chains $\Mcal'\subset\Mcal_{\chi,\Ibb}$ and any estimation procedure, $\hat \Mbb$. 
The rest of the proof proceeds through 2 cases by constructing appropriate subclasses $\Mcal'$. {The motivation behind these choices are based on the fact that the uniform distribution is the least favourable choice for estimation \citep{brandwein1980minimax,van2006minimax,lehmann2006theory,fourdrinier2013bayes}. We make these examples explicit in Section \ref{sec:minmax-ex}.}

\paragraph{CASE I:}$\lp m<\frac{8d}{\epsilon^2\lp1+\max\{\zeta_1,1-\zeta_2\}\rp}\rp$  

\noindent Here, we choose a class of controlled Markov chains with controls distributed uniformly over $\Ibb$ and transition matrices, for vectors $\sigma=\lp\sigma_1,\dots,\sigma_{\frac{d}{2}}\rp\in\lc -1,1\rc^\frac{d}{2}$, given by
\begin{align*}
 M_{\sigma} & = \begin{pmatrix}
 \frac{1-p_{\star}}{d} & \hdots & \frac{1-p_{\star}}{d} & p_{\star} \\
  \vdots & \vdots & \vdots & \vdots\\
  \frac{1-p_{\star}}{d} & \hdots & \frac{1-p_{\star}}{d} & p_{\star} \\
  \frac{1 - p_\star + 16 \sigma_1 \eps}{d} & \hdots &  \frac{1 - p_\star - 16 \sigma_{\frac{d}{2}} \eps}{d} & p_{\star} 
\end{pmatrix}.
\end{align*}
We then use Tsybakov's reduction method to lower bound $\inf_{\hat\Mbb} \sup_{\Pcal\in\Mcal'} \prob\lp{\nrm{\hat \Mbb-\Mbb}_\infty^* > \eps}\rp$ for our chosen subclass of controlled Markov chains.
\paragraph{CASE II:} $m<\lp2\ \constant_{T,\delta}\lp\log\constant_{T,\delta}\rp^2, 2\ \constant_{\zeta,\delta}\lp\log \constant_{\zeta,\delta}\rp^2\rp$

\noindent \textbf{STEP 1.} For this case, we set $\Mcal'$ to be a class of controlled Markov chains with controls and transition probability matrices described in Section \ref{sec:minmax-ex}. 

\noindent\textbf{STEP 2.} We then use Tsybakov's reduction method \citep[Chapter 2.2]{tsybakov2009introduction} to observe that for any random variable $\Tbb$,
\begin{align*}
 \mmrisk &\geq\inf_{\hat\Mbb} \sup_{\Pcal\in\Mcal'} \prob\lp{\nrm{\hat \Mbb-\Mbb}_\infty^* > \eps}\gn\Tbb > m\rp\\
 &\qquad \times\prob\lp{\Tbb > m}\rp.    
\end{align*}
$\Tbb$ is chosen to be an appropriate ``touring time'' (the time to visit sufficiently many state-control pairs). 

\noindent\textbf{STEP 3.} We then prove that $\prob\lp{\Tbb > m}\rp$ is bounded away from zero as long as 
$m<2\ \constant_{T,\delta}\lp\log\constant_{T,\delta}\rp$.

\noindent\textbf{STEP 4.} We then argue that whenever $\Tbb>m$, there exists a state-control pair $s_0,l_0$ such that $N_{s_0}\pow{l_0}=0$.

\noindent\textbf{STEP 5.} If $N_{s_0}\pow{l_0}=0$, so is $N_{s_0,t}\pow{l_0}=0$ for all $t\in\chi$. This proves that there is an uniform error to estimate $M_{s_0,t}\pow{l_0}$, which proves our claim.
}
\section{Applications}\label{sec:examples}
We first briefly discuss how Assumptions \ref{assume:return_time}, \ref{assume:eta-mix}, and \ref{assume:eta-exp}, can be reduced to simpler assumptions. 
\subsection{Reduction of Assumptions}
{
    \paragraph{Reduction of Return Times.} First consider the assumption on return times introduced in Assumption \ref{assume:return_time}.
{
We call a sequence of random variables $\{Z_i\}_{i\geq 0}$ a $j$th-order Markov chain if for all $i$, the conditional distribution of $Z_\infty,\dots, Z_i$ satisfies
$\probl\lp Z_\infty,\dots,Z_i|Z_{i-1},\dots,Z_{0} \rp = \probl\lp  Z_\infty,\dots,Z_i|Z_{i-1},\dots,Z_{i-j} \rp.$
Observe that if $a_i$ is $j$th-order Markovian then so is the paired process $(X_i,a_i)$. For convenience of notation, define $\tau^\dagger:=\sum_{p=0}^{i-1}\tau_{s,l}\pow{p}$ and observe that $\tau_{s,l}\pow i$ is a function of only $X_{\tau^\dagger+1},a_{\tau^\dagger+1}, \dots,X_\infty,a_{\infty}$. It follows that}
\begin{align*}
    \sup_{s,l,i}\E[\tau^{(i)}_{s,l}|{\Fcal_{\tau^\dagger}}]= \sup_{s,l,i}\E[\tau^{(i)}_{s,l}|\History_{\tau^\dagger-j}^{\tau^\dagger}]
\end{align*}
almost everywhere. Moreover, if $a_i$ is independent of time point $i$ (also called ``stationary"), then have almost everywhere
\begin{align}~\label{eqn:return-timered}
    \sup_{s,l,i}\E[\tau^{(i)}_{s,l}|\History_{\tau^\dagger-j}^{\tau^\dagger}]=\sup_{s,l}\E[\tau^{(1)}_{s,l}|X_0=s,a_0=l].
\end{align}
}
 \paragraph{Reduction of Mixing Coefficients.} Next, we decompose the $\bar \eta$-mixing coefficients of the paired process $\{X_i,a_i\}$ into mixing coefficients over states and controls. 
 We motivate this decomposition using two facts:
\begin{enumerate}
    \item\label{fact:decomp-mot1} The controls of a controlled Markov chain are often chosen by the user and is well behaved.
    \item\label{fact:decomp-mot2} The mixing coefficients of the individual processes can be analysed more directly.
\end{enumerate}
We begin by defining the $\gamma$-mixing coefficients  $\gamma_{p,j,i}$ for controls as the following total variation distance 
\begin{small}
\begin{align*}
    \gamma_{p,j,i} &:= 
    \sup_{s_p, \history_{i+j}^{p-1}, \history_0^i} 
    \Big\lVert 
        \probl\Big( 
            a_p \Big| X_p = s_p, \History_{i+j}^{p-1} = \history_{i+j}^{p-1}, \History_0^i = \history_{0}^i
        \Big) \\
    &\quad - 
        \probl\Big( 
            a_p \Big| X_p = s_p, \History_{i+j}^{p-1} = \history_{i+j}^{p-1}
        \Big)
    \Big\rVert_{TV},
\numberthis\label{eq:def-gamma}
\end{align*}    
\end{small}
{
where $\prob\lp X_p=s_p,\History_{i+j}^{p-1}=\history_{i+j}^{p-1},\History_0^i=\history_{0}^{i}\rp>0$.}
\begin{assumption}[Mixing of controls]~\label{assume:control-mixing} There exists a constant $\constant\geq 0$ such that
\begin{align*}
    \sup_{1\leq i\leq{\color{black} \infty}}\sum_{j=1}^{{\color{black} \infty}} \sum_{p=i+j+1}^{{\color{black} \infty}} \gamma_{p,j,i}\leq \frac{\constant}{2}.
\end{align*}

\end{assumption}
\begin{remark}
In the Markovian settings, when the sequence of controls $a_i$ depend only upon $X_i$,
$
 \gamma_{p,j,i}=0
$
for all $p,j,i$. In such case, Assumption \ref{assume:control-mixing} is satisfied with $\constant=0$. This extends to the case where $a_i$ depends upon $j$ many past time points. If $a_i$ depend only upon $X_i,a_{i-1},X_{i-1},\dots,$ $a_{i-j+1},X_{i-j+1}$, then Assumption  \ref{assume:control-mixing} is satisfied with $\constant=j-1$. 
\end{remark}
{Next, we generalize the Dobrushin coefficients \citep{dobrushin1956central,dobrushin1956centralII,mukhamedov2013dobrushin} for inhomogenous Markov chains to the realm of controlled Markov chains.} Let $\indexeddata$ be a collected sample.
For all integers $j\geq i$ define the mixing coefficient $\bar\theta_{i,j}$
\begin{align*}
& \bar \theta_{i,j}:= \underset{\substack{s_1,s_2\in\chi,l_1,l_2\in \Ibb,\\ \prob(X_i=s_1,a_i=l_1)>0,\\ \prob(X_i=s_2,a_i=l_2)>0  }}{\sup}  \| \probl\lp X_j|X_i=s_1,a_i=l_1\rp\\
&\qquad\qquad\qquad\quad-\probl\lp X_j|X_i=s_2,a_i=l_2\rp \|_{TV},\numberthis\label{eq:def-theta}
\end{align*}
such that $(s_1,l_1)\neq(s_2,l_2)$. 
{ The following assumption on $\bar\theta$ controls the mixing of the state process $X_i$.}
\begin{assumption}[Mixing of States]~\label{assume:chain-mix} There exists a constant $\constant_{\theta}\geq 0$ such that
\[
\sup_{1\leq i\leq { \infty}}\sum_{j=i+1}^{{ \infty}} \bar \theta_{i,j}\leq\constant_{\theta}. 
\]
\end{assumption}
Note that neither of Assumptions \ref{assume:control-mixing} and \ref{assume:chain-mix} imply the other, as the following counter-examples illustrate.
\begin{enumerate}
    \item Let $(X_i,a_i)$ be an inhomogenous Markov chain for which  
    $
    \sup_{1\leq i\leq \infty}\sum_{j=i+1}^\infty \bar \theta_{i,j}=\infty.
    $
   {One example of such an inhomogenous Markov chain can be found in Lemma \ref{lemma:counter-example}}. However, since the controls are deterministic, every inhomogenous Markov chain satisfies Assumption \ref{assume:control-mixing}. We prove this fact formally in \Cref{prop:inhomogenous-mc}. Therefore, this chain satisfies Assumptions \ref{assume:control-mixing} but not Assumption \ref{assume:chain-mix}.
    \item For the second counter-example consider a controlled Markov chain $(X_i,a_i)$ where the $a_i$'s do not satisfy Assumption \ref{assume:control-mixing}. Let $X_i$ be independent draws from a uniform distribution over $\chi$. It is easily seen that $\bar\theta_{i,j}=0$ for this example. Therefore, this chain satisfies Assumptions \ref{assume:chain-mix} but not \ref{assume:control-mixing}.
\end{enumerate}
Observe that the previous assumptions on the states and controls imply the summability of the weak mixing coefficients. We formalize it in the following Lemma.
\begin{lemma}~\label{lemma:delta-bound}
 For any controlled Markov chain that satisfies Assumptions \ref{assume:control-mixing} and \ref{assume:chain-mix},
$
    \lVert\Delta_m\rVert\leq \constant+\constant_{\theta}+1,
$
where $\|\Delta_m\|=\underset{1\leq i\leq m}{\max} (1+ \Bar{\eta}_{i,i+1}+ \Bar{\eta}_{i,i+2}+\dots \Bar{\eta}_{i,m})$, and $\Bar{\eta}_{i,j}$ is as defined in \cref{def:weak-mixing}
\end{lemma}
{
\begin{remark}~\label{remark:mixing-assume}
We remark that \Cref{thm:sample-complexity} continues to hold under the weaker Assumption \ref{assume:eta-mix}. 
However, { since all of our examples satisfy Assumptions \ref{assume:control-mixing} and \ref{assume:chain-mix} we can invoke Lemma \ref{lemma:delta-bound} to prove that Assumption \ref{assume:eta-mix} holds.} Next, we state the following two assumptions as stronger versions of Assumptions \ref{assume:control-mixing} and \ref{assume:chain-mix}.
\end{remark}

\begin{assumption}[Geometric mixing of controls]~\label{assume:control-exp} There exists a constant $\constant_\star> 0$ independent of $m$ such that for all integers $j\geq i$, we have
$
    \sum_{p=i+j+1}^\infty \gamma_{p,j,i}\leq e^{-\constant_\star\lp j-i\rp}.
$
\end{assumption}
\begin{assumption}[Geometric mixing of States]~\label{assume:chain-exp} There exists a constant $\constant_{\theta,\star}> 0$ independent of $m$ such that for all integers $j\geq i$, we have
$
\bar \theta_{i,j}\leq e^{-\constant_{\theta,\star}\lp j-i\rp}.
$
\end{assumption}
We then get the following lemma as a counterpart to Lemma \ref{lemma:delta-bound}.
\begin{lemma}~\label{lemma:eta-bound} For any controlled Markov chain that satisfies Assumptions \ref{assume:control-exp} and \ref{assume:chain-exp}, there exists a positive constant $\constant_{cof}$ independent of $m$ such that $\forall$ integers $j\geq i$, we have
$
    \bar \eta_{i,j}\leq e^{-\constant_{cof}(j-i)}.
$
\end{lemma}
\begin{remark}~\label{remark:exp-mix}
It can be seen that if Assumptions \ref{assume:control-exp} and \ref{assume:chain-exp} are satisfied, then so are Assumptions \ref{assume:control-mixing} and \ref{assume:chain-mix} with constants $1/(1-e^{-\constant_\star})$ and $1/(1-e^{-\constant_{\theta, \star}})$ respectively. 
To simplify notations, we will denote $1/(1-e^{-\constant_\star})$ by $\constant$ and $1/(1-e^{-\constant_{\theta, \star}})$ by $\constant_\theta$ respectively. Finally, observe that Assumptions \ref{assume:control-exp} and \ref{assume:chain-exp} provide a sufficient condition for $ \bar \eta_{i,j}$ to be geometrically decaying uniformly over $m$.
\end{remark}
}

\subsection{Controlled Markov chains with stationary randomized controls}~\label{sec:example-stationary}
We say a CMC has stationary randomized controls if for any time $i$, state $s$, control $l$, and history $\history_0^{i-1}$ 
\begin{align*}
\probl\lp a_i=l|X_i=s,\History_0^{i-1}=\history_0^{i-1}\rp & =\probl\lp a_i=l|X_i=s \rp\\
&=\probl\lp a_1=l|X_1=s \rp.\numberthis\label{eq:stationary-eq1}    
\end{align*}
In this section we show that assumptions \ref{assume:sets},\ref{assume:return_time}, \ref{assume:control-mixing}, and \ref{assume:chain-mix} hold for a controlled Markov chain with stationary controls. Writing $P_s^{(l)}$ for $\prob\left(a_1=l|X_{1}=s\right)$,  the transition probability of the joint state control pair is
\begin{align*}
    & \prob\left(X_i=t,a_i=l|X_{i-1}=s,a_{i-1}=l'\right)\\
    &\qquad = \prob\left(X_i=t|X_{i-1}=s,a_{i-1}=l\right)\prob(a_i=l|X_i=t)\\
    &\qquad = M_{s,t}^{(l)}\times P_t^{(l)}.
\end{align*}
The state-control pair is a time homogeneous Markov chain with transition probabilities given by $M_{s,t}^{(l)}\times P_t^{(l)}$. 
Our goal is to estimate the transition probabilities $M^{(l)}_{s,t}$. 
Assume that $M\pow{l}$ is an aperiodic and irreducible (ergodic) transition probability matrix for all $l\in\Ibb$. Then, we have the following proposition. 
\begin{proposition}~\label{prop:station-prop1}
The paired process $\indexeddata$ is an uniformly ergodic Markov chain.
\end{proposition}
\vspace{-0.2 cm}
By verifying the aperiodicity and irreducibility of the paired process, the proof of Proposition \ref{prop:station-prop1} follows readily from \citet[Theorem 16.0.2]{meyntweedie}.
Let $\nu$ be the invariant distribution of this Markov chain with $\nu_{(s,l)}$ being invariant probability corresponding to $(s,l)$. 
{The following proposition proves that $\{(X_i,a_i)\}$ satisfies Assumptions \ref{assume:sets},\ref{assume:return_time}, and \ref{assume:eta-exp}. Its proof can be found in \Cref{prf:sta-mdp}.}

\begin{proposition}~\label{prop:stationary-mdp}
Let $\indexeddata$ be a sample from a controlled Markov chain with $d = |\chi|$, $k = |\Ibb|$, and stationary randomized controls. 
Fix $\epsilon>0$, and $\delta \in (0,1)$. Then there exists a universal constant $c>0$ and a constant $\constant_\theta>0$ such \Cref{thm:sample-complexity} is satisfied with $T=\sup_{s,l} {1}/{\nu_{s,l}}$, $\zeta_2=P_{min}$ $\zeta_1=1-(k-1)P_{min}$ and $\constant_\theta$. Moreover, if $\initialD=\nu$, then $\zeta_1=\zeta_2=1/T$ satisfies Assumption \ref{assume:sets}.
\end{proposition}

\subsection{Controlled Markov chains with non-stationary Markov controls}\label{sec:exam-Markov}
As the next example, we consider a controlled Markov chain with non-stationary control process. A controlled Markov chain is said to have non-stationary Markov controls if for any time period $i$, state $s$, control $l$, and sample history $\history_0^{i-1}$,
\[
\probl\lp a_i=l|X_i=s,\History_0^{i-1}=\history_0^{i-1}\rp=\probl\lp a_i=l|X_i =s\rp.
\]
 Observe that we allow the law of the control sequence to depend upon the time step $i$. For convenience, we refer to this as a `Markov control'.
We can write the transition probability of the  state-control pair as
\begin{align*}
    &\prob\left(X_i=t,a_i=l'|X_{i-1}=s,a_{i-1}=l\right)\\
    & 
    \qquad = \prob\left(X_i=t|X_{i-1}=s,a_{i-1}=l\right)\prob(a_i=l'|X_i=t)
     \\
     & \qquad = M_{s,t}^{(l)}\times P_{t,l'}^{(i)}.
\end{align*}
It is straightforward to see that the state-control pair is a time inhomogeneous Markov chain with transition probabilities given by $M_{s,t}^{(l)}\times P_{s,l'}^{(i)}$. 
Our goal is to estimate the transition probabilities $\prob(X_i=t|X_{i-1}=s,a_{i-1}=l')$. We proceed by making assumptions on the return times of the controls.
\begin{definition}
Define $\tau_{s,l}\pow{i,\star,j}$ to be the time between the $j-1$-th and $j$-th visit to control $l$ after visiting state-control pair $s,l$ for the $i$-th time. For ease of notation, denote $ \sum_{k=1}^i\tau_{s,l}\pow{k}+\sum_{k=1}^{j-1}\tau^{(i,\star,k)}_{s,l}=\tau_\star$, Then $\tau_{s,l}\pow{i,\star,j}$ is recursively defined as
\begin{small}
$\tau_{s,l}\pow{i,\star,j}:= \min\{n:(a_{\tau_\star+n}=l),a_{j}\neq l\enspace\forall \enspace \tau_\star<j<\tau_\star+n\}.
$
\end{small}
\end{definition}
Next we make some simplifying assumptions on $\tau_{s,l}\pow{i,\star,j}$ and $M\pow{l}$.
\begin{assumption}~\label{assume:markov}
\begin{enumerate}
    \item For some constant $T_\star>0$,
    $
    \sup_{i\geq 0} \expec[\tau_{s,l}\pow{i,\star,j}|\Fcal_{\sum_{p=1}^{i-1} \tau_{s,l}\pow{p}+\sum_{p=1}^{j-1}\tau_{s,l}\pow{i,\star,p}}]< T_\star \textit{ almost everywhere.}
    $
    \item There exists $M_{min}$ and $M_{max}$ such that for all $s,t\in \chi$ and $l\in\Ibb$ 
    \begin{align}
    0<M_{min}\leq M_{s,t}\pow l\leq M_{max}<1.   \label{eq:markov-eq1}     
    \end{align}
\end{enumerate}
\end{assumption}
{ The next lemma proves that under this assumption $\{(X_i,a_i)\}$ satisfies Assumption \ref{assume:return_time} and derives $T$.}
\begin{lemma}\label{lemma:return-time-markov}
Under the conditions of Assumption \ref{assume:markov} for all $(i,s,l)\in\naturalset\times\chi\times\Ibb$ it holds almost everywhere that
\begin{small}
    \begin{align*}
    &\expec[\tau_{s,l}^{(i)}|\Fcal_{\sum_{p=1}^{i-1} \tau_{s,l}\pow{p}}]\\
    &\ <\frac{T_\star M_{max}}{\max\{M_{max},1-M_{min}\}(1-\max\{M_{max},1-M_{min}\})}.\numberthis\label{eq:markov-eq3}         
    \end{align*}
\end{small}
\end{lemma}

We can now state our main result about the sample complexity of a controlled Markov chain with a non-stationary Markov controls. Its proof can be found in \Cref{prf:markov-mdp}.
\begin{proposition}\label{prop:markov-mdp}
Let $\indexeddata$ be a sample from a controlled Markov chain with non-stationary Markovian controls satisfying Assumption \ref{assume:markov}. Fix $\epsilon>0$, and $\delta \in (0,1)$. Then Assumption \ref{thm:sample-complexity} holds with $\zeta_1=M_{max}$, $\zeta_2=M_{min}$, $T=\frac{T_\star M_{max}}{\max\{M_{max},1-M_{min}\}(1-\max\{M_{max},1-M_{min}\})}\  \text{ and }\  \constant_{\theta}:= \frac{1}{dM_{min}}.$ 
\end{proposition}
We next illustrate how one can use Theorem \ref{thm:sample-complexity} to derive a sample complexity of learning the demand distribution of an inventory control problem. 

\subsection{Sample complexity of estimating transitions of a $(\texttt{\textbf{s}},\texttt{\textbf{S}})$-inventory  control problem:}~\label{sec:inventory-control}
{
In this section, we consider estimating the transition probability matrices 
for a finite state inventory control problem.
Here, the Markov state $X_i$ is the inventory at time $i$, and the controls are such that the inventory is always brought up to a  level $\texttt{\textbf{S}}$ whenever it falls below a level $\texttt{\textbf{s}}$. Assume that $\texttt{\textbf{S}}>2\texttt{\textbf{s}}$. Then, with $b_i = l\in\{0,\dots,\texttt{\textbf{s}}\}$ the demand at time $i$ (having
probability $\Pcal(b_i=l)=p_l$). the system has the following dynamics:
\begin{align*}
   & X_{i+1} = X_i + (\texttt{\textbf{S}}-X_i) a_i(X_i) - b_i
\\ 
&\quad \text{where} \quad  a_i(X_i) =   \begin{cases}
            1 \ \text{if} \ X_i<\texttt{\textbf{s}}\\
            0 \ \text{if} \ X_i\geq \texttt{\textbf{s}}
        \end{cases}
\end{align*}
Note that we have assumed $b_i\in\{0,\dots,\texttt{\textbf{s}}\}$ resulting in a system without backlog. We do this for simplicity, though we can easily relax this with some simple if tedious bookkeeping.

Observe that there are two controls $\{0,1\}$ and $\prob\lp X_i\geq\texttt{\textbf{s}},a_i=1\rp=\prob\lp X_i<\texttt{\textbf{s}},a_i=0\rp=0$. Thus we only need to estimate the transition probabilities $\prob(X_{i+1}=t|X_i=s,a_i=l)$ if either $\{s<\texttt{\textbf{s}}, l=1\} \text{ or } \{s\geq\texttt{\textbf{s}},l=0\}$. Therefore, the combined state-space for $\{X_i,a_i\}$ is $\{(0,1),(1,1),\dots,(\texttt{\textbf{s}}-1,1),(\texttt{\textbf{s}},0),(\texttt{\textbf{s}}+1,0)\dots,(\texttt{\textbf{S}},0)\}$. This is a Markov chain with {$1$-step transition probability matrix} $M$, whose $(s,l_1),(t,l_2)$'th element $M_{(s,l_1),(t,l_2)}=\prob(X_{i+1}=t,a_{i+1}=l_2|X_i=s,a_i=l_1)$ is,
\begin{small}
\begin{align*}
    &M_{(s,l_1),(t,l_2)} =  \\
    &\quad \begin{cases}
        p_{ \texttt{\textbf{S}}-t} \text{ if } t\in\{\texttt{\textbf{S}}-\texttt{\textbf{s}},\dots\texttt{\textbf{S}}\}, s\in \{0,\dots,\texttt{\textbf{s}}-1\},\\
        \qquad \qquad\ l_1=1,\ l_2=0\\
        p_{t-\texttt{\textbf{s}}} \text{ if } s\in\{\texttt{\textbf{s}},\dots,\texttt{\textbf{S}}\}, t\in\{s-\texttt{\textbf{s},\dots,s}\},\\
        \qquad \qquad \ l_1=0,\ l_2=0 \\
        0 \text{ otherwise. }
    \end{cases}\numberthis\label{eq:inventory-controldef}
\end{align*}    
\end{small}
This leads us to the following proposition whose proof can be found in Section \ref{sec:prf-icc}.
\begin{proposition}~\label{prop:ic-control}
    Let $\{(X_i,a_i)\}$ be a $(\texttt{\textbf{s}},\texttt{\textbf{S}})$ inventory process with $\texttt{\textbf{S}} > 2\texttt{\textbf{s}}$. Then Theorem \ref{thm:sample-complexity} applies with constants $T=(\texttt{\textbf{S}}-2\texttt{\textbf{s}})/M_{min}^{\texttt{\textbf{S}}-2\texttt{\textbf{s}}}$, $\Cbb_{\Delta} = (d/M_{min}^{\texttt{\textbf{S}}-2\texttt{\textbf{s}}})^{\lfloor(j-i)/(\texttt{\textbf{S}}-2\texttt{\textbf{s}})\rfloor}$, $\zeta_2 = M_{min}^{\texttt{\textbf{S}}-2\texttt{\textbf{s}}}$, and $\zeta_1=1-M_{min}^{\texttt{\textbf{S}}-2\texttt{\textbf{s}}}$.
\end{proposition}

\begin{remark}
    Observe that assuming $\texttt{\textbf{S}}>2\texttt{\textbf{s}}$ entails a loss of generality, but lets us reuse our earlier results.
    This assumption can be further generalized with more calculations and we leave it to the interested reader.
\end{remark}

{
The examples in Sections \ref{sec:example-stationary} and \ref{sec:inventory-control} can also be viewed as stationary Markov chains, and therefore the sample complexity results may also be recovered using the analysis in~\cite{wolfer2021statistical}. However, a naive application of~\cite{wolfer2021statistical} to more general CMCs will yield suboptimal sample complexity results. Furthermore,~\cite{wolfer2021statistical} can only be applied to stationary, ergodic chains, while our theory is applicable in much greater generality as highlighted by the examples in Section \ref{sec:exam-Markov}, and Appendix \ref{sec:more-ex}. This also opens up the interesting possibility of designing controls to estimate the transition matrices faster than it would be possible by a stationary ergodic Markov chain. The following section provides one such example.} 

{
\subsection{Example: Designing Controls for Faster Learning of Transition Matrices }~\label{sec:control-design}
 Let $M\pow 1$ and $M\pow 2$ be two stationary ergodic transition matrices. We assume that $M\pow 1$ and $M\pow 2$ have states $t_1$ and $t_2$ which are ``difficult to reach" in comparison to other states. Formally we assume the following:

\begin{assumption}~\label{assume:combined-chain}
Let $t_1\neq t_2$ be states in $\chi$, and $\iota_1:= \sum_{s\in\chi}M_{s,t_1}\pow 1,\text{ and }\iota_2:= \sum_{s\in\chi} M_{s,t_2}\pow 2. 
$
We assume $\exists\ M_{min}<1$ such that $ M_{s,t}\pow l\geq M_{min} > 8\min\lc \iota_1,\iota_2 \rc$ for any $s\in\chi$ and any $(t,l)\notin \{(t_1,1),(t_2,2)\}$.\end{assumption}
 \begin{remark}
     We remark that this assumption simplifies the calculations, and the result holds true much more generally as demonstrated by the empirical findings.
 \end{remark}
For such transition matrices, the following proposition demonstrates that (modulo Assumption~\ref{assume:combined-chain}) the sample complexity of individually  estimating such transition matrices can be larger than estimating them simultaneously as a CMC with a pre-designed control sequence. 
%
%
Its proof can be found in Section \ref{sec:prf-combchn}

\begin{proposition}~\label{prop:combined-chain} Let $m\pow 1$ and $m\pow 2$ be, respectively, the sample complexity (ignoring $\log$ terms) of learning the transition probabilities of a Markov chain with the transition probability matrix $M\pow 1$ and $M\pow 2$. 
Then, one can construct a controlled Markov chain with transition matrices $M\pow 1$ and $M\pow 2$ with deterministic controls $\{a_i\}$ such that the sample complexity (ignoring $\log$ terms) of learning this controlled Markov chain  $m\pow c$ satisfies
$
m\pow c < (m\pow 1+m\pow 2)/2$.
    
\end{proposition}
{
\begin{remark}
   From the explicit expressions (barring $\log$-terms)
   \begin{small}
   \begin{align*}
        m\pow 1 & = \; \max\left\{\frac{1}{\iota_1\eps^2}\log\lp \frac{d}{\iota_1\epsilon^2\delta} \rp, \frac{1}{\lp(d-2)M_{min}\iota_1\rp^2}\right\}\\
        m\pow 2 & = \; \max\left\{\frac{1}{\iota_2\eps^2}\log\lp \frac{d}{\iota_2\epsilon^2\delta} \rp, \frac{2}{\lp(d-2)M_{min}\iota_2\rp^2}\right\}\\
         m\pow c & =\max\lc \frac{4}{M_{min}\eps^2}\log\lp\frac{4d}{M_{min}\epsilon^2\delta}\rp, \frac{32}{(d-2)^2 M_{min}^4}\rc,
    \end{align*}
   \end{small}
    which is found in the proof of Proposition \ref{prop:combined-chain} one can see that $m\pow c$ is independent of either $\iota_1$ or $\iota_2$. The rest of the proof follows by straightforward algebra.
\end{remark}
}

{

As a numerical illustration, consider a controlled Markov chain  with corresponding transition matrices $M\pow 1$ and $M\pow 2$ that take the form of \cref{eq:inventory-controldef}, determined solely by two (deterministic) probability vectors $\textbf{p}\pow 1$ and $\textbf{p}\pow 2$. 
We can interpret this as corresponding to a time-inhomogeneous $(\texttt{\textbf{S}},\texttt{\textbf{s}})$ inventory system, wherein the demand distribution changes as a function of a ``low'' price $\textbf{p}\pow 1$ and a ``high'' price $\textbf{p}\pow 2$ control set by the inventory manager. 

Now, setting $(\texttt{\textbf{S}},\texttt{\textbf{s}})=(6,2)$, $\textbf{p}\pow 1$ and $\textbf{p}\pow 2$ can be constructed as $3\times 1$ dimensional vectors as follows: generate the ``low'' price vector $\textbf{p}\pow 1$ by sampling a $3\times1$ random vector of independent uniform $[0,1]$ marginals, divide the first coordinate by $1000$ and renormalize the vector to sum up to $1$. The ``high'' price vector is constructed similarly by multiplying the first coordinate by $100$. Roughly speaking, $\textbf{p}\pow 1$ and $\textbf{p}\pow 2$ correspond to $0$ sales having $0.1\%$ and $90\%$ probabilities.

}
We calculate the number of samples required to achieve a target error when the price/control sequence alternates between the ``low'' price $\textbf{p}\pow 1$ and the ``high'' $\textbf{p}\pow 2$, and compare it with the average number of samples required to achieve the same error for the individual chains. Figure \ref{fig:control-design} displays the results, with the error bars resulting from 100 repetitions. {The dotted curve represents $(m\pow 1+m\pow 2)/2$ where $m\pow 1$ is the number of samples required by the first Markov chain to hit a target error and $m\pow 2$ is the same for the second Markov chain. The dashed curve represents the samples required by the controlled chain to hit the target error $m\pow c$}. As we target smaller and smaller error, the average number of samples required to achieve a certain error for the individual Markov chains \textbf{exceeds} the number of samples required to learn the controlled Markov chain. This empirically validates our claim.

\begin{figure}[!h]
    \centering
    \includegraphics[width = \textwidth]{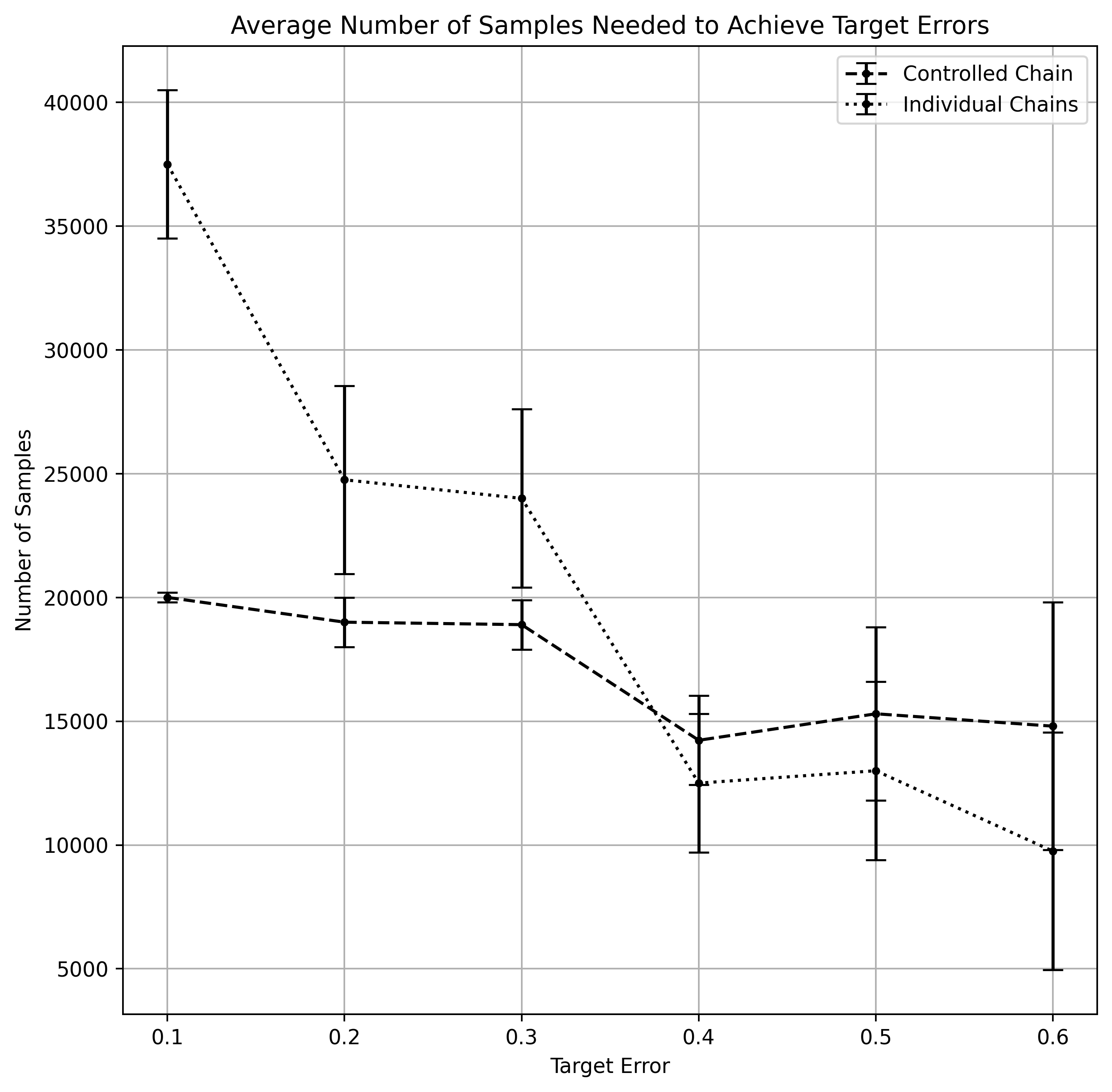}
    \caption{The number of samples required to achieve a target error for two individual Markov chains as described in Section \ref{sec:control-design} vs. learning them together as a controlled Markov chain. As we target smaller and smaller error, the controlled Markov chain requires fewer samples and also has  smaller variation compared to 
    learning the chains individually.}
    \label{fig:control-design}
\end{figure}
}

}
\section{Sample Complexity of Policy Evaluation and Optimal Policy Recovery}~\label{sec:value}
Optimal policy recovery (OPR) is a key problem in offline RL settings, wherein one wishes to identify the policy that maximizes the \textit{value function} given a sequence of states and controls generated by an unknown logging policy and unknown transition probabilities. An allied problem is offline policy evaluation (OPE) where we use the state-control sequence to estimate the value function of an arbitrary policy. In this section we demonstrate how our previous results can be used to provide a high probability bound for recovering the optimal policy. We first use Theorem \ref{thm:sample-complexity} to provide a high probability bound for OPE, and then extend that to a corresponding result for OPR.

Let $\Delta(\mathbb{I})$ denote the probability simplex on the control space, and suppose $\pi : \mathcal X \to \Delta(\mathbb{I})$ is a given stationary stochastic policy. 
Let the matrices $M$ and $\Pi$ represent the probabilities of the next state and action respectively, given the current state and action. Thus $M=[M\pow 1,\dots,M\pow k]^T$, and $\Pi = \lb\pi,\pi,\dots,\pi\rb$ are $dk\times d$ dimensional matrices.
Let $\tilde g := (\tilde g(x,a) : (x,a) \in \mathcal X \times \mathbb{I})$ be a $dk\times 1$ vector where element $s+k(l-1)$ denotes the cost associated with the state $s$ and control $l$. Then, the per-stage expected reward function $g :=(g(x):x\in\chi )$ is a $d\times 1$ vector where $g(x)= \sum_{a \in \mathbb{I}} \pi(x,a) \tilde g(x,a)$. 

For a known discount factor $0<\alpha_{dis}<1$, the value function $V_\Pi := (V_\Pi(x) : x \in \mathcal{X}) \in \mathbb R^d$,
obtained by solving the Bellman equation \citep{bertsekas2011dynamic}, is given by
\begin{align*}
V_\Pi &= \lp I-\alpha_{dis} \Pi^TM \rp^{-1} g.
\end{align*}
Substituting $M$ by $\hat M$, we obtain the plug-in estimate $\hat V_\Pi = \lp I-\alpha_{dis} \Pi^T\hat M \rp^{-1}g $.
%
%
 The next theorem provides a sample complexity bound on estimating the value function $V$.
\begin{theorem}~\label{thm:value-thm}
Let $\indexeddata$ be a sample from a controlled Markov chain with stationary randomized controls. Assume that for some $T>0$,  $\prob\lp X_i=s,a_i=l \rp >T^{-1}$ for all $s,l,i$. Then, there exists a universal constant $c>1$ such that $\prob\lp\lV \hat V_\Pi - V_\Pi\rV_\infty>\eps\rp<\delta
$ if
\begin{align*}
    &m>c \; \max\left\{\frac{T_\alpha}{\eps^2}\log\lp\frac{dkT_\alpha}{\epsilon^2\delta} \rp, \rdot\\&\ \ldot\constant_\theta^2\max\lc {T}^2, \frac{1}{\lp 1-\max\{\zeta_1,1-\zeta_2\}\rp^2}\rc \log\lp \frac{dk}{\delta} \rp\right\},
\end{align*}
$ \text{ where } T_\alpha = {\lV g\rV_{1}^2 d\alpha_{dis}^2 T}/{(1-\alpha_{dis})^4}, \zeta_1=\zeta_2=T^{-1}, \text{ and } \constant_\theta  = T/dk.$
\end{theorem}
The proof of this theorem can be found in \Cref{prf:value-thm}
}
\begin{remark}
    The assumption $\prob\lp X_i=s,a_i=l \rp >T^{-1}$ can be relaxed with appropriate assumptions on the return time of $X_i,a_i$. In practice, this would translate to an assumption on the logging policy. 
\end{remark}
{
{
Next, one can use any method to find the optimal policy; for instance policy iteration \citep[Chapter 1]{bertsekas2011dynamic} or policy gradient \citep[Chapter 13]{sutton2018reinforcement}. Let $\Pi_{opt}$ and $\hat \Pi_{opt}$ denote the optimal policies for maximising the reward function for the true and the estimated transition matrices $M$ and $\hat M$ respectively. The following theorem provides a sample complexity of recovering the optimal value. Its proof can be found in Section \ref{sec:prf-optpol1}.

\begin{theorem}~\label{thm:opt-pol1}
 Under the conditions of Theorem \ref{thm:value-thm}, we have  $\lV V_{\Pi_{opt}} -  V_{\hat \Pi_{opt}}\rV_\infty\leq\frac{d\sqrt{d}\ \alpha_{dis}}{(1-\alpha_{dis})^2} \lV g\rV_{1}\eps
    $ with probability at least $1-\delta$ if
    \begin{align*}
         m>c \; \max\left\{\frac{T_\alpha}{\eps^2}\log\lp\frac{dkT_\alpha}{\epsilon^2\delta} \rp, \constant_\theta^2\max\lc {T}^2, \rdot\rdot\\\ldot
         \ldot\frac{1}{\lp 1-\max\{\zeta_1,1-\zeta_2\}\rp^2}\rc \log\lp \frac{dk}{\delta} \rp\right\}.\numberthis\label{eq:optpol-eq1}
    \end{align*}
\end{theorem}

\begin{remark}
We can  make a further assumption that for any $\eps<1/1.5$, and policy $\Pi'$ such that $\lV \Pi_{opt}-\Pi'\rV_{\infty}>\eps$, the value functions corresponding to $\Pi_{opt}$ and $\Pi'$ satisfy  
$
         \inf_{s\in \chi}\lc V_{\Pi_{opt}}(s)-V_{\Pi'}(s)\rc>{2}\eps.$ 
        Then it will trivially follow that $\lV\Pi_{opt}-\hat\Pi_{opt}\rV_\infty<\eps$ giving us a theoretical guarantee for recovering the optimal policy with high probability. 
\end{remark}

One can compare the dependence of the sample complexity (eq. \ref{eq:optpol-eq1}) on the discount factor $\alpha_{dis}$ to those in \cite{li2022settling}. Our dependence is $\alpha_{dis}^2(1-\alpha_{dis})^{-4}$, which is better than $(1-\alpha_{dis})^{-3}$  in \cite{li2022settling} when $\alpha_{dis}<0.618$. However, neither achieves the lower bound $(1-\alpha_{dis})^{-2}$ given by \cite{agarwal2020model}. }
}}
\section{Conclusions}~\label{sec:conclusions}
In this paper, we derive exact rates of convergence for the empirical estimator of the transition probability matrices of a controlled Markov chain, and used these to derive the sample complexity of achieving a desired estimation error.
We tease out the exact effect of the mixing coefficients of the states and controls on the sample complexity, and provide conditions under which the empirical estimator is minimax optimal. 
We use our sample complexity results in a number of examples, including error bounds for the value function and the optimal policy estimated from data.
Below, we highlight three possible topics for future research.

\paragraph{Countable state spaces.}  As an obvious extension to our work, consider the problem of countably infinite state and control spaces. Some work in this regard can be found in \cite{wolfer2021statistical} where state spaces are countably infinite, but there are no easy extensions to the setting with countably infinite control space.

\paragraph{Uncountably infinite state space and finite controls} We have also found no result which derives the minimax sample complexity of estimating the transition probability distribution of a controlled Markov chain on an uncountably infinite state space. The histogram estimator is the obvious counterpart to this work. Indeed, recent studies \citep{sart2014estimation,loffler2021spectral} demonstrates promising properties of the histogram in estimating the transition functions of a continuous Markov chain. But to the best of our knowledge, the techniques do not translate to uncountable control spaces, and optimally estimating the transition probability distribution remains an open question. 

\paragraph{Learning in presence of weaker mixing or adversarial controls.} Although strong mixing properties of the controls is a sufficient condition for the `estimability' of the transition matrices, it may not be a realistic assumption when the system dynamics are weakly mixing or adversarial \citep{pinto2017robust}. 
System identification under the presence of an adversary remains an interesting question 
was addressed in a recent paper \citep{showkatbakhsh2016system} using strong linearization assumptions and exponential computation times. However, this is well beyond the scope of the current work and is a direction for future study.  

{\paragraph{Instance-dependent and data-dependent learning.} Our bounds are derived over an entire class of CMC models. Indeed, so long as a specific model conforms to the assumptions we have imposed over this class, our bounds are applicable. This reflects extant analyses of offline RL~\citep{li2022settling,li2022settlingh}. On the other hand recent work in online RL seeks to establish \emph{instance-dependent} bounds for specific models  (for example~\cite{zanette2019tighter,mou2021optimal,khamaru2022instance}). Establishing such bounds for offline CMC identification is also an interesting open problem. 
\cite{khamaru2022instance} also emphasize that inferential theory for RL should be data-dependent, for instance allowing for the computation of data-dependent confidence intervals. This is an important future direction for our work as well.}


\section{Acknowledgements}
Imon Banerjee was supported in part by the Ross-Lynn fellowship and McLean scholarship at Purdue University. Harsha Honnappa was partly supported by the National Science Foundation through grants CAREER/2143752, DMS/1812197 and DMS/2153915. Vinayak Rao was supported by the National Science Foundation grants RI/1816499 and DMS/1812197. Imon Banerjee thanks Anamitra Chaudhuri for numerous insightful discussions and comments throughout the duration of this project. We thank the anonymous reviewers for their insightful comments and especially for pointing us towards the interesting application detailed in Section \ref{sec:control-design}.


\vspace{0.1 cm}
\hrule
\vspace{0.5 cm}
\begin{footnotesize}
\textbf{Imon Banerjee} is an IEMS alumni fellow at Northwestern University. His current research encompasses mathematical statistics and stochastic processes, with applications in reinforcement learning. More broadly, he is interested in exploring the theoretical aspects of machine learning using tools from applied probability. His email address is imon750@gmail.com.

\textbf{Harsha Honnappa} is an Associate Professor of Industrial Engineering at Purdue University. His research interests as an applied probabilist encompass stochastic modeling, optimization and control, with applications to machine learning, simulation and statistical inference. His research is supported by the National Science Foundation, including an NSF CAREER award, DOD, and the Purdue Research Foundation. He is an editorial board member at Operations Research, Operations Research Letters and Queueing Systems journals.  His email address is {honnappa@purdue.edu} and his website is \url{https://engineering.purdue.edu/SSL}.

\textbf{Vinayak Rao}  is an associate professor in the Department of Statistics at the Purdue University, West Lafayette. His research interests include methodological, computational and theoretical aspects of Bayesian and nonparametric statistics. His email address is {varao@purdue.edu} and his website is \url{https://varao.github.io/}.    

\end{footnotesize}
\newpage
\begin{APPENDICES}
\vspace{1em} 
    \begin{center}
        \hrule height 1pt \vspace{0.5em} 
        {\Large\bfseries Online Companion File To Offline Estimation of Controlled Markov Chains: Minimaxity and Sample Complexity} 
        \vspace{0.5em} \hrule height 1pt 
    \end{center}
    \vspace{1em} 
{\protect\NoHyper%
    \tableofcontents
    \addtocontents{toc}{\protect\setcounter{tocdepth}{2}}
    \protect\endNoHyper}%

\section{Index of notations}
\ 
\begin{table}[!ht]
    \centering
    \begin{tabular}{c|c|c|c|c|c}
            \hline
            Item & Description & Location & Item & Description & Location\\
            \hline
            $\chi$, $ \Ibb$& State and control spaces & Notations & $\Omega,\Fcal,\prob$ & A probability space  & Notations\\
            $X_i,a_i$ & Random variable for state/control & Notations& $d,k$ & Number of states and controls & Notations \\
            $\Hcal_i^j$ & $\sigma(X_i,a_i,\dots, X_j,a_j)$ & Notations & $\history_i^j$& A realisation of $\Hcal_i^j$ & Notations\\
            (s,l,t)& some state-control-state triplet & Notations & $\Fbb$& A filtration $\{\Fcal_i\}_{i\geq 1}$ & Notations\\
            $\tau_{s,l}\pow 1$ & First hitting time of $(s,l)$ & Def. \ref{def:return-time} & $\tau_{s,l}\pow 1$& $i$-th return time to $(s,l),i>1$& Def. \ref{def:return-time} \\ 
            $\bar \eta_{i,j}$ & The weak mixing coefficient & Eq. \ref{def:weak-mixing} & $\phi$ & The uniform mixing coefficient & Eq. \ref{def:uniform-mixing}\\
            $N_s\pow l, N_{s,t}\pow l$ & Count statistics & Eq. \ref{eq:pair-visit-count} & $T$ & Bound on expected return time & Assm. \ref{assume:return_time}\\
            $\|\Delta_m\|$ & sum of $\bar\eta_{i,j}$'s & Assm. \ref{assume:eta-mix} & $\constant_\Delta$ & Bound on $\|\Delta_m\|$ & Assm. \ref{assume:eta-mix} \\
            $\rho_s\pow l$ & $\sup_{1\leq i\leq m}\prob(X_i=s,a_i=l)$ & Prop. \ref{prop:tail-bound-peligrad}& $\rho_\star$& $\sup_{s,l} \rho_{s}\pow l$ & Eq. \ref{sec:minimax-sc}\\
            $\|\cdot\|_\infty^*$ & $\sup_{l\in\Ibb} \| \cdot \|_\infty$ & Sec. \ref{sec:prf-skt-minmax} & & &\\
            \hline
    \end{tabular}
    \caption{Notations}
    \label{tab:index}
\end{table}

\section{Technical desiderata}\label{sec:tech-des}

\subsection{Concentration Inequalities}\label{sec:concentration-ineq}
\paragraph{Hoeffding's concentration inequality for mixing sequences.} 
For two sequences of real numbers $\tilde{x}=\datas$ and $\tilde{y}=\datay$, define the {Hamming metric} $d$ between them as $d(\tilde{x},\tilde{y}) := \sum_i \indicator[x_i\neq y_i]$. 
{Observe that for any two sequences $\tilde{x}$ and $\tilde{y}$ we have, 
\begin{align*}
    |N_s(\tilde{x})-N_s(\tilde{y})| & = \left|\sum_{i=1}^m \indicator[x_i=s]-\sum_{i=1}^m \indicator[y_i=s] \right|\\
    & \leq \sum_{i=1}^m \left| \indicator[x_i=s]-\indicator[y_i=s] \right|\\
    & \leq \sum_{i=1}^m \indicator[x_i\neq y_i]= d(\tilde{x},\tilde{y}).
\end{align*}
Therefore, the function  $N_s(\tilde{x}):=\sum_{i=1}^m \indicator[x_i=s]$ is $1$-Lipschitz in Hamming metric.
This allows us to specialize Theorem 1.1 from \cite{kontorovich2008concentration} to our current setting.}
\begin{lemma}~\label{lemma:kontorovich}
    Let $\indexeddata$ be a sequence of stochastic random variables on a finite state space $\chi\times\Ibb$. 
    Then, for any $t > 0$ we have,
  \begin{align}
      \prob(|N_s\pow{l}-\expec[N_s\pow{l}]|>t)\leq 2\exp\left(-\frac{t^2}{2m\|\Delta_m\|^2}\right),
  \end{align}
where $\|\Delta_m\|=\underset{1\leq i\leq m}{\max} (1+ \Bar{\eta}_{i,i+1}+ \Bar{\eta}_{i,i+2}+\dots \Bar{\eta}_{i,m})$, and $\Bar{\eta}_{i,j}$ is as defined in \cref{def:weak-mixing}.
\end{lemma}
\begin{remark}
It is clear that analysing the tail properties of $N_s\pow{l}$ requires an appropriate concentration inequality.
The result from \citet{kontorovich2008concentration}, as recalled in Lemma \ref{lemma:kontorovich} is a refinement of the celebrated Azuma-Hoeffding's inequality \citep{azuma1967weighted} where the martingale differences are linked to the underlying mixing coefficients of the stochastic process. Observe that the inequality is guaranteed to be as tight as Azuma-Hoeffding's inequality \cite[Theorem 7.5]{kontorovich2008concentration}. However, as we discuss is Remark \ref{remark:exp-mix}, it is not tight enough to achieve minimaxity.
\end{remark}

{   
\paragraph{Chernoff's concentration inequality for geometrically mixing sequences.} We begin this subsection with the following lemma whose proof proceeds by a careful book-keeping of inverses of functions and is deferred to Section \ref{sec:prf-fmxcfd}.
\begin{lemma}~\label{lemma:funcmix-coefbd}
For some $i,j$, let $T_\star\in\lc0,1\rc^{m-j+1}$ and $\history_\star\in\lc0,1\rc^{i+1}$. For the convenience of notation, denote $\lp\indicator[X_i,a_i],\dots \indicator[X_0,a_0]\rp$ by $\indicator(\History_0^i)$. Then,
\begin{small}
\begin{align*}
    \phi_{i,j} & = \sup_{\substack{\Tbb,\history_0^{i-1}}}\left|\prob\left((X_m,a_m,\dots, X_j,a_j)\in \mathbb{T}|\History_0^{i}=\history_0^{i}\right)-\prob\left(X_m,a_m,\dots,X_j,a_j\in\mathbb{T}\right)\right|\\
    & \geq \sup_{\substack{\Tbb_\star,\history_\star}}\left|\prob\left(\indicator[X_m=s,a_m=l],\dots, \indicator[X_j=s,a_j=l])\in \mathbb{T}_\star|\indicator\lp\History_0^{i}\rp=\history_\star\right)\right.\\
    & \quad \left.-\prob\lp\indicator[X_m=s,a_m=l],\dots, \indicator[X_j=s,a_j=l])\in \mathbb{T}_\star\rp\right|.
\end{align*}
\end{small}
\end{lemma}
\begin{remark}
The uniform mixing coefficients of random variables form an upper bound to the uniform mixing sequences of the indicator functions of the same random variables. 
\end{remark}
Using the aforementioned fact, we specialize Theorem 2 from \cite{merlevede2009bernstein}. 
As in the proof of Proposition \ref{prop:tail-bound-peligrad}, we drop s and l, and denote $\indicator[X_i=s,a_i=l]$ by $I_{i}$. Then we have the following lemma.
\begin{lemma}~\label{lemma:peligrad}
Let $(X_m,a_m,\dots,X_0,a_0)$ be a sequence of stochastic random variables on a finite state space $\chi\times\Ibb$. Assume that there exists a positive constant $\constant_{cof}>0$ such that the uniform mixing coefficient $\phi_{i,i+j}$ satisfies 
\[
\sup_{i} \phi_{i,i+j}\leq e^{-j\times \constant_{cof}}.
\]
Then, there exists a constant $\constant_{pel}$ which depends only upon $\constant_{cof}$ such that for all $m\geq 2$
\begin{small}
\begin{equation}
    \prob\lp \lv N_s\pow{l}-\expec[N_s\pow{l}]\rv\geq t \rp\leq \exp\lp-\frac{\constant_{pel}t^2}{m\sup_{i\geq 0}\lp \Var(I_{i}) +2\sum_{j\geq i}|\Cov(I_{i},I_{j})|\rp+1+t(\log m)^2}\rp
\end{equation}
\end{small}

\end{lemma}
\begin{proof} \
The proof follows from \cite[eqn. 2.1]{merlevede2009bernstein} by observing from \citet[eqn. 1.12]{bradley} that uniform mixing coefficients form a natural upper bound to strong mixing coefficients (defined as in \citet[eqn. 1.1]{bradley}), and then observing that for indicator variables $M=1$.
\end{proof}
\begin{remark}
This gives us a cleaner version of the original inequality appropriated to our current setup.
\end{remark}
}
{To control the $\Cov(I_i,I_j)$ terms in Lemma \ref{lemma:peligrad} we adapt the following H\"older's inequality for uniformly mixing sequences by setting $p=1$ and $q=\infty$ in \citet[Theorem A.6]{hall2014martingale}.
\begin{lemma}~\label{lemma:hall-holder's}
For any two real functions $f$, and $g$,
\begin{align*}
    |\Cov(f(X_{j},a_j),g(X_{i},a_{i}))|\leq \phi_{i,j}\expec\lv 
    f(X_{j},a_j)-\expec[f(X_{j},a_j)]\rv\ \mathrm{ess }\sup_{X_i,a_i}\lv g(X_{i},a_i)\rv.
\end{align*}
It follows from Lemma \ref{lemma:mixing-lemm} that,
\begin{align*}
    |\Cov(f(X_{j},a_j),g(X_{i},a_{i}))|\leq \bar\eta_{i,j}\expec\lv 
    f(X_{j},a_j)-\expec[f(X_{j},a_j)]\rv\ \mathrm{ess }\sup_{X_i,a_i}\lv g(X_{i},a_i)\rv.
\end{align*}
\end{lemma}
}
The following lemma serves to decompose the mixing coefficients of a bivariate random variable into the sum of mixing coefficients of the corresponding univariate random variables. Its proof (deferred to Section \ref{sec:prf-tvbd}) proceeds by repeated applications of triangle inequality and carefully taking suprema. Our next technical lemma will be useful in the proof of Lemma \ref{lemma:delta-bound}.
\begin{lemma}~\label{lemma:tv-bound}
Let $X,Y,$ and $Z$ be 3 discrete random variables in the sample space $\Omega$. Then
    \begin{align*}
        \lVert \probl(X,Y|Z=z_1)& -\probl(X,Y|Z=z_2) \rVert_{TV} \leq \lVert \probl(Y|z=z_1)-\probl(Y|Z=z_2) \rVert_{TV}\\
        & \qquad +\sup_y\lVert \probl(X|Y=y,Z=z_1)-\probl(X|Y=y,Z=z_2) \rVert_{TV}.
    \end{align*}
\end{lemma}
Two finite-state Markov chains are of the same \emph{type} if their transition matrices have zeros in the same entries.
The following lemma is found in \citet[Theorem 1]{wolfowitz1963products}.
\begin{lemma}~\label{prop:wolfowitz}
Suppose that $\indexeddata$ is a sample from a time-inhomogeneous Markov chain. Then $\bar\theta_{i,j}\leq e^{-\upsilon|j-i|}$, where $\upsilon$ is a constant depending only on the number of types of matrices in the set $\Mbb$.
\end{lemma}
As a follow-up remark, we note from \citet{wolfowitz1963products} that whether a transition matrix belongs to an ergodic and irreducible Markov chain depends solely on the type. 
The following lemma extends the previous one to the case of controlled Markov chains with non-stationary Markov controls.
\begin{lemma}\label{lemma:markov-assume4}
Let $\indexeddata$ be a sample from a controlled Markov chain with transition matrices $M\pow{l}$ and Markov controls such that for any $s\in\chi$,$l\in\Ibb$, $\history_{0}^{i-1}\in(\chi\times\Ibb)^{i-1}$ and $i\in\naturalset\cup\{0\}$,
\[
P_{s,l}\pow{i}:=\prob\lb a_i=l\gn X_{i}=s,\History_0^{i-1}=\history_0^{i-1}\rb=\prob\lb a_i=l\gn X_{i}=s\rb.
\]
If there exists a $\chi_0\subseteq\chi$ and $M_{min>0}$ such that,
\[
\min_{s\in\chi,l\in\Ibb} M_{s,t}\pow{l}\geq M_{min}, \quad \forall\ t\in\chi_0,
\]
then, for this controlled Markov chain,
\[
\bar\theta_{i,j} \leq \lp1-|\chi_0|M_{min}\rp^{j-i-1}.
\]
As a consequence, this CMC satisfies Assumption \ref{assume:chain-mix} with $\constant_\theta=1/(|\chi_0|M_{min})$.
\end{lemma}
The proof of this Lemma can be found in Section \ref{sec:prf-mkvass4}.
\begin{remark}~\label{remark:pel-discuss}
If the number of types of matrices is independent of $k$, observe it follows from Lemma \ref{prop:wolfowitz} that there exists a class of Inhomogenous Markov chains for which $\constant_{cof}$ is independent of the parameters of the Markov chain, where the definition of $\constant_{cof}$ is as in Lemma \ref{lemma:peligrad}. This naturally implies that $\constant_{pel}$ is a universal constant for that class.

Furthermore, it can be derived from Lemma \ref{lemma:markov-assume4} that whenever $\|\chi_0\|M_{min}$ is independent of $d,k$ there exists a CMC with non-stationary Markov controls, which satisfy a similar property.
\end{remark}

\begin{lemma}~\label{lemma:mm-p1cond}
Let $\indexeddata$ be a sample from a controlled Markov chain belonging to the class of CMC's as defined in \cref{eq:mm-p1eq4}. Then, $\indexeddata$ satisfies Assumptions \ref{assume:sets},\ref{assume:return_time},\ref{assume:control-mixing}, and \ref{assume:chain-mix}. 
\end{lemma}
\begin{proof} \
This CMC has only $1$ control and a single positive transition matrix positive making it a geometrically ergodic Markov chain \cite[Theorem 15.0.1]{meyntweedie}. The rest of the proof follows similarly to the proof of \Cref{prop:stationary-mdp}.
\end{proof}
\begin{lemma}~\label{lemma:counter-example} Let $(X_i,a_i)$ be an inhomogenous Markov chain with $\chi=\lc1,2,4,5\rc$ and $\Ibb=\lc1,2\rc$ such that $\prob(a_0=1)=1$, $\prob(a_1=2)=1$, $\prob(a_2=1)=1$ and so on. Moreover, $a_i$ depends only upon time point $i$ and is independent of $(X_i,\History_0^{i-1})$. The transition matrices are given by
\begin{align*}
    M\pow{1}=\begin{bmatrix}
    1/2 & 1/2 & 0 & 0\\
    1/2 & 1/2 & 0 & 0\\
    0 & 0 & 1/2 & 1/2\\
    0 & 0 & 1/2 & 1/2\\
    \end{bmatrix} & \textit{ and } M\pow{2}=\begin{bmatrix}
    0 & 0 & 1/2 & 1/2\\
    0 & 0 & 1/2 & 1/2\\
    1/2 & 1/2 & 0 & 0\\
    1/2 & 1/2 & 0 & 0\\
    \end{bmatrix}.
\end{align*}
$X_0$ is drawn uniformly from $\{1,2,3,4\}$. Then, for this inhomogenous Markov chain $\bar \theta_{i,j}=1$ for any time points $j>i$.
\end{lemma}
\begin{proof} \ 
We observe that 
\[
\bar \theta_{i,i+1} = \underset{\substack{s_1,s_2\in\chi,l_1,l_2\in \Ibb,\\ \prob(X_i=s_1,a_i=l_1)>0,\\ \prob(X_i=s_2,a_i=l_2)>0  }}{\sup} \| \probl\lp X_j|X_i=s_1,a_i=l_1\rp-\probl\lp X_j|X_i=s_2,a_i=l_2\rp \|_{TV}
\]
is only well-defined if $l_1=l_2$. Let $s_1=1$, and $s_2=3$. Then,
\begin{align*}
    \bar \theta_{i,i+1} \geq  | \prob\lp X_j\in\{1,2\}|X_i=1,a_i=l_1\rp-\prob\lp X_j\in\{1,2\}|X_i=3,a_i=l_1\rp | = 1.
\end{align*}
Thus, $\bar \theta_{i,i+1}=1$. Similarly, $\bar \theta_{i,j}=1$ for any $j>i$ and 
\[
\sup_{1\leq i\leq \infty}\sum_{j=i+1}^\infty \bar \theta_{i,j}=\infty.
\]
\end{proof}

{
\section{Numerical Examples}

In this section we illustrate the validity of our theorems through simulations. We take two positive transition matrices in a time inhomogenous Markov chain and vary the return times and the mixing coefficients. To demonstrate the correctness of the sample complexities recovered by our theorems, we do two simulations. (1) Increase $T$, and (2) Increase $\|\Delta\|$. We then rerun the experiment $100$ times. The y-axis notes the average error along with error bars. The x-axis (scaling factor) denotes a multiple of the number of samples scaled by the sample complexity $10*m/(T\Delta)$ derived from the theory. $\Delta$ was derived using Lemma \ref{lemma:markov-assume4}, and $T$ was derived using Proposition \ref{prop:inhomogenous-mc}. In the left hand side figure, we varied $T$ values as $T, 10T$, and $20T$ and varied $\Delta$ similarly on the right hand side figure. The results demonstrate that estimation errors overlap, indicating the correctness of our theory.

}

\begin{figure}[htbp]
    \centering
    \begin{minipage}{0.48\textwidth}
        \centering
        \includegraphics[width = \textwidth]{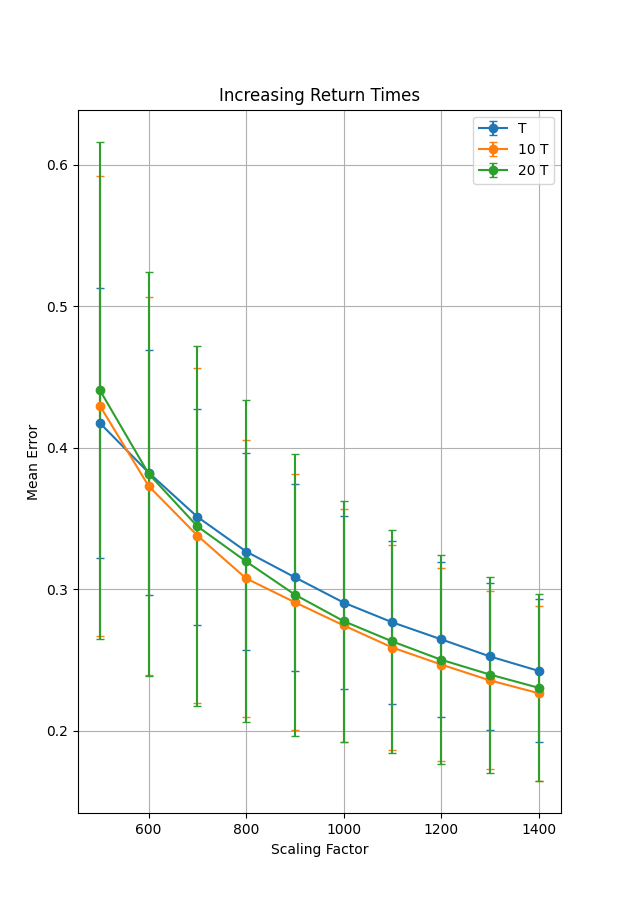}
        \label{fig:return-times}
    \end{minipage}\hfill
    \begin{minipage}{0.49\textwidth}
       \centering
        \includegraphics[width = \textwidth]{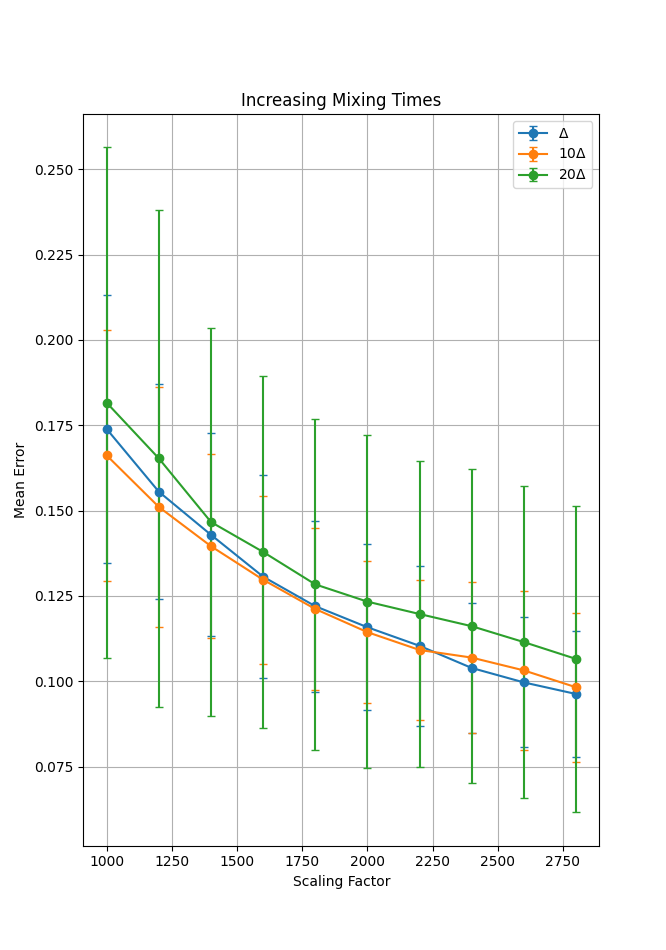}
        \label{fig:mixing-times}
    \end{minipage}
    \caption{\textbf{X-axis:} Scaling Factor. \textbf{Y-axis:} Mean error. \textbf{Legend:} Increasing $T$ or $\Delta$. \textbf{Left:} Plot of estimation error as the return times $T$ increase. \textbf{Right:} The same as $\Delta$ increase (depicted by lines of different colors). Note that the bars represent 2.5\%-97.5\% quantiles of the estimation errors across the 100 simulations.}
\end{figure}

\section{Further Examples}~\label{sec:more-ex}
\subsection{Inhomogenous Markov chains}
A controlled Markov chain is said to be an inhomogenous Markov chain if there exists a sequence of constants $l_0,l_1,\dots$ such that for any non-negative integer $i$, history $\history_0^{i-1}\in(\chi\times\Ibb)^i$, and state $s\in\chi$
\[
\prob(a_i=l_i|\History_0^{i-1}=\history_0^{i-1},X_i=s)=1,
\]
\noindent We make the following assumptions on the controls and the transition probabilities. Let $\Tbb\in \naturalset$ be a known integer. Then, we assume that
\begin{align*}
    \sum_{i=j}^{\Tbb+j} \indicator[a_i=l]>1,\numberthis\label{eq:inhomogenous-eq2}
\end{align*}
for all $1\leq j\leq m-\Tbb$ and $l\in \Ibb$.
Moreover, we assume that the transition matrices $M\pow{l}$ are \textit{positive} stochastic matrices where every entry in the matrix is strictly positive. 
It follows from the fact that 
\begin{align*}
  &\prob\lp X_{i+1}=t|X_i=s,a_i=l\rp>M_{min}:=\min_{s,t,l} M_{s,t}\pow{l}, \quad \text{ and }& (I1)\\
  &\prob\lp X_{i+1}\neq t|X_i=s,a_i=l\rp<1-M_{min}.& (I2)
\end{align*}
\begin{proposition}~\label{prop:inhomogenous-mc}
Let $\indexeddata$ be a sample from an inhomogenous Markov chain satisfying $M_{min}>0$ and \cref{eq:inhomogenous-eq2}. Fix $\epsilon>0$, and $\delta \in (0,1)$. Then \Cref{thm:sample-complexity} holds with
\[
T=\lp1-M_{min}\rp^{1-1/\Tbb}/\lp{1-\lp1-M_{min}\rp^{\frac{1}{\Tbb}}}\rp,  \zeta_2=M_{min}, \zeta_1=M_{max}:=\max_{s,t,l}M_{s,t}\pow{l}, \textit{ and } \constant_{\theta}= e/(e-1).
\]
\end{proposition}
\begin{proof} \
The proof of this result can be found in \Cref{prf:inh-mc}.
\end{proof}
\begin{remark}
If $\{X_n\}$ is Markov, then it is well known that $Y_n=(X_n,\dots,X_{n-j+1})$ is a Markov process on state space $\chi^j$. Marginalising on the first component, conclusions can then be obtained about the distribution of $X_n$ given $X_{n-1},\dots,X_{n-j+1}$. It follows that our conclusions also hold for $j$th-order inhomogenous Markov chains.
\end{remark}
\begin{remark}
Observe that 
\[
\frac{(1-M_{min})^{1-1/\Tbb}}{1-(1-M_{min})^{\frac{1}{\Tbb}}}\leq\frac{1}{1-(1-M_{min})^{\frac{1}{\Tbb}}}\leq \frac{\Tbb}{M_{min}}\leq \frac{\Tbb M_{max}}{\max\{M_{max},1-M_{min}\}(1-\max\{M_{max},1-M_{min}\})}
\]
where the first inequality follows from Bernoulli's inequality \citep{brannan2006first} and the second inequality follows by multiplying the numerator and denominator by $M_{max}$ and some trivial algebra.
Comparing this to Proposition \ref{prop:markov-mdp} consequently implies that tighter upper bounds to expected return times can be derived when the controls are deterministic. In other words, a CMC with Markov non-deterministic controls require more samples to guarantee PAC-bounds than an inhomogenous Markov chain. 
\end{remark}

\subsection{Controlled Markov chains with episodic controls}
{
Controlled Markov chains with episodic controls appear frequently in offline reinforcement learning  \citep{rashidinejad2021bridging,li2022settling,li2022settlingh,yin2021towards} and constitute a fixed horizon $H$ after which the algorithm restarts. To formalize,
}
let $H $ be a fixed positive integer. For any positive integer $i>H $, let $H \pow{i}$ denote the greatest multiple of $H$ less than or equal to $i$. To be precize
\[
H \pow{i}:=H \lfloor\frac{i}{H}\rfloor.
\]
A sample $\indexeddata$ is said to be from a controlled Markov chain with episodic controls and horizon length $H $ if for any $s\in\chi$ and $\history_0^{i-1}\in(\chi\times\Ibb)^i$,
\begin{small}
\begin{align*}
\probl\lp a_i|X_i=s,\History_0^{i-1}=\history_0^{i-1}\rp = \probl\lp a_i|X_i=s, \History_{H ^{(i)}}^{i-1}=\history_{H \pow{i}}^{i-1}\rp,\numberthis\label{eq:episodic-eq2}
\end{align*}    
\end{small}
We assume that for all state-control pair $(t,l')$, there exists a time step $i^{(t,l')}<H $ such that 
\[
\min_{t,l'\in\chi\times\Ibb}\prob\lp X_{i^{(t,l')}}=t,a_{i^{(t,l')}}=l'|X_0=s, a_0=l \rp>0.
\]
for some state control pair $(s,a)$. Since this infimum is taken over finitely many objects there exists a probability $M_{min}$ such that
\[
\prob\lp X_{i^{(s,l)}}=s,a_{i^{(s,l)}}=l|X_0=s', a_0=l' \rp>M_{min}. 
\]
We observe that this assumption is similar to the so called \textit{persistence of excitation} (Assumption 1 in \cite{beck2012error}) and is close to the infinite updates assumption in \cite{yu2013boundedness} and \cite{yu2013q}. Our second technical assumption is that .  
\begin{align*}
    (X_i,a_i)\overset{i.i.d.}{\sim} Uniform{\lc(1,1),\dots,(d,k)\rc}.\numberthis\label{eq:episodic-eq1}
\end{align*}
whenever $i$ is a multiple of $H$.
In other words, multiples of $H$ denote the start of a new episode.
\begin{remark}
We would like to point out that assuming \cref{eq:episodic-eq1} simplifies our calculations. However, our results also hold when the distribution is not uniform. The corresponding calculations are very similar (yet more tedious and notationally challenging). 
\end{remark}
\noindent Let $\zeta_1=M_{max}, \ \zeta_2=M_{min}, T=dkH -1,\ \constant=H ^2,\ \text{ and }\  \constant_{\theta}=H $. We can now state our main result about the sample complexity of a controlled Markov chain with episodic controls.
\begin{proposition}\label{prop:episodic-mdp}
Let $\indexeddata$ be a sample from a controlled Markov chain with episodic controls. Fix $\epsilon>0$, and $\delta \in (0,1)$. Then \Cref{thm:sample-complexity} holds with $\zeta_1=M_{max}, \ \zeta_2=M_{min}, T=dkH -1,\ \constant=H ^2,\ \text{ and }\  \constant_{\theta}=H$.

\end{proposition}
\begin{proof} \ 
The proof of this result can be found in \Cref{prf:ep-mdp}.
\end{proof}

\subsection{Controlled Markov chains with greedy controls}  
Greedy algorithms are an important facet of reinforcement learning appearing in numerous exploration-exploitation like multi-armed bandit \citep{kuleshov2014algorithms}, $Q$-learning \citep{wunder2010classes}, Levy-flights \citep{liu2021improving} among others. In the current context, we define a controlled Markov chain with greedy controls as follows. 
\begin{definition}
Let $\{\omega_i\}_{i\geq 0}$ be a sequence of i.i.d. Bernoulli variables with $\prob(\omega_i=1)=\upsilon$ for some fixed $\upsilon\in(0,1)$. 
A controlled Markov chain with greedy controls is defined to be a triplet $(X_i,\omega_i,a_i)$ taking values in $\chi\times\{0,1\}\times\Ibb$, where $X_i$ represents the state and $a_i$ represents the control at time point $i$ such that
\begin{align*}
   a_i = (1-\omega_i)\alpha_i+\omega_i D_i\pow{1} 
\numberthis\label{eq:greedy-eq5} 
\end{align*}
where $\alpha_i$ are random variables on $\Ibb$ adapted to the history 
and $D_i\pow{1}$ are i.i.d. uniform random variables supported on $\Ibb$. Moreover for some probability $M_{s,t}\pow{l}$ depending only upon $s,t$, and $l$, and for any $\xi\in\{0,1\}$ the transition probabilities satisfy,
\[
\prob\lp X_i=t|X_{i-1}=s,a_{i-1}=l,\omega_{i}=\xi \rp=M_{s,t}\pow{l}
\]

\end{definition}
This completes our formalisation of a controlled Markov chains with greedy controls. 
Intuitively, the definition suggests that greedy controls keep exploring the control space irrespective of the history.
We make the mild assumption that $M\pow{l}$ is an aperiodic and irreducible (ergodic) transition probability matrix for all $l\in\Ibb$ and observe that the $l$-th transition matrix of $(X_i,\omega_i)$ is a block matrix $\mc\pow{l}$ where
\[
\mc\pow{l} = \begin{bmatrix}
(1-\upsilon)M\pow{l} & \upsilon M\pow{l}\\
(1-\upsilon)M\pow{l} & \upsilon M\pow{l}
\end{bmatrix}.
\]
$\upsilon$ is fixed, so we only need to accurately estimate $M\pow{l}$. The proof proceeds by making a suitable transformation on data.
\paragraph{Transformation.} Our objective in this transformation is to carefully isolate all those time points $i$ where $\omega_i=\omega_{i-1}=1$ and use it to create our empirical estimator.
Let $D_i\pow{2}$ be a sequence of i.i.d. uniform random variables distributed on $\chi$ and $D_i\pow{3}$ be the same on $\Ibb$. We construct a sequence of random variables $\tilde X_i, \tilde a_i$ as follows,
\begin{small}
\begin{align*}
     \tilde X_i & =   (1-\omega_i) D_i\pow{2} + \omega_i X_i &(G1) \\
    \tilde a_i & = (1-\omega_i)D_i\pow{3}+ \omega_ia_i. &  (G2)
\end{align*}
\end{small}
From $(G1)$ and $(G2)$ we get,
\begin{small}
\begin{align*}
    & \prob\lp\tilde X_i=t|\tilde a_{i-1}=l, \omega_i=\omega_{i-1}=1 ,\tilde X_{i-1}=s\rp\\
    &\quad =\frac{ \prob\lp\tilde X_i=t, \tilde X_{i-1}=s, \tilde a_{i-1}=l| \omega_i=\omega_{i-1}=1\rp}{\prob\lp\ \tilde X_{i-1}=s, \tilde a_{i-1}=l| \omega_i=\omega_{i-1}=1\rp}\\
    &\quad =\frac{ \prob\lp X_i=t,  X_{i-1}=s,  a_{i-1}=l| \omega_i=\omega_{i-1}=1\rp}{\prob\lp\ X_{i-1}=s,  a_{i-1}=l| \omega_i=\omega_{i-1}=1\rp}\\
    &\quad =  \prob\lp X_i=t| a_{i-1}=l, \omega_i=\omega_{i-1}=1 , X_{i-1}=s\rp \\
    &\quad = M_{s,t}\pow{l}.\numberthis\label{eq:greedy-eq3}
\end{align*}
\end{small}
Therefore, the $l$-th transition probability matrix associated with $(\tilde X_{i+1},\omega_{i+1})$ is
\begin{align*}
\tilde \mc\pow{l} = \begin{bmatrix}
(1-\upsilon)J & \upsilon M\pow{l}\\
(1-\upsilon)J & \upsilon M\pow{l}
\end{bmatrix}\numberthis\label{eq:greedy-eq1}    
\end{align*}
where $J$ is a $d\times d$ matrix with each element $1/d$.
The following lemma holds
\begin{lemma}~\label{lemma:greedy-lem1}
The controlled Markov chain $(Y_i,\tilde a_i)$, where $Y_i:= (\tilde X_i, \omega_i)$ denotes the state at time point $i$, and $\tilde a_i$ denotes the control, is a CMC over $\chi \times\{0,1\}\times \Ibb$ with ergodic transition probability matrix $\tilde \mc\pow{l}$ for all $l\in \Ibb$ and stationary controls $\tilde a_i$. 
\end{lemma}
As an immediate consequence, we apply Proposition \ref{prop:station-prop1} to obtain that the $(\tilde X_i, \omega_i, \tilde a_i)$ is a uniformly ergodic Markov chain.
Let $\nu$ be its stationary distribution on $\chi\times\{0,1\}\times\Ibb$ and for $(s,\xi,l)\in\chi\times\{0,1\}\times\Ibb$, let $\nu_\star := \inf_{s,\xi,l} \nu_{s,\xi,l}$. Furthermore, without losing generality let $\upsilon<(1-\upsilon)$. Next, let $\constant=0$, $\constant_\theta=1/\upsilon$, $T=1/\nu_\star$, $\zeta_2=1/k$, $\zeta_1=1/k$, and by $\hat \mc\pow{l}$ denote the empirical estimator for $\tilde \mc\pow{l}$. Then we have the following theorem,
\begin{theorem}~\label{thm:greedy-policy}
There exists a universal constant $c>1$, such that for any $\epsilon>0$, and $\delta \in (0,1)$,  if
\begin{align*}
  m>c\max\left\{\frac{T}{\eps^2}\log\lp \frac{dkT}{\epsilon^2\delta} \rp, \lp 1+\constant_{\theta}\rp^2\log\lp \frac{dk}{\delta} \rp\max\lc {T}^2, \frac{1}{\lp 1-\max\{\zeta_1,1-\zeta_2\}\rp^2}\rc \right\},
\end{align*} 
then the empirical estimator of $\tilde M\pow{l}$ satisfies,
\begin{align}
  \prob\left(\underset{l\in \Ibb}{\sup} \left\|\Hat{\mc}^{(l)}-\tilde \mc^{(l)}\right\|_{\infty}>\epsilon  \right)< \delta.
\end{align}
\end{theorem}
Applying \Cref{prop:stationary-mdp}, the proof is straightforward. 
The immediate corollary follows.
\begin{corollary}Let $\hat M\pow{l}$ be the $d\times d$ block given by rows and columns $\{d+1\dots,2d\}$ of $\tilde \mc\pow{l}$. Then,
\begin{align}
  \prob\left(\underset{l\in \Ibb}{\sup} \left\|\frac{1}{\upsilon}\hat M\pow{l}- M\pow{l}\right\|_{\infty}>\epsilon  \right)< \delta,
\end{align}
whenever
\begin{align*}
  m>c\max\left\{\frac{ T}{\upsilon^2\eps^2}\log\lp \frac{ dkT}{\upsilon^2\epsilon^2\delta} \rp, \lp 1+\constant_{\theta}\rp^2\log\lp \frac{dk}{\delta} \rp\max\lc {T}^2, \frac{1}{\lp 1-\max\{\zeta_1,1-\zeta_2\}\rp^2}\rc \right\}.
\end{align*} 
\end{corollary}
\begin{proof} \ 
The proof of this follows readily from the fact that $\prob\left(\underset{l\in \Ibb}{\sup} \left\|\hat M\pow{l}- \upsilon M\pow{l}\right\|_{\infty}>\frac{\epsilon}{\upsilon} \right)>\prob\left(\underset{l\in \Ibb}{\sup} \left\|\Hat{\mc}^{(l)}-\tilde \mc^{(l)}\right\|_{\infty}>\frac{\epsilon}{\upsilon}  \right).$
\end{proof}

\section{Minimaxity example}\label{sec:minmax-ex}
{\color{black}
That uniform distribution is the least favourable choice for estimation purposes is well studied know in frequentist and in Bayesian statistics --- Section 3.2 \citet{brandwein1980minimax}, \citet{van2006minimax}, page 340-342, example 3.4 \citet{lehmann2006theory}, \citet{fourdrinier2013bayes}. 
Consequently, the controlled Markov chain which we use to prove minimaxity of the empirical estimator closely mirrors an uniform distribution.

Let $d$ be an integer divisible by 3. Consider a class of controlled Markov chains with $\chi=\{1,\dots,d\}$ and $\Ibb=\lc1,\dots,k\rc$ such that the sequence of controls $a_i$ satisfies,
\begin{align}
    a_i= l \text{ with probability } \frac{1}{k} \ \forall\ l\in\Ibb
    \label{eq:mm-ex2eq1}
\end{align}
Our next step is to construct the transition probability matrices.  \whiteqed
\textbf{Transition Matrices .} Let $\iota$ be a fixed real number between $0$ and $31/64$ and for each $l\in \{1,\dots,k\}$, let $\xi\pow{l}=(\xi_1\pow{l},\dots,\xi_{d/3}\pow{l})$ be some vector in $\lc0,1\rc^{d/3}$. For convenience, assume that $\xi\pow{l}\neq (0,\dots,0)$ for at least some $l\in\Ibb$. Then the $l$-th transition probability matrix $M_{\iota,\xi\pow{l}}\pow{l}$ is a block matrix such that
\begin{align*}
M_{\iota,\xi^{(l)}}^{(l)}
  = 
\begin{bmatrix}
\boldsymbol{C}_{\iota} & \boldsymbol{R}_{\xi^{(l)}} \\
\boldsymbol{J}_{\iota} & \boldsymbol{L}_{\iota}
\end{bmatrix},\numberthis\label{eq:mm-ex2eq2}
\end{align*}
where the blocks
$\boldsymbol{C}_{\iota} \in \Rbb^{d/3\times d/3}$,
$  \boldsymbol{L}_{\iota}  \in\Rbb^{2d/3\times2d/3}$,
$\boldsymbol{J}_{\iota}\in \Rbb^{2d/3\times d/3}$,
and
$\boldsymbol{R}_{\xi^{(l)}} \in\Rbb^{d/3\times 2d/3}$
are given by
\begin{small}
\begin{align*}
&\boldsymbol{L}_{\iota}  = \diag\left(1-\iota,1-\iota,\dots,1-\iota \right),\\
&\boldsymbol{R}_{\xi^{(l)}}  = \frac{1}{2}
\begin{bmatrix}
1 +  \xi^{(l)}_1\eps-2\iota & 1 -  \xi^{(l)}_1\eps-2\iota  & \frac{3\iota}{d-3} & \frac{3\iota}{d-3} & \hdots &  \frac{3\iota}{d-3} \\
\frac{3\iota}{d-3} & \frac{3\iota}{d-3} & 1 + \xi^{(l)}_2 \eps-2\iota & 1-  \xi^{(l)}_2 \eps -2\iota  & \hdots & \frac{3\iota}{d-3} \\
\vdots & \vdots & \vdots  & \vdots & \vdots & \vdots \\
\frac{3\iota}{d-3} & \hdots & \hdots & \hdots & 1 +  \xi^{(l)}_{d/3} \eps -2\iota & 1-  \xi^{(l)}_{d/3} \eps-2\iota\\
\end{bmatrix},
\end{align*}
\end{small}
and, $\boldsymbol{C}_\iota$  and $\boldsymbol{J}_\iota$ are matrices with every element equal to $3\iota/d$. 
It can be verified by summing over the rows that for each $l$, $M_{\iota,\xi\pow{l}}\pow{l}$ is a valid transition probability matrix.
We get the following lemma:
\begin{lemma}~\label{lemma:stationary-dist}
The stationary distribution of a Markov chain with transition probability matrix $M_{\iota,\xi\pow{l}}\pow{l}$ is a block vector $\Pi\pow{l,\iota,\xi\pow{l}}:=\lb \Pi\pow{\iota}_1,\Pi_2\pow{l,\iota,\xi\pow{l}}\rb$ where $ \Pi_1\pow{\iota}$ is a row vector of length $d/3$ and every element $3\iota/d$. Furthermore, 
\begin{small}
\begin{align*}
    \Pi_2\pow{l,\iota,\xi\pow{l}} =  & \lp\frac{ 3(1+\xi_1\pow{l}\eps-\iota)}{2d}\rdot,\ \frac{3(1-\xi_1\pow{l}\eps-\iota)}{2d},\ \vdots\ \ldot\frac{3(1-\xi_{d/3}\pow{l}\eps-\iota)}{2d} \rp.
\end{align*}
\end{small}
\end{lemma}

\noindent The proof follows by verifying $\Pi\pow{l,\iota,\xi\pow{l}}M_{\iota,\xi^{(l)}}^{(l)}=\Pi\pow{l,\iota,\xi\pow{l}}$ and is straightforward. Therefore, we omit it.
\begin{proposition}\label{prop:mm-ex2prop1}
Let $\indexeddata$ be a sample from a controlled Markov chain with controls given by \cref{eq:mm-ex2eq1}, transition probability matrices given by \cref{eq:mm-ex2eq2}, and initial distribution $\Pi = \frac{1}{k}\lp\Pi\pow{1,\iota,\xi\pow{1}},\dots,\Pi\pow{k,\iota,\xi\pow{k}}\rp$. 
Then, it satisfies Assumptions \ref{assume:sets} with $\zeta_1=\zeta_2=\iota$, \ref{assume:return_time} with $T = 2dk/3\iota$, and \ref{assume:control-exp}-\ref{assume:chain-exp} with $\constant_\star$ and $\constant_{\theta,\star}$ independent of $d,k$ or $\xi\pow{l}$ for any $l\in\Ibb$. Furthermore, for this controlled Markov chain, $\rho_\star = 3(1-\iota)/2dk$ where $\rho_\star$ is defined as in \Cref{thm:minimax}.
\end{proposition}
\begin{proof} \ 
It is obvious that the state-control pair $(X_i,a_i)$ is a Markov chain with transition probability matrix given by the block matrix
\begin{align*}
 \frac{1}{k}\begin{bmatrix}
    M_{\iota,\xi^{(1)}}^{(1)}&\dots&M_{\iota,\xi^{(k)}}^{(k)}\\
    \vdots& \vdots & \vdots\\
     M_{\iota,\xi^{(1)}}^{(1)}&\dots&M_{\iota,\xi^{(k)}}^{(k)}
   \end{bmatrix}.
\end{align*}

Taking $\Scal_i=\lc(1,1),\dots,(d/3,1),(2,1),\dots,(d/3,k) \rc$, we can immediately see that Assumption \ref{assume:sets} is satisfied with $\zeta_1=\zeta_2=\iota$. 
Next, let $\Pi$ be the stationary distribution of this Markov chain.
Recall from the proof of Proposition \ref{prop:stationary-mdp} that any controlled Markov chain with stationary Markov controls satisfies Assumption \ref{assume:return_time} with $T$ to be the supremum of the inverse of its stationary probabilities. In other words,
\[
T=\sup_{s,l}\frac{1}{\Pi_{s,l}}.
\]
We simply verify that $\Pi_{s,l}>3\iota/2dk$ for any $(s,l)\in\chi\times \Ibb$.
It is known from Lemma \ref{lemma:stationary-dist} that $\Pi\pow{l,\iota,\xi\pow{l}}$ is the stationary distribution of $M_{\iota,\xi\pow{l}}\pow{l}$.
Using this fact, it can be easily verified that the stationary distribution of the paired process $(X_i,a_i)$ is given by 
\[
\Pi = \frac{1}{k}\lp\Pi\pow{1,\iota,\xi\pow{1}},\dots,\Pi\pow{k,\iota,\xi\pow{k}}\rp.
\]
Recall from hypothesis that $\eps<1/32$. This implies that, for any $\xi\in\lc0,1\rc$ \[1-\xi\eps-\iota>31/32-\iota>\iota\] whenever $\iota<31/64$. 
Thus,
\[
\frac{3(1-\xi\eps-\iota)}{2dk}>\frac{3\iota}{2dk}
\]
Obviously, $3\iota/dk>3\iota/2dk$. Thus, $\Pi_{s,l}>3\iota/2dk$ for any pair $(s,l)\in\chi\times \Ibb$. 
As a consequence, since $\{X_i,a_i\}$ is a stationary Markov chain, it follows that the marginal probabilities
\[
\prob(X_i=s,a_i=l)={\Pi_{s,l}}<\frac{3(1-\xi\eps-\iota)}{2dk}<\frac{3(1-\iota)}{2dk}.
\]
This establishes that $\rho_\star = 3(1-\iota)/2dk$.
Next, because $a_i$'s are distributed uniformly over $\Ibb$, it is obvious that
\begin{small}
\begin{align*}
   \gamma_{p,j,i}
   &= \sup_{s_p,\history_{i+j}^{p-1},\history_0^i}\lV\probl\lp a_p|X_p=s_p,\History_{i+j}^{p-1}=\history_{i+j}^{p-1},\History_0^i=\history_{0}^{i}\rp
   -\probl\lp a_p|X_p=s_p,\History_{i+j}^{p-1}=\history_{i+j}^{p-1}\rp\rV_{TV}\\
   & =0. 
\end{align*}
\end{small}
Consequently, we get that $\constant_\star$ is independent of $d,k,\xi\pow{l}$ and $\constant=0$, where $l\in\Ibb$.
Finally, observe that every controlled Markov chain with stationary controls is also trivially a CMC with non-stationary Markov controls.
Therefore, we can use Lemma \ref{lemma:markov-assume4} with $\chi_0=(1,\dots,d/3)$, and $M_{min}=3\iota/d$ to see that 
\[
\bar\theta_{i,j} \leq \lp 1-\frac{d}{3}\frac{3\iota}{d}\rp^{j-i-1}=(1-\iota)^{j-i-1},
\]
Consequently, we get that $\constant_\star$ is independent of $d,k,\xi\pow{l}$ and $\constant_\theta=\iota$, where $l\in\Ibb$. 
This proves our claims.
\end{proof}
}
\section{The controlled Markov chain Sampling Scheme}~\label{Sec:mdp-sampling}
For each $l \in \Ibb$ and $M^{(l)}\in\Mbb$, create the following infinite array of i.i.d random variables which are also \emph{independent} of the data $\indexeddata$.
\begin{align}
    \Xbb^{(l)} : \begin{array}{ccccc}
        X_{1,1}^{(l)} & X_{1,2}^{(l)} & \dots & X_{1,\tau}^{(l)} & \dots \\
        X_{2,1}^{(l)} & X_{2,2}^{(l)} & \dots & X_{2,\tau}^{(l)} & \dots \\ 
        \dots &\dots &\dots &\dots & \dots \\
        X_{d,1}^{(l)} & X_{d,2}^{(l)} & \dots & X_{d,\tau}^{(l)} & \dots
    \end{array}
\end{align}
where, $\forall \enspace (s,t,\tau)\in \statespace \times \statespace \times \mathbb{N}$, the random variables $X_{s,\tau}^{(l)}$ follow the mass function given by $\prob (X_{s,\tau}^{(l)}=t)=M_{s,t}^{(l)}$.
Moreover, for every time point $i\geq 1$, and $(s_0,l_0,\dots,s_{i-1},l_{i-1},s_i)\in (\chi\times\Ibb)^{i-1}\times \chi$, let, $\alpha_i\pow{s_0,l_0,\dots,s_{i-1},l_{i-1},s_i}$ be independent random variables with support $\Ibb$ and mass function given by,
\begin{align*}
    \prob( \alpha_i\pow{s_0,l_0,\dots,s_{i-1},l_{i-1},s_i}=l)& = \prob(a_i=l|X_i=s_i,\History_0^{i-1}=s_0,l_0,\dots,s_{i-1},l_{i-1})\\
    & =: P_l\pow{s_0,l_0,\dots,s_{i-1},l_{i-1},s_i}.
\end{align*}
The sampling scheme runs as follows: sample $\Tilde{X_0}\sim \initialD$ and set $\tilde a_0\overset{d}{=} a_0$. 
Recursively sample $\Tilde{X}_{i+1}= X_{\Tilde{X_i},\Tilde{N}_{\Tilde{X_i}}^{(i,\tilde a_i)}+1}\pow{\tilde a_i}$ from the array $\Xbb^{(\tilde a_{i})}$ and $\tilde a_{i+1}{=}\alpha_{i+1}\pow{\tilde X_0,\tilde{a_0},\dots,\tilde X_{i+1}}$, where each $i\geq 0$, define $\tilde{N}_{s}\pow{i,l}:=\underset{j\leq i}{\sum}\indicator[\Tilde{X}_j=s,\tilde a_j=l]$ and $\tilde {N}_{s}\pow{m, l_m}=\tilde {N}_{s}\pow{l_m}$. 
This completes the sampling scheme.  
\begin{proposition}\label{prop:mod-sampling-scheme}
$\lp X_0,a_0,\dots,X_m,a_m\rp$ is identically distributed to $\lp \tilde X_0,\tilde a_0,\dots,\tilde X_m,\tilde a_m\rp$.
\end{proposition}
\begin{proof} \ 
Using induction, the proof is straightforward and can be found in \Cref{prf:sampling-scheme}.
\end{proof}

\section{Proofs of Theorems}
\subsection{Proof of Theorem \ref{thm:sample-complexity} (Sample Complexity)}~\label{sec:prf-sampcomp}
\begin{proof} \ 
We start by analysing the event $\lc \underset{l\in \Ibb}{\sup} \|\Hat{M}^{(l)}-M^{(l)}\|_{\infty}>\epsilon \rc$. 
We note that if $\underset{l\in \Ibb}{\sup} \|\Hat{M}^{(l)}-M^{(l)}\|_{\infty}>\epsilon$, then it must be that for at least some $l_0\in \Ibb$, $\|\Hat{M}^{(l_0)}-M^{(l_0)}\|_{\infty}>\epsilon$ and vice versa. Therefore, it follows that
\begin{align*}
    \lc \underset{l\in \Ibb}{\sup} \|\Hat{M}^{(l)}-M^{(l)}\|_{\infty}>\epsilon \rc & = \bigcup_{l=1}^k \lc  \|\Hat{M}^{(l)}-M^{(l)}\|_{\infty}>\epsilon \rc.
\end{align*}
Hence, applying the union bound,

\begin{align}
    \prob\left(\underset{l\in \Ibb}{\sup} \left\|\Hat{M}^{(l)}-M^{(l)} \right\|_{\infty}>\epsilon  \right) & = \prob\left( \bigcup_{l\in\Ibb} \lc  \|\Hat{M}^{(l)}-M^{(l)}\|_{\infty}>\epsilon \rc\right)\\
    & \leq \sum_{l\in \Ibb}\prob\left(\left\|\Hat{M}^{(l)}-M^{(l)}\right\|_{\infty}>\epsilon  \right).
\end{align}
Fix $l\in \Ibb$. Recall the definition of $\hat M\pow{l}(s,\cdot)$ and $\mls$ from \cref{eq:proofsketch-eq1}. 

Using the fact that $\|\cdot\|_{\infty}<\|\cdot\|_1$, we get the following.
\[
    \lc \left\|\Hat{M}^{(l)}-M^{(l)}\right\|_{\infty}> \epsilon \rc \subseteq \bigcup_{s\in \chi}\lc \left\|\hatmls-\mls\right\|_{1}>\epsilon\rc.
\]
It follows from the union bound that

\begin{align*}
    \prob\left(\left\|\Hat{M}^{(l)}-M^{(l)}\right\|_{\infty}>\epsilon  \right) & \leq \sum_{s\in\chi} \prob\left(\left\|\hatmls-\mls\right\|_1 > \epsilon\right).
\end{align*}
Our next objective is to find an upper bound for the probability $\prob\left(\left\|\hatmls-\mls\right\|_1 > \epsilon\right)$ that is independent of $s$.
Fix an $s \in \chi$ and recall the definition of $N_s\pow{l}$ from \Cref{sec:setup}. Using the law of total probability \citep[Proposition 4.1]{gut2005probability}, it follows that
\begin{align*}
    \prob\left(\left\|\hatmls-\mls\right\|_1 > \epsilon\right)=\sum_{n=1}^m \prob\left(\left\|\hatmls-\mls\right\|_1 > \epsilon, N_s^{(l)}=n \right).\numberthis\label{eq:main-thmeq8}
\end{align*}
For any two integers $0 \leq n_{low,s} < n_{high,s} \leq m$, we can decompose the right hand side of \cref{eq:main-thmeq8} into two parts, 
\begin{align*}
     \sum_{n=1}^{m} & \prob\left(\left\|\hatmls-\mls\right\|_1 > \epsilon, N_s^{(l)}=n \right)\\ & = \sum_{n=n_{low,s}}^{n_{high,s}} \prob\left(\left\|\hatmls-\mls\right\|_1 > \epsilon, N_s^{(l)}=n \right)\\
     & \qquad +\sum_{n\notin [n_{low,s} ,n_{high,s}]} \prob\left(\left\|\hatmls-\mls\right\|_1 > \epsilon, N_s^{(l)}=n \right)\numberthis\label{eq:main-thmeq1}.
\end{align*}
We can further decompose the second summation by noting that for each $n$, every summand,
\[
\prob\left(\left\|\hatmls-\mls\right\|_1 > \epsilon, N_s^{(l)}=n \right) \leq \prob\left(N_s^{(l)}=n \right).
\]
Therefore, it follows that
\[
\sum_{n\notin [n_{low,s} ,n_{high,s}]} \prob\left(\left\|\hatmls-\mls\right\|_1 > \epsilon, N_s^{(l)}=n \right)\leq \prob(N_s^{(l)}\notin [n_{low,s} ,n_{high,s}]).
\]
Hence, the right hand side of \Cref{eq:main-thmeq1} is upper bounded by
\begin{align*}
    & \sum_{n=n_{low,s}}^{n_{high,s}} \prob\left(\left\|\hatmls-\mls\right\|_1 > \epsilon, N_s^{(l)}=n \right) +\prob(N_s^{(l)}\notin [n_{low,s} ,n_{high,s}]) \numberthis \label{eq:prob_sum}\\
    & \quad = \text{Term 1 + Term 2},
\end{align*}
We deal with the two terms on the right hand side (RHS) separately. 
\subsubsection*{Term 1.} 
The analysis of the first term follows directly via Proposition \ref{prop:err-bnd}. We get
\begin{small}
\begin{align*}
\sum_{n=n_{low,s}}^{n_{high,s}} \prob\left(\left\|\hatmls-\mls\right\|_1 > \epsilon, N_s^{(l)}=n \right)&\leq m\exp\left({-\frac{n_{low,s}}{2} \max \left\{0,\epsilon-\sqrt{\frac{d}{n_{high,s}}}\right\}^2}\right).
\end{align*}
\end{small}
\subsubsection*{Term 2.}
We begin the analysis of the second term by observing that
\begin{align*}~\label{eq:p1term2}
    \prob\left(N_s^{(l)}\notin [n_{low,s} ,n_{high,s}]\right) & = \prob\left[N_s^{(l)}-\E[N_s^{(l)}]<n_{low,s}-\E[N_s^{(l)}]\right]\\
    & \qquad + \prob\left[N_s^{(l)}-\E[N_s^{(l)}]>n_{high,s}-\E[N_s^{(l)}]\right]\numberthis.
\end{align*}
As long as $n_{high,s}-\E[N_s^{(l)}]>0$, directly applying the upper bound in Lemma \ref{lemma:kontorovich} gives us,
\begin{align}\label{eq:kontorovich_applied_upper}
    \prob \left(N_s^{(l)}-\E[N_s^{(l)}]>n_{high,s}-\E[N_s^{(l)}] \right) & \leq 2\exp\left(-{\frac{\left(n_{high,s}-\E\left[N_s^{(l)}\right] \right)^2}{2m\|\Delta_{m}\|^2}}\right).
\end{align}
Our next step is to select an $n_{high,s}$ such that
\begin{align*}
\left(n_{high,s}-\E[N_s^{(l)}] \right)>0\numberthis\label{eq:main-thmeq5}.    
\end{align*}
Recall from Lemma \ref{lemma:expec-bound} that under our hypothesis,
\[
\E\left[ N_s^{(l)} \right] \leq m \max\{ \zeta_1, 1-\zeta_2\}.
\]
Therefore, by setting $n_{high,s} = m\lp\frac{1+\max\{\zeta_1,1-\zeta_2\}}{2}\rp$, we can ensure that  
\begin{align*}
   \left(n_{high,s}-\E[N_s^{(l)}]\right) =m\lp\frac{1+\max\{\zeta_1,1-\zeta_2\}}{2}\rp-\E\left[ N_s^{(l)} \right] > m\lp\frac{1-\max\{\zeta_1,1-\zeta_2\}}{2}\rp>0. 
\end{align*}
Similarly, by choosing $n_{low,s}=\frac{m}{2T}$ we ensure $n_{low,s} - E\left[ N_{s}^{(l)}\right] < -\frac{m}{2T}$.
Using Lemma \ref{lemma:kontorovich} again we obtain,
\begin{align}\label{eq:kontorovich_applied_lower}
    \prob \left(N_s^{(l)}-\E[N_s^{(l)}]<n_{low,s}-\E[N_s^{(l)}] \right) & \leq 2\exp\left(-{\frac{\left(n_{low,s}-\E\left[N_s^{(l)}\right] \right)^2}{2m\|\Delta_{m}\|^2}}\right).
\end{align}
This completes the analysis of term 2.
\noindent Combining the results from equations \ref{eq:t1bound}, \ref{eq:kontorovich_applied_upper} and  \ref{eq:kontorovich_applied_lower}, we arrive at the following upper bound. 
\begin{align*}
    & \prob\left(\left\|\hatmls-\mls\right\|_1 > \epsilon\right)\\
    &\leq m\exp\left({-\frac{n_{low,s}}{2} \max \left\{0,\epsilon-\sqrt{\frac{d}{n_{high,s}}}\right\}^2}\right) + 2 \exp\left({-{\frac{(n_{high,s}-\E[N_s^{(l)}])^2}{2m\|\Delta_{m}\|^2}}}\right)\\
    & \qquad \qquad + 2\exp\left({-{\frac{(n_{low,s}-\E[N_s^{(l)}])^2}{2m\|\Delta_{m}\|^2}}}\right)\\
    & = \text{A+B+C}.\numberthis\label{eq:main-thmeq7}
\end{align*}
Substituting the values of $n_{high,s}$ and $n_{low,s}$ in A we get, 
\begin{align*}
    & m\exp\left({-\frac{n_{low,s}}{2} \max \left\{0,\epsilon-\sqrt{\frac{d}{n_{high,s}}}\right\}^2}\right) \\
    & \quad = m\exp\left({-\frac{m}{16 T} \max \left\{0,\epsilon-\sqrt{\frac{d}{m\lp\frac{1+\max\{\zeta_1,1-\zeta_2\}}{2}\rp}}\right\}^2}\right).
\end{align*}
Recall that by hypothesis, 
\[
m>\frac{8d}{\epsilon^2\lp1+\max\{\zeta_1,1-\zeta_2\}\rp}.
\] 
This implies that,
\[
\lp\epsilon-\sqrt{\frac{d}{m\lp\frac{1+\max\{\zeta_1,1-\zeta_2\}}{2}\rp}}\rp^2>\epsilon^2\lp1-\frac{1}{2}\rp^2=\frac{\epsilon^2}{4}.
\]
Thus,
\begin{align*}
 \exp\left({-\frac{n_{low,s}}{2} \max \left\{0,\epsilon-\sqrt{\frac{d}{n_{high,s}}}\right\}^2}\right) & \leq \exp\left({-\frac{m\epsilon^2}{64T}}\right).
\end{align*}
This gives us an upper bound for $A$. 

Recall that, we have chosen $n_{high,s}$ such that $n_{high,s}-\expec[N_s\pow{l}]>m\lp\frac{1-\max\{\zeta_1,1-\zeta_2\}}{2}\rp$. Consequently, $\lp n_{high,s}-\expec[N_s\pow{l}]\rp^2>m^2\lp\frac{1-\max\{\zeta_1,1-\zeta_2\}}{2}\rp^2$, and
\begin{align*}
    2 \exp\left({-{\frac{(n_{high,s}-\E[N_s^{(l)}])^2}{2m\|\Delta_{m}\|^2}}}\right) \leq 2 \exp\left({-{\frac{m\lp1-\max\{\zeta_1,1-\zeta_2\}\rp^2}{8\|\Delta_{m}\|^2}}}\right),
\end{align*}
which provides us an upper bound for B.
Recall that $n_{low,s}-\expec[N_s\pow{l}]<-\frac{m}{8T}$. Therefore, $(n_{low,s}-\expec[N_s\pow{l}])^2>\lp\frac{m}{8T}\rp^2$, and
\begin{align*}
    2\exp\left({-{\frac{(n_{low,s}-\E[N_s^{(l)}])^2}{2m\|\Delta_{m}\|^2}}}\right)\leq 2 \exp\left({-{\frac{m}{128T^2\|\Delta_{m}\|^2}}}\right).
\end{align*}
This gives us an upper bound for C.

Returning to \cref{eq:main-thmeq7}, we substitute in the upper bounds for A,B, and C.
Consequently, 
\begin{align*}
     \prob \left(\left\|\hatmls-\mls\right\|_1 > \epsilon\right) & \leq  m\exp\left({-\frac{m\epsilon^2}{64T}} \right)\\
     &\qquad +2 \exp\left({-{\frac{m}{128T^2\|\Delta_{m}\|^2}}}\right)\\
     & \qquad +2 \exp\left({-{\frac{m\lp1-\max\{\zeta_1,1-\zeta_2\}\rp^2}{8\|\Delta_{m}\|^2}}}\right)\\
     & \leq  \frac{64T}{\eps^2}\exp\left({-\frac{m\epsilon^2}{128T}} \right)\\
     &\qquad +2 \exp\left({-{\frac{m}{128T^2\|\Delta_{m}\|^2}}}\right)\\
     & \qquad +2 \exp\left({-{\frac{m\lp1-\max\{\zeta_1,1-\zeta_2\}\rp^2}{8\|\Delta_{m}\|^2}}}\right),
\end{align*}
where the inequality follows from the fact that $xe^{-x}\leq e^{-x/2}$.
Recall that the control/control space satisfies $|\Ibb| = k$ and the state space satisfies $|\chi|=d$. Using a union bound on $l$, we see,
\begin{align*}
     \prob\left(\underset{l\in \Ibb}{\sup} \left\|\Hat{M}^{(l)}-M^{(l)} \right\|_{\infty}>\epsilon  \right) & \leq \sum_{l\in\Ibb}\sum_{s\in \chi} \prob\lp\left\|\hatmls-\mls \right\|_1 > \epsilon\rp\\
    & \leq \sum_{l\in\Ibb}\sum_{s\in \chi}\Bigg(  \frac{64T}{\eps^2}\exp\left({-\frac{m\epsilon^2}{128T}} \right)\\
    &\qquad +2 \exp\left({-{\frac{m}{128T^2\|\Delta_{m}\|^2}}}\right)\\
    & \qquad +2 \exp\left({-{\frac{m\lp1-\max\{\zeta_1,1-\zeta_2\}\rp^2}{8\|\Delta_{m}\|^2}}}\right)\Bigg)\\
    & = dk\Bigg( \frac{64T}{\eps^2} \exp\left({-\frac{m\epsilon^2}{128T}} \right)\\
    &\qquad +2 \exp\left({-{\frac{m}{128T^2\|\Delta_{m}\|^2}}}\right)\\
    & \qquad +2 \exp\left({-{\frac{m\lp1-\max\{\zeta_1,1-\zeta_2\}\rp^2}{8\|\Delta_{m}\|^2}}}\right)\Bigg).\numberthis\label{eq:main-thmeq9}
\end{align*}
\noindent By Assumption \ref{assume:eta-mix}, $\|\Delta_m\|\leq \constant_\Delta$.
Let $\gamma_1,\gamma_2,\gamma_3,\gamma_4$ be four constants such that,
\begin{align*}
    \gamma_1 & = \constant_\Delta^2T^2\log\lp \frac{6dk}{\delta} \rp\\
    \gamma_2 & = \constant_\Delta^2\frac{1}{\lp 1-\max\{\zeta_1,1-\zeta_2\}\rp^2}\log\lp \frac{6dk}{\delta} \rp\\
    \gamma_3 & =\frac{128T}{\eps^2} \log\lp \frac{192dkT}{\epsilon^2\delta} \rp\\
    \gamma_4 & =\frac{d}{\epsilon^2\left({1+\max\{\zeta_1,1-\zeta_2\}}\right)}. 
\end{align*}
Hence, there exists a universal constant $c$ large enough such that if,
\begin{align*}
m>c\lc \gamma_1,\gamma_2,\gamma_3,\gamma_4 \rc,
\end{align*}
then 
\begin{align}
  \prob\left(\underset{l\in \Ibb}{\sup} \|\Hat{M}^{(l)}-M^{(l)}\|_{\infty}>\epsilon  \right)< \delta.
\end{align}
However, recall from \cref{eq:expec-bd-rem1} that $T>\frac{dk}{2}$. 
Therefore, it follows that $\gamma_3>\gamma_4$.
Hence, there exists an universal constant $c$ such that as long as
\begin{align*}
  m>c\max\left\{\constant_\Delta^2\log\lp \frac{dk}{\delta} \rp\max\lc {T}^2, \frac{1}{\lp 1-\max\{\zeta_1,1-\zeta_2\}\rp^2}\rc,\frac{T}{\eps^2}\log\lp \frac{dkT}{\epsilon^2\delta} \rp \right\},
\end{align*} 
\begin{align}
  \prob\left(\underset{l\in \Ibb}{\sup} \|\Hat{M}^{(l)}-M^{(l)}\|_{\infty}>\epsilon  \right)< \delta.
\end{align}
\end{proof}

\subsection{Proof of Sample Complexity in Theorem \ref{thm:minimax}}
{
\begin{proof} \
 The first part of this proof follows similarly to that of \Cref{thm:sample-complexity}. The key difference is to use a tighter Chernoff concentration inequality that is available for exponentially mixing random variables, instead of a weaker Hoeffding's inequality. This produces a tighter sample complexity that we can proceed to prove is minimax.
We proceed until \cref{eq:prob_sum}, and analyse Term 1 similarly as before to get
\begin{equation}
     \sum_{n=n_{low,s}}^{n_{high,s}} \prob\left(\left\|\hatmls-\mls\right\|_1 > \epsilon, N_s^{(l)}=n \right)\leq m\exp\left({-\frac{n_{low,s}}{2} \max \left\{0,\epsilon-\sqrt{\frac{d}{n_{high,s}}}\right\}^2}\right).
\end{equation}
The difference arises in the analysis of Term 2, where, instead of using Proposition \ref{prop:tail-bound-k&R}, we use Proposition \ref{prop:tail-bound-peligrad} to obtain
\begin{align*}
    \text{Term 2} & = \prob\left(N_s^{(l)}\notin [n_{low,s} ,n_{high,s}]\right)\leq 2 \exp\left(- \; \ \frac{\constant_{pel}\lp n_{low,s}-\frac{m}{2T}\rp^2}{4m\constant_\Delta\rho_s\pow{l}+1+\lp \frac{m}{2T}-n_{low,s}\rp\lp\log m\rp^2}  \right)\\
    & + 2 \exp\left(- \; \ \frac{\constant_{pel}\lp n_{high,s}-m\max\{\zeta_1,1-\zeta_2\}\rp^2}{4m\constant_\Delta\rho_s\pow{l}+1+\lp n_{high,s}-m\max\{\zeta_1,1-\zeta_2\}\rp\lp\log m\rp^2}  \right).
\end{align*}
Observe that the right-hand side of the previous equation is increasing in $\rho_s\pow{l}$. 
Furthermore, we also have $\rho_s\pow{l}\leq \rho_\star$. 
Thus, we can replace $\rho_s\pow{l}$ by $\rho_\star$ in the upper bound to get
\begin{align*}
    \text{Term 2} & \leq 2 \exp\left(- \; \ \frac{\constant_{pel}\lp n_{low,s}-\frac{m}{2T}\rp^2}{4m\constant_\Delta\rho_\star+1+\lp \frac{m}{2T}-n_{low,s}\rp\lp\log m\rp^2}  \right)\\
    &\quad + 2 \exp\left(- \; \ \frac{\constant_{pel}\lp n_{high,s}-m\max\{\zeta_1,1-\zeta_2\}\rp^2}{4m\constant_\Delta\rho_\star+1+\lp n_{high,s}-m\max\{\zeta_1,1-\zeta_2\}\rp\lp\log m\rp^2}  \right).
\end{align*}
\noindent Next, we substitute the values $n_{high,s}=m(1+\max\{\zeta_1,1-\zeta_2\})/2$ and $n_{low,s}=m/4T$ into the previous term to get
\begin{align*}
   \text{Term 2} & \leq 2 \exp\left(- \; \frac{1}{16T^2} \frac{\constant_{pel} m^2}{4m\constant_\Delta\rho_\star+1+\frac{m}{4T}\lp\log m\rp^2}  \right)\\
    &\quad + 2 \exp\left(- \; \ \frac{1}{4}\frac{\constant_{pel}m^2\lp 1-\max\{\zeta_1,1-\zeta_2\}\rp^2}{4m\constant_\Delta\rho_\star+1+m\lp \frac{1-\max\{\zeta_1,1-\zeta_2\}}{2}\rp\lp\log m\rp^2}  \right).\numberthis\label{eq:mm-p0eq2}
\end{align*}
We only analyse the first term. The calculations for the second term follow in a similar way. Recall from hypothesis that $m>\constant_T$. Obviously, $m>4T$. In other words, $1<m/T$. 
Substituting $1$ for $m/T$ into the denominator, we get
\begin{align*}
      & 2 \exp\left(- \; \frac{1}{16T^2} \ \frac{\constant_{pel} m^2}{4m\constant_\Delta\rho_\star+1+\frac{m}{4T}\lp\log m\rp^2}  \right)\\
     & \quad \leq  2 \exp\left(- \; \frac{1}{16T^2} \ \frac{\constant_{pel} m^2}{4m\constant_\Delta\rho_\star+\frac{m}{4T}+\frac{m}{4T}\lp\log m\rp^2}  \right)\\
     & \quad = 2 \exp\left(- \; \frac{1}{4} \ \frac{\constant_{pel} m}{16\constant_\Delta\rho_\star T^2+T+T\lp\log m\rp^2}  \right)\\
     & \quad \leq 2 \exp\left(- \; \frac{1}{64} \ \frac{\constant_{pel} \frac{m}{\lp\log m\rp^2}}{\constant_\Delta\rho_\star T^2+2T}  \right).
\end{align*}
The last inequality follows by taking $16(\log m)^2$ common from the denominator and trivially upper bounding $(\log m)^{-2}$ by $1$.
Observe that $\constant_{pel}/64\lp \constant_\Delta\rho_\star T^2+2T \rp$ is less than $1$ and constant in $m$.
Recall from \cref{eq:mm-constantT} that the \emph{inverse of} this term was defined to be $\constant_T$.
It is clear that
$    2dk \exp\left(- \; \frac{m}{\constant_T(\log m)^2}\right) $
is decreasing in $m$. Our objective is to find an $m$ such that 
\begin{align*}
    2dk \exp\left(- \; \frac{m}{\constant_T(\log m)^2}\right) & \leq \frac{\delta}{3},
    \intertext{which is equivalent to finding an $m$ such that}
    \frac{m}{(\log m)^2} & > \constant_T\log\lp\frac{6dk}{\delta}\rp=\constant_{T,\delta}.\numberthis\label{eq:mm-p0eq1}
\end{align*}
Let $m=2\constant_{T,\delta}\lp\log\constant_{T,\delta}\rp^2$. 
The denominator of the term on the left-hand side of the previous equation decomposes into
\begin{align*}
    (\log m)^2 & =\lp\log \lp2\constant_{T,\delta}\lp\log\constant_{T,\delta}\rp^2\rp\rp^2\\
    & = \lp\log 2+\log\constant_{T,\delta}+2\log\log\constant_{T,\delta}\rp^2\\
    & = \lp\log \constant_{T,\delta}\rp^2\lp 1+\frac{\log 2}{\log \constant_{T,\delta}} + 2\frac{\log \log \constant_{T,\delta}}{\log \constant_{T,\delta}}\rp^2.
\end{align*}
Consider the function, 
\begin{align*}
    f(x) = \lp1+\frac{\log 2}{\log x}+2\frac{\log \log x}{\log x}\rp^2.
\end{align*}
This is obviously decreasing in $x$. It can be easily verified that there exists a universal constant `$c$' such that $f(x)<2$ if $x>c$. 
Using this fact and substituting $m$ in \cref{eq:mm-p0eq1}, we get
\begin{align*}
      \frac{m}{(\log m)^2} & = 2\constant_{T,\delta}\frac{\lp\log \constant_{T,\delta}\rp^2}{\lp\log \constant_{T,\delta}\rp^2 f(\constant_{T,\delta})}\\
      & \geq \constant_{T,\delta}.
\end{align*}
Therefore, the first term on the left-hand side of \cref{eq:mm-p0eq2} is upper bounded by $\delta/3$.
We can similarly show that, whenever $m>2\constant_{\zeta,\delta}\lp\log\constant_{\zeta,\delta}\rp^2$ the second term on the right-hand side of \cref{eq:mm-p0eq2} can be upper bounded by $\delta/3$.
We now proceed similarly to \cref{eq:main-thmeq9}. 
This gives us that, under our current hypothesis, 
\begin{align*}
    \prob\left(\underset{l\in \Ibb}{\sup} \|\Hat{M}^{(l)}-M^{(l)}\|_{\infty}>\epsilon  \right)< \delta.
\end{align*}
This completes the proof of the sample complexity. We can now proceed to the proof of Minimaxity.
\end{proof}
\subsection{Proof of Minimaxity in Theorem \ref{thm:minimax}}~\label{sec:proof-scmm}
\begin{proof} \ 
Let $\Mcal_{\chi,\Ibb}$ be the class of all controlled Markov chains on state space $\chi$ with control space $\Ibb$.
We can view an element $\Pcal$ of $\Mcal_{\chi,\Ibb}$ as a doublet $(\Mbb,P)$, where
$\Mbb:=(M\pow{1},\dots,M\pow{k})$ is a collection of distinct $d$-state Markov transition matrix, and $P:=\lp P_1,P_2,\dots, \rp$ is the distribution of control sequences, with each $P_i$ being a probability measure on $\Ibb$ that depends on the history until time point $i$.
Let $\Mcal_\chi$ and $\Mcal_\Ibb$ be the set of all $\Mbb$ and $P$, respectively. 
As before, for $\Mbb_1,\Mbb_2\in\Mcal_\chi$ let  
\[
\lV \Mbb_1-\Mbb_2 \rV_\infty^*=\sup_{l\in\Ibb}\lV M_1\pow{l}-M_2\pow{l}\rV_\infty.
\]

For $\indexeddata \in (\chi\times\Ibb)^m$, a sample of length $m$ from some CMC belonging to $\Mcal_{\chi,\Ibb}$, define an estimation procedure $\hat\Mbb$ 
as the mapping $\hat\Mbb:(\chi\times\Ibb)^m\mapsto \Mcal_\chi$.
We seek to provide a lower bound for the minimax risk over all estimation procedures:
\begin{align*}
\mmrisk &= \inf_{\hat\Mbb} \sup_{(\Mbb,P)\in\Mcal_{\chi,\Ibb}} \prob\lp{\nrm{\hat \Mbb-\Mbb}_\infty^* > \eps}\rp.\numberthis\label{eq:mmrisk}    
\end{align*}
We note that if $\Mcal'\subset\Mcal_{\chi,\Ibb}$ is a subclass of CMC's, then
\begin{align*}
    \mmrisk & \geq \inf_{\hat\Mbb} \sup_{(\Mbb,P)\in\Mcal'} \prob\lp{\nrm{\hat \Mbb-\Mbb}_\infty^* > \eps}\rp.\numberthis\label{eq:mm-p1eq1}
\end{align*}
The rest of the proof proceeds by constructing an appropriate subclass $\Mcal'$.
\subsubsection*{Part 1 ($m < cT/\eps^2$):}
Our example below considers a CMC with stationary Markovian controls. Any such CMC with $k$ transition matrices and $d'$ states
can be viewed as a single Markov chain with $d=d'k$ states. 
For convenience, we use the latter representation. By ${M}$ denote its transition matrix, and by $\hat M$ denote the estimate.
Without losing generality, let there be $d+1$ states where $d$ is even; the odd case is handled similarly. 
Let $0<p_\star<\frac{1}{d+1}$ and for a vector $\sigma=\lp\sigma_1,\dots,\sigma_{\frac{d}{2}}\rp\in\lc -1,1\rc^\frac{d}{2}$, define $\etab(\sigma)$ as the following perturbation of $\left( \frac{1-p_\star}{d},  \frac{1-p_\star}{d},\dots, p_\star \right)$:
\begin{small}
\[
\etab(\sigma) = \left( \frac{1 - p_\star + 16 \sigma_1 \eps}{d}, \frac{1 - p_\star - 16 \sigma_1 \eps}{d}, \dots, \frac{1 - p_\star + 16 \sigma_{\frac{d}{2}} \eps}{d}, \frac{1 - p_\star - 16 \sigma_{\frac{d}{2}} \eps}{d} , p_\star \right).
\]
\end{small}
Since $\epsilon<\frac{1}{32}$ and $d>2$ by hypothesis, it follows that $\etab(\sigma)$ is a valid probability mass function on $\{1,\dots,d+1\}$.
Let $\Mcal_{\sigma}$ be a class of transition matrices indexed by $\sigma$, taking the form
\begin{align*}
 M_{\sigma} & = \begin{pmatrix}
 \frac{1-p_{\star}}{d} & \hdots & \frac{1-p_{\star}}{d} & p_{\star} \\
  \vdots & \vdots & \vdots & \vdots\\
  \frac{1-p_{\star}}{d} & \hdots & \frac{1-p_{\star}}{d} & p_{\star} \\
  \frac{1 - p_\star + 16 \sigma_1 \eps}{d} & \hdots &  \frac{1 - p_\star - 16 \sigma_{\frac{d}{2}} \eps}{d} & p_{\star} 
\end{pmatrix}. \numberthis\label{eq:mm-p1eq5}
\end{align*}
From the Varshamov-Gilbert lemma \citep[Theorem 5.1.7]{van2012introduction},  there exists
$ \Sigma \subset \lc-1,1\rc^{d/2}$, $\abs{\Sigma} = 2^{d/16}$, such that
for $(\boldsymbol{\sigma}, \boldsymbol{\sigma}') \in \Sigma$ with $\boldsymbol{\sigma} \neq \boldsymbol{\sigma}'$,
we have
\[
\sum_{i=1}^{d/2}\indicator[\boldsymbol{\sigma_i}\neq \boldsymbol{\sigma_i'}] \geq \frac{d}{16},
\] 
Define the subclass $\Mcal'$ as
\begin{align*}
    \Mcal':= \lc M_{\sigma}:\sigma\in \Sigma\rc\numberthis\label{eq:mm-p1eq4}
\end{align*}
Recall that by \cref{eq:mm-p1eq1}, that it is enough to find a lower bound on
\[
\inf_{\hat M} \sup_{M\in\Mcal'} \prob\lp{\nrm{M - \hat M}_\infty > \eps}\rp.
\] 
By applying Tsybakov's reduction method \citep[Theorem~2.5]{tsybakov2009introduction} to our problem, we obtain the following lower bound, 
\begin{align*}
    \inf_{\hat M} \sup_{M\in\Mcal'} \prob\lp{\nrm{M - \hat M}_\infty > \eps}\rp \geq \frac{1}{2} \left(1 - \cfrac{2^{2-\frac{d}{16}} \sum_{\boldsymbol{\sigma} \in \Sigma} \mathcal{D}_{\boldsymbol{\sigma},m}}{\log{2^{\frac{d}{16}}}} \right)\numberthis\label{eq:mm-p1eq3},
\end{align*}
where $\mathcal{D}_{\boldsymbol{\sigma},m}$ is the KL-divergence between
$M_0$ and $M_{\sigma}$ (for some $\sigma\in\Sigma$), both viewed as distributions over sequences of length $m$. 
Recall the following chain rule for KL-divergence from \citet[Lemma 6.2]{wolfer2021statistical}
\begin{align*}
   \mathcal{D}_{\boldsymbol{\sigma},m} \leq p_\star m \mathcal{D_{\boldsymbol{\sigma}}},\numberthis\label{eq:mm-p1eq2} 
\end{align*}
where $\mathcal{D_{\boldsymbol{\sigma}}}$ is the KL-divergence between $\etab(\sigma)$ and $\lp\frac{1-p_\star}{d},\dots,\frac{1-p_\star}{d},p_\star\rp$. A direct computation of $\mathcal{D_{\boldsymbol{\sigma}}}$ yields 
\begin{align*}
    \mathcal{D_{\boldsymbol{\sigma}}} & = \sum_{i=1}^{\frac{d}{2}}\lp \frac{1 - p_\star + 16 \sigma_i \eps}{d} \log\lp \frac{\frac{1 - p_\star + 16 \sigma_i \eps}{d}}{\frac{1-p_\star}{d}}\rp+\frac{1 - p_\star - 16 \sigma_i \eps}{d} \log\lp \frac{\frac{1 - p_\star - 16 \sigma_i \eps}{d}}{\frac{1-p_\star}{d}}\rp \rp\\
    & = \frac{d}{2} \lp \frac{1 - p_\star + 16  \eps}{d} \log\lp\frac{1 - p_\star + 16 \eps}{1-p_\star}\rp+\frac{1 - p_\star - 16 \eps}{d} \log\lp \frac{1 - p_\star - 16 \eps}{1-p_\star}\rp \rp.
\end{align*}
Denoting $1-p_\star$ by $A$ and $16\eps$ by $B$ allows us to rewrite the previous equation as
\begin{align*}
\mathcal{D_{\boldsymbol{\sigma}}}=\frac{1}{2} \lp (A+B)\log\lp 1+\frac{B}{A} \rp+ (A-B)\log\lp 1-\frac{B}{A} \rp \rp.
\end{align*}
Observe that $B=16\eps<\frac{1}{2}$ and $A=1-p_\star>1-\frac{1}{d+1}>\frac{1}{2}$. This implies that $\frac{B}{A}<1$. Since $\log\lp 1+x \rp\leq x$ whenever $x>-1$, it follows that
\begin{align*}
\mathcal{D_{\boldsymbol{\sigma}}} & \leq \frac{1}{2} (A+B)\frac{B}{A}-(A-B)\frac{B}{A} = \frac{B^2}{A} = \frac{256\eps^2}{1-p_\star} \leq 512\eps^2.
\end{align*}
Substituting this value in \cref{eq:mm-p1eq2}, and further substituting in that value into \cref{eq:mm-p1eq3}, we obtain
\begin{align*}
    \inf_{\hat M} \sup_{M\in\Mcal'} \prob\lp{\nrm{M - \hat M}_\infty  > \eps}\rp & \geq \frac{1}{2} \left( 1 - \cfrac{2^{2-\frac{d}{16}} \sum_{\boldsymbol{\sigma} \in \Sigma} 512p_\star m\eps^2 }{\log{2^{\frac{d}{16}}}} \right)\\
    & = \frac{1}{2} \left( 1 - \cfrac{2^{2-\frac{d}{16}} |\Sigma| 512p_\star m\eps^2 }{\log{2^{\frac{d}{16}}}} \right).
\end{align*}
Recall that our choice of $\Sigma$ satisfies $|\Sigma|=2^{\frac{d}{16}}$. 
Thus,
\begin{align*}
    \frac{1}{2} \left( 1 - \cfrac{2^{2-\frac{d}{16}} |\Sigma| 512p_\star m\eps^2 }{\log{2^{\frac{d}{16}}}} \right) = \frac{1}{2} \left( 1 - \cfrac{ 32768p_\star m\eps^2 }{d\log{2}} \right).
\end{align*}
Therefore, whenever
$m \leq \frac{d (1 - 2 \delta) \log {2}}{32768 p_\star\eps^2}$,
\begin{align*}
    \inf_{\hat\Mbb} \sup_{(\Mbb,P)\in\Mcal'} \prob\lp{\nrm{\hat \Mbb-\Mbb}_\infty^* > \eps}\rp>\delta.
\end{align*}
To complete the proof, we need to relate the quantities $d$ and $p_*$ to the expected return time $T$ of the process. From the statement of the theorem, we need to show that there exists an universal constant $c_1$ such that $T$ satisfies
\[
T\leq c_1\frac{d}{p_\star}, 
\]
where $T$ is the return time as defined in Assumption \ref{assume:return_time}.
It is easily verifiable that for any $\sigma\in\lc-1,1\rc^{d/2}$ the vector 
\begin{align*}
  \lp\frac{(1-p_\star)^2+\lp1-p_\star+16\sigma_1\eps\rp p_\star}{d},\dots,\frac{(1-p_\star)^2+\lp1-p_\star-16\sigma_{d/2}\eps\rp p_\star}{d},p_\star\rp\numberthis\label{eq:mm-p1eq6}  
\end{align*}
represents the stationary distribution corresponding to the transition probability Matrix $M_\sigma$ as defined in \ref{eq:mm-p1eq5}. We know from Kac's theorem \citep[Theorem 10.2.2]{meyntweedie} that the expected return time to any state is the inverse of its stationary probability. It follows from \cref{eq:mm-p1eq6} that the expected return times are
\[
\lp\frac{d}{(1-p_\star)^2+\lp1-p_\star+16\sigma_1\eps\rp p_\star},\dots,\frac{d}{(1-p_\star)^2+\lp1-p_\star-16\sigma_{d/2}\eps\rp p_\star},\frac{1}{p_\star}\rp.
\]
Since 
$
\frac{d}{(1-p_\star)^2+\lp1-p_\star+16\sigma_1\eps\rp p_\star}\leq \frac{d}{(1-p_\star)^2}
$
for any value of $\sigma_1\in\lc1,-1\rc$ it follows that
\begin{align*}
& \max \lc\frac{d}{(1-p_\star)^2+\lp1-p_\star+16\sigma_1\eps\rp p_\star},\dots,\frac{d}{(1-p_\star)^2+\lp1-p_\star-16\sigma_{d/2}\eps\rp p_\star},\frac{1}{p_\star}\rc \\
& \ \leq \max\lc\frac{d}{(1-p_\star)^2},\frac{1}{p_\star}\rc.
\end{align*}
Since $p_\star<1/(d+1)$, it follows that 
\[
\max\lc \frac{d}{(1-p_\star)^2},\frac{1}{p_\star}\rc<\frac{d}{p_\star}.
\]
Therefore, setting $c_1=1$, it follows that
whenever,
$m \leq \frac{T (1 - 2 \delta) \log {2}}{32768 \eps^2}$,
\begin{align*}
    \inf_{\hat\Mbb} \sup_{(\Mbb,P)\in\Mcal'} \prob\lp{\nrm{\hat \Mbb-\Mbb}_\infty^* > \eps}\rp>\delta.
\end{align*}
This completes the proof of the first part.
\subsubsection*{Part 2 $\lp m <\lc 2\ \constant_{T,\delta}\lp\log\constant_{T,\delta}\rp^2, 2\ \constant_{\zeta,\delta}\lp\log \constant_{\zeta,\delta}\rp^2\rc\rp$:}
\paragraph{Case I: $m < 2\ \constant_{T,\delta}\lp\log\constant_{T,\delta}\rp^2$.}

In this part, we prove that there exists a subclass $\Mcal' \subset \Mcal_{\chi,\Ibb}$ and an universal constant $c>0$ for which $\mmrisk\geq 1/(2+2\pi^2)$ whenever
\[
m< c\lp1+\constant+\constant_{\theta}\rp^2\max\lc {T}^2, \frac{1}{\lp 1-\max\{\zeta_1,1-\zeta_2\}\rp^2}\rc.
\]
For an CMC with $d$ states and $k$ transition matrices, define the random variable $\Tbb$ to be the first time all of the states $1,\dots,d/3$ were visited in all of the $k$ transition matrices.
That is,
\begin{align*}
    \Tbb=\min\lc n\geq 0: \bigcap_{s\in\lc1,\dots,d/3\rc,l\in\{1,\dots,k\}}\lc\bigcup_{k=0}^n\lc X_k=s,a_k=l\rc\rc \neq \emptyset\rc.\numberthis\label{eq:mm-p2eq1}
\end{align*}
Then, we can further lower bound $\mmrisk$ as
\begin{align*}
      \mmrisk & \geq\inf_{\hat\Mbb} \sup_{(\Mbb,P)\in\Mcal'} \prob\lp{\nrm{\hat \Mbb-\Mbb}_\infty^* > \eps}\gn\Tbb > m\rp
  \prob\lp{\Tbb > m}\rp.
\end{align*}
Our next objective is to produce a subclass $\Mcal'$.
Define $P\pow{0}$ to be the sequences of probability mass function on controls as in \cref{eq:mm-ex2eq1}. Let $\Hbb_{\iota}$ be the set of all $k+1$ tuples $(M_{\iota,\xi\pow{1}}\pow{1},\dots,M_{\iota,\xi\pow{k}}\pow{k})$ where $M_{\iota,\xi\pow{1}}\pow{1},\dots,M_{\iota,\xi\pow{k}}\pow{k}$ are matrices as defined in \cref{eq:mm-ex2eq2}. To be precise, 
\begin{align*}
    \Hbb_\iota:=\lc\lc M_{\iota,\xi^{(l)}}^{(l)}:l\in \lc1,\dots,k\rc\rc:(\xi^{(1)},\dots,\xi\pow{k})\in\lc0,1\rc^{d/3\times k}\rc\numberthis
\label{eq:Heta}.
\end{align*}
Set $\Mcal'=\Hbb_{\iota}\times\lc P\pow{0}\rc$. As a consequence, we get the following lemma whose proof can be found in Section \ref{sec:prf-cvrtm}.
\begin{lemma}\label{lemma:cover-time}
Let $\Tbb$ be the time to visit the state-control pairs $\lc(1,1),\dots,(d/3,1),(2,1),\dots,(d/3,k) \rc$ of an CMC belonging to class $\Hbb_\iota\times \lc P\pow{0}\rc$ as defined in \cref{eq:Heta}.
If $n<\frac{dk}{6\iota}\log\lp\frac{dk}{3}\rp$, then 
\[
\prob(\Tbb>n)\geq \frac{1}{1+\pi^2}.
\]
\end{lemma}
Substituting it in the lower bound to $\mmrisk$ gives 
\begin{align*}
\mmrisk & \geq\inf_{\hat\Mbb} \sup_{(\Mbb,P)\in \Hbb_\iota\times \{P\pow{0}\}} \prob\lp{\nrm{\hat \Mbb-\Mbb}_\infty^* > \eps}\gn\Tbb > m\rp
  \prob\lp{\Tbb > m}\rp.
\end{align*}
An application of Lemma~\ref{lemma:cover-time} and Proposition~\ref{prop:mm-ex2prop1} implies that whenever $m\leq \constant_\tau \rho_\star T^2\log T$ for some universal constant $\constant_\tau$
\[
 \prob\lp{\Tbb > m}\rp\geq \frac{1}{1+\pi^2}.
\]
Consequently, we get, 
\begin{align*}
     \mmrisk &\geq \frac{1}{1+\pi^2} \inf_{\hat\Mbb} \sup_{(\Mbb,P)\in \Hbb_{\iota}\times\lc P\pow{0}\rc} \prob\lp{\nrm{\hat \Mbb-\Mbb}_\infty^* > \eps}\gn\Tbb > m\rp. 
\end{align*}
Next, let $l_0$ be any control. 
By definition of $\nrm{\cdot}_\infty^*$, it holds that 
\[
\nrm{\hat \Mbb-\Mbb}_\infty^*\geq\nrm{\hat M\pow{l_0} - M_{\iota, \xi\pow{l_0}}^{(l_0)}}_\infty,
\]
We recall from the construction in \cref{eq:mm-ex2eq2} that $\iota$ is known. Therefore, to correctly estimate the transition matrix $M_{\iota,\xi\pow{l_0}}\pow{l_0}$ we only need to correctly estimate $\xi\pow{l_0}$. 
Combining these facts we get, 
\begin{align*}
    & \inf_{\hat\Mbb} \sup_{(\Mbb,P)\in \Hbb_{\iota}\times\lc P\pow{0}\rc} \prob\lp{\nrm{\hat \Mbb-\Mbb}_\infty^* > \eps}\gn\Tbb >  m\rp \\
    &\ \geq \inf_{\hat M\pow{l_0}} \sup_{\xi\pow{l_0}\in\lc0,1\rc\pow{d/3}} \prob\lp{\nrm{\hat M\pow{l_0} - M_{\iota, \xi\pow{l_0}}^{(l_0)}}_\infty > \eps}\gn\Tbb > m\rp.
\end{align*}
We note that whenever $\xi_1\pow{l_0}\neq\xi_2\pow{l_0}\in
\lc 0, 1\rc^{d/3}$, we have $
\nrm{M_{\iota, \xi_1\pow{l_0}}^{(l_0)} - M_{\iota, \xi_2\pow{l_0}}^{(l_0)}}_\infty = 2\eps.
$
For any estimate $\hat M\pow{l_0}$ define
$
\xi^\star = \argmin_{\xi}\nrm{\hat M^{(l_0)} - M_{\iota, \xi}^{(l_0)}}_\infty.
$
Then for $\xi^{(l_0)}\neq\xi^\star$ we have 
\begin{equation*}
2\eps  = \nrm{M_{\iota, \xi^{(l_0)}}^{(l_0)} - M_{\iota, \xi^\star}^{(l_0)}} _\infty \leq  \nrm{M_{\iota, \xi^{(l_0)}}^{(l_0)} - \hat M^{(l_0)} }_\infty + \nrm{\hat M^{(l_0)} - M_{\iota, \xi^\star}^{(l_0)}}_\infty \leq 2 \nrm{M_{\iota, \xi^{(l_0)}}^{(l_0)} - \hat M^{(l_0)} }_\infty.
\end{equation*}
Therefore, 
$\lc l_0 : \xi^\star \neq \xi\pow{l_0} \rc \subset \lc l_0: \nrm{M_{\iota, \xi\pow{l_0}}^{(l_0)} - \hat M^{(l_0)} }_\infty \geq \eps \rc$ and $\mmrisk$ can be further lower bounded by 
\begin{align*}
  \mmrisk &\geq
 \frac{1}{1+\pi^2} \inf_{\hat M\pow{l_0}} \max_{\xi\pow{l_0}\in\lc0,1\rc^{d/3}}
  \prob\lp\xi^\star \neq \xi\pow{l_0} \gn \Tbb > m\rp \\
	&=
  \frac{1}{1+\pi^2} \inf_{
\hat \xi
  } \max_{\xi\pow{l_0}\in\lc0,1\rc^{d/3}}
   \prob\lp{ \hat\xi \neq \xi\pow{l_0} \gn \Tbb > m}\rp,
\end{align*}
    where $\hat\xi$ any estimate of $\xi^*$ $(X_0,a_0,\dots,X_m,a_m)\mapsto \{0,1\}^{d/3}$. We now observe that 
that the events $\lc N_s\pow{l_0}=0 \text{ for some } l_0\in \Ibb \text{ and some  } s\in \chi\rc$ 
and $\lc\Tbb > m\rc$ are equivalent.
Therefore,
\begin{align*}
   \prob\lp{ \hat\xi \neq \xi\pow{l_0} \gn \Tbb > m}\rp=\prob\lp{ \hat\xi \neq \xi\pow{l_0} \gn N_s\pow{l_0}=0 }\rp.
\end{align*}
When $N_s\pow{l_0}=0$, the estimate $\hat \xi$ is equivalent to choosing uniformly over all possible $\xi\pow{l_0}$. 
Since there are $2^{d/3}$ many possible choices for $\xi\pow{l_0}$, the probability of choosing incorrectly is $1-1/2^{d/3}$. 
We get as a consequence that,
\begin{align*}
     \inf_{
\hat \xi
  } \max_{\xi\pow{l_0}\in\lc0,1\rc^{d/3}}
   \prob\lp{ \hat\xi \neq \xi\pow{l_0} \gn \Tbb > m}\rp\geq 1-\frac{1}{2^{d/3}}>\frac{1}{2}.
\end{align*}
In conclusion, whenever $m\leq \constant_{\tau}T\log T$, \[
\inf_{\hat\Mbb} \sup_{(\Mbb,P)\in\Mcal_{\chi,\Ibb}} \prob\lp{\nrm{\hat \Mbb-\Mbb}_\infty^* > \eps}\rp \geq \frac{1}{2}{\color{black}\frac{1}{1+\pi^2}}.
\]
\paragraph{Case II: $m < 2\ \constant_{\zeta,\delta}\lp\log \constant_{\zeta,\delta}\rp^2$.}
For the final case, we must now show that for our chosen class of CMC's, ${1}/\lp 1-\max\{\zeta_1,1-\zeta_2\}\rp^2$ must lie in a fixed interval for appropriate choices of $\zeta_1$ and $\zeta_2$. Recall that by Proposition \ref{prop:mm-ex2prop1}, $\zeta_1=\zeta_2=\iota$, which is a known constant. The rest of the argument follows.

Next, recall from Proposition \ref{prop:mm-ex2prop1} that $\constant_\star$ and $\constant_{\theta,\star}$ are independent of $d$ and $k$. Since $\iota$ is a known constant, this implies that $\constant_{pel}$ is an universal constant. This completes the proof.
\end{proof}
}
\subsection{Proof of \Cref{thm:value-thm}}~\label{prf:value-thm}
\begin{proof}\ 

For convenience we drop the $\Pi$ in the subscript of $V_\Pi$ and $\hat V_\Pi$. It can be easily seen that $\Pi^T M$ is a stochastic matrix. 
Therefore, without loss of generality, we set $k=1$ and write $\Pi^TM$ simply as $M$, $V = \lp I -  \alpha_{dis} M\rp^{-1}g$  and  $\hat V = \lp I -  \alpha_{dis} \hat M\rp^{-1}g$. 
We now establish that $\Pi^T M$ is a transition matrix which satisfies Assumptions \ref{assume:sets}, and \ref{assume:eta-mix} with constants $\zeta_1=\zeta_2=T^{-1}, \text{ and } \constant_\theta =T/dk$ and Assumption \ref{assume:return_time} with constant $T$. Setting $\Scal_i = \{(1,1)\}$ we get that $\prob\lp (X_i,a_i)\in\Scal_i\rp>T^{-1}$. Consequently, can set $\zeta_1$, $\zeta_2=T^{-1}$ respectively. Since by the assumption of Theorem \ref{thm:value-thm} $M$ is a positive matrix with minimum element $1/T$ we get from Lemma \ref{lemma:markov-assume4} that $\constant_\theta=T/dk$. Finally, similarly to the proof of Proposition \ref{prop:mm-ex2prop1} we get that $M$ satisfies Assumption \ref{assume:return_time} with constant $T$.

We define $\hat V$ as the plug-in estimate of $V$ by substituting $M$ by its estimate $\hat M$.
The rest of proof proceeds by using known perturbation equalities of matrices to bound $\lV V-\hat V\rV_{\infty}$. Observe the following variation of the Woodbury matrix identity from eq. 2.3 \cite{wei2005note}
\begin{align*}
    B^{-1}-A^{-1} = A^{-1}(A-B)B^{-1}
\end{align*}
where $A$ and $B$ are matrices of appropriate dimension.
Let $\lV \cdot \rV_{op}$ be the operator norm on $d\times d$ matrices as defined in Section 4.2 \cite{zhang2011matrix}. 
Using the facts that operator norm is sub-multiplicative and $\lV \cdot \rV_{\infty}\leq \lV \cdot \rV_{op}\leq \sqrt{d}\lV \cdot \rV_{\infty}$, we have that
\begin{align*}
    \lV B^{-1}-A^{-1} \rV_{\infty}\leq \lV B^{-1}-A^{-1} \rV_{op} =  \lv A^{-1}(A-B)B^{-1}\rV_{op} & \leq \lV A^{-1} \rV_{op}\lV B^{-1} \rV_{op}\lV A-B \rV_{op}\\
    & \leq \lV A^{-1} \rV_{op}\lV B^{-1} \rV_{op}\sqrt{d}\lV A-B \rV_{\infty}.
\end{align*}
Substitute $A=\lp I -  \alpha_{dis} \hat M\rp$, $B=\lp I -  \alpha_{dis} M\rp$.
It is well known that the eigenvalues of stochastic matrices lie between $[-1,1]$. 
As a consequence, the eigenvalues of $A$ and $B$ are at least $1-\alpha_{dis}$ and therefore, $\lV A^{-1} \rV_{op}, \lV B^{-1} \rV_{op}\leq \lp 1-\alpha_{dis}\rp^{-1}$. Therefore,
\begin{align*}
    \lV B^{-1}-A^{-1} \rV_{\infty}\leq (1-\alpha_{dis})^{-2} \sqrt{d}\lV A-B \rV_{\infty}.
\end{align*}
Next, observe that, 
\begin{align*}
    \lV V-\hat V \rV_{\infty} & = \lV \lp\lp I -  \alpha_{dis} \hat M\rp^{-1}-\lp I -  \alpha_{dis} M\rp^{-1}\rp g\rV_{\infty}\\
    & \leq \lV \lp\lp I -  \alpha_{dis} \hat M\rp^{-1}-\lp I -  \alpha_{dis} M\rp^{-1}\rp \rV_{\infty}\lV g\rV_{1}\\
    &\leq (1-\alpha_{dis})^{-2}\sqrt{d}\lV \lp\lp I -  \alpha_{dis} \hat M\rp-\lp I -  \alpha_{dis} M\rp\rp \rV_{\infty}\lV g\rV_{1}\\ 
    & \leq \frac{\alpha_{dis}}{(1-\alpha_{dis})^2}\sqrt{d}\lV  \hat M - M \rV_{\infty}\lV g\rV_{1}.\numberthis~\label{eq:policy-cont}
\end{align*}
It follows from \cref{thm:sample-complexity} that,
\begin{align*}
    \prob\lp \lV V-\hat V \rV_{\infty} >\eps \rp\leq & \prob \lp \frac{\alpha_{dis}}{(1-\alpha_{dis})^2}\sqrt{d}\lV  \hat M - M \rV_{\infty}\lV g\rV_{1}>\eps \rp\\ 
    &  = \prob \lp \lV  \hat M - M \rV_{\infty}>\frac{(1-\alpha_{dis})^2}{\lV g\rV_{1}\sqrt{d}\alpha_{dis}}\eps \rp\\ 
    & \leq \delta
\end{align*}
whenever 
\begin{small}
\[ 
m>c \; \max\left\{\frac{T_\alpha}{\eps^2}\log\lp\frac{dkT_\alpha}{\epsilon^2\delta} \rp, \constant_\theta^2\max\lc {T}^2, \frac{1}{\lp 1-\max\{\zeta_1,1-\zeta_2\}\rp^2}\rc \log\lp \frac{dk}{\delta} \rp\right\}
\]
\end{small}
\end{proof}

\subsection{Proof of Theorem \ref{thm:opt-pol1}}~\label{sec:prf-optpol1}

\begin{proof}\ 
The proof of this Theorem follows by using Theorem \ref{thm:value-thm}. It is obvious that $V_{\Pi_{opt}}(s) - V_{\hat \Pi_{opt}}(s)>0$ for every state $s$. We show that $\lV V_{\Pi_{opt}} - V_{\hat \Pi_{opt}}\rV_{\infty}<2\eps$ with high probability. We have,
\begin{align*}
    V_{\Pi_{opt}} - V_{\hat \Pi_{opt}} & = V_{\Pi_{opt}} - V_{\hat \Pi_{opt}}-\hat V_{\Pi_{opt}} + \hat V_{\Pi_{opt}}-\hat V_{\hat \Pi_{opt}} +\hat V_{\hat \Pi_{opt}}\\
    & = V_{\Pi_{opt}} - \hat V_{\Pi_{opt}} +\hat V_{\hat \Pi_{opt}} - V_{\hat \Pi_{opt}} +  \hat V_{\Pi_{opt}}-\hat V_{\hat \Pi_{opt}}\\
    & \leq V_{\Pi_{opt}} - \hat V_{\Pi_{opt}} +\hat V_{\hat \Pi_{opt}} - V_{\hat \Pi_{opt}},
\end{align*}
    where the final equality follows from the fact that $\hat V_{\Pi_{opt}}(s)-\hat V_{\hat \Pi_{opt}}(s)<0$ for each state $s$, because $\hat \Pi_{opt}$ is the optimal policy for maximising the estimated value function. 
    It follows as a corollary to Theorem \ref{thm:value-thm} that there exists a universal constant $c>1$ such that, for any policy $\Pi$, $\prob\lp\lV \hat V_\Pi - V_\Pi\rV_\infty>\eps\rp<\delta
$ if
\[ 
m>c \; \max\left\{\frac{T_\alpha}{\eps^2}\log\lp\frac{dkT_\alpha}{\epsilon^2\delta} \rp, \constant_\theta^2\max\lc {T}^2, \frac{1}{\lp 1-\max\{\zeta_1,1-\zeta_2\}\rp^2}\rc \log\lp \frac{dk}{\delta} \rp\right\}.
\]
This proves that $\lV V_{\Pi_{opt}} - V_{\hat \Pi_{opt}}\rV_{\infty}<2\eps$ with high probability, thereby completing the proof.
\end{proof}

\section{Proofs of Applications}
\subsection{Proof of \Cref{prop:stationary-mdp}}~\label{prf:sta-mdp}
\begin{proof} \ 

The proof proceeds by verifying Assumptions \ref{assume:sets}, \ref{assume:return_time},  \ref{assume:control-mixing}, and \ref{assume:chain-mix}. 

\noindent Assumption \ref{assume:sets}: To verify Assumption \ref{assume:sets}, we recall that $P_s\pow{l}>0$ and $P_{min}=\min_{s,l} P_s\pow{l}$. 
Since the minimisation is over finitely many positive quantities, it must follow that $P_{min}>0$. For any $l_0\in\Ibb$, it follows that 
\[
P_s\pow{l_0}=1-\sum_{l\neq l_0}P_s\pow{l}<1-(k-1)P_{min}.
\]
Recall that to satisfy Assumption \ref{assume:sets}, it is enough to produce sets $\Scal_i$ such that
\[
\zeta_2<\prob\lp (X_i,a_i)\in \Scal_i \rp<\zeta_1 
\]
for some probabilities $\zeta_2$ and $\zeta_1$.
We observe by setting $\Scal_i=\arg\sup_{}$ that $\prob\lp (X_i,a_i)\in \Scal_i \rp=\prob\lp a_i=1 \rp$. It then follows, 
\[
P_{min}<\prob\lp (X_i,a_i)\in \Scal_i \rp<1-(k-1)P_{min}.
\]
and Assumption \ref{assume:sets} is satisfied for $\zeta_2=P_{min}$ and $\zeta_1=1-(k-1)P_{min}$.

\noindent Assumption \ref{assume:return_time}: Next, we proceed to the verification of {Assumption \ref{assume:return_time}}. As discussed in \cref{eqn:return-timered}, we only need to show that for some $T>0$,
\[
\sup_{s,l}\E[\tau^{(1)}_{s,l}|X_0,a_0] < T \text{ almost everywhere.}
\]
Recall from Proposition \ref{prop:station-prop1} that $(X_i,a_i)$ is a time homogenous uniformly ergodic Markov chain. It follows from KAC's theorem \citep[Theorem 10.2.2]{meyntweedie} that for any positive integer $i$,
\[
\E[\tau^{(i)}_{s,l}]=\frac{1}{\nu_{s,l}^{(x,a)}}.
\] 
Setting $T=\sup_{s,l} {1}/{\nu_{s,l}^{(x,a)}}$ then completes the verification of Assumption \ref{assume:return_time}. 

\noindent Assumption \ref{assume:control-mixing}: Next, we verify that an MDP with stationary controls satisfy {Assumption \ref{assume:control-mixing}.} Recall from \cref{eq:stationary-eq1} the definition of stationary controls.
\[
\probl\lp a_i|X_i=s_i,\History_0^{i-1}=\history_0^{i-1}\rp=\probl\lp a_i|X_i=s_i \rp=\probl\lp a_1|X_1=s_i \rp.
\]
Which consequently implies that
\begin{align*}
    &\probl\lp a_p|X_p=s_p,\History_{i+j}^{p-1}=\history_{i+j}^{p-1},\History_0^i=\history_{0}^{i}\rp-\probl\lp a_p|X_p=s_p,\History_{i+j}^{p-1}=\history_{i+j}^{p-1}\rp\\
    & \ = \probl\lp a_1|X_1=s_i \rp - \probl\lp a_1|X_1=s_i \rp\\
    & \ = 0.
\end{align*}
Therefore $\gamma_{p,j,i}=0$ and setting $\constant=0$ completes the verification of Assumption \ref{assume:control-mixing}. 

\noindent Assumption \ref{assume:chain-mix}: Next, we verify that a MDP with stationary controls satisfy {Assumption \ref{assume:chain-mix}.} It follows from \citet[Theorem 16.0.2]{meyntweedie} that any aperiodic time homogenous Markov chain on a finite state space with a single communicating class is uniformly ergodic. \citet[Theorem 3.4]{mukhamedov2013dobrushin} establishes the equivalence of weak and uniform ergodicity for time homogenous Markov chains with the ergodic coefficient between $0$ and $1$. This completes the verification of Assumption \ref{assume:chain-mix}.

Lastly, let $(s_0,l_0):=\arg\sup_{s,l} 1/\nu_{s,l}$ and observe that if $\initialD=\nu$, then
\[
\prob\lp(X_i,a_i)=(s_0,l_0)\rp=\frac{1}{\nu_{s_0,l_0}}=\frac{1}{T}=\zeta_1=\zeta_2.
\]
This completes the proof.
\end{proof}
\subsection{Proof of \Cref{prop:inhomogenous-mc}}~\label{prf:inh-mc}
\begin{proof} \

The proof proceeds by verifying Assumptions \ref{assume:sets}, \ref{assume:return_time},  \ref{assume:control-mixing}, and \ref{assume:chain-mix}.

Assumption \ref{assume:sets}: Set $\Scal_i = (1,\Ibb)$ and observe that
\begin{align*}
    \prob\lp(X_i,a_i)\in\Scal_i\rp & = \prob\lp X_i=1\rp\\
    & = \expec\lb  \prob\lp X_i=1|X_{i-1},a_{i-1}\rp \rb\\
    & = \expec[ M_{X_{i-1},1}\pow{a_{i-1}}].
\end{align*}
We observe that
\[
 M_{min}\leq\prob\lp \Scal_i\rp\leq M_{max}.
\]
Thus, setting $\zeta_1$ as $M_{min}$ and $\zeta_2$ as $M_{max}$, completes the verification of Assumption \ref{assume:sets}.

Assumption \ref{assume:return_time}: A consequence of \Cref{eq:inhomogenous-eq2} is that for each $l\in \Ibb$, the transition matrix $M^{(l)}$ is visited at least $\lfloor n/\Tbb \rfloor$ times for any time interval of length $n$. 
For some time point $i$ let $a_i=l$.

For any integer $p>i$, $(I2)$ implies
\begin{small}
\begin{align*}
    \prob\lp \tau_{s,l}\pow{i}>p \rp  & =\prob\lp \lc X_j\neq s\bigcup a_j\neq l \rc \forall j\in\{i+1\dots,p+i\}|X_i=s,a_i=l \rp\\
    & \leq \lp1-M_{min}\rp^{\lfloor\frac{p}{\Tbb}\rfloor}\\
    & \leq \lp1-M_{min}\rp^{\frac{p}{\Tbb}-1}.
\end{align*}
\end{small}
Thus, 
\begin{align*}
    \expec[\tau_{s,l}^{(i)}]=\sum_{p\geq 1}\prob\lp\tau_{s,l}^{(i)}>p\rp\leq \frac{1}{\lp1-M_{min}\rp^{\frac{1}{\Tbb}-1}\lp 1-\lp1-M_{min}\rp^{\frac{1}{\Tbb}}\rp}.\numberthis\label{eq:inhomogenous-eq3}
\end{align*}
This completes the verification of Assumption \ref{assume:return_time}.

{Assumption \ref{assume:control-mixing} :}
Recall that $a_i$ is a deterministic sequence of indices. It follows that,
\begin{align*}
\gamma_{p,j,i} & = \sup_{s_p,\history_{i+j}^{p-1},\history_0^i}\lV\probl\lp a_p|X_p=s_p,\History_{i+j}^{p-1}=\history_{i+j}^{p-1},\History_0^i=\history_{0}^{i}\rp-\probl\lp a_p|X_p=s_p,\History_{i+j}^{p-1}=\history_{i+j}^{p-1}\rp\rV_{TV}\\
& = 0.
\end{align*}
Hence,
\[
\sup_{1\leq i\leq m}\sum_{j=i}^{m} \sum_{p=i+j+1}^m \gamma_{p,j,i}=0,
\]
and Assumption \ref{assume:control-mixing} is verified with $\constant=0$. 

Assumption \ref{assume:chain-mix} :
Observe from the definition of $\bar \theta_{i,j}$ in \cref{eq:def-theta} that,
\begin{align*}
\bar \theta_{i,j} =\sup_{l,l',s_1,s_2} \nrm{\probl\lp X_j|X_i=s_1,a_i=l \rp-\probl\lp X_j|X_i=s_2,a_i=l'\rp}_{TV}.
\end{align*}
Observe from the definition of an inhomogenous Markov chains that 
\begin{small}
\begin{align*}
    & \sup_{s_1,s_2} \nrm{\probl\lp X_j|X_i=s_1,a_i=l_1 \rp-\probl\lp X_j|X_i=s_2,a_i=l_2\rp}_{TV}\\
    & = \sup_{s_1,s_2} \nrm{\probl\lp X_j|X_i=s_1,(a_{j-1},\dots,a_i)=(l_{j-1},\dots,l_i) \rp-\probl\lp X_j|X_i=s_2,(a_{j-1},\dots,a_i)=(l_{j-1},\dots,l_i)\rp}_{TV},
\end{align*}
\end{small}

where the last line follows since all the controls are deterministic. Since all of the transition matrices are positive, an application of \citet[Theorem 1]{wolfowitz1963products} implies that, there exists an integer $C$ for which, 
\[
\bar \theta_{i,j}\leq e^{-C(j-i)}.
\]
Since $e^{-C(j-i)}<e^{-(j-i)}$ for any integer $C$, it consequently implies that,
\[
\sup_{i\geq 1}\sum_{j>i} \bar \theta_{i,j}\leq \frac{e^{-1}}{1-e^{-1}}\leq \frac{1}{1-e^{-1}}. 
\]
This completes the proof.
\end{proof}
\subsection{Proof of Propopsition \ref{prop:markov-mdp}}~\label{prf:markov-mdp}
\begin{proof} \ 

The proof proceeds by verifying Assumptions \ref{assume:sets}, \ref{assume:return_time},  \ref{assume:control-mixing}, and \ref{assume:chain-mix}.
Assumptions {\ref{assume:sets} and \ref{assume:control-mixing}} are verified similarly to that in the proof of \Cref{prop:inhomogenous-mc}.
{Assumption \ref{assume:return_time}} is verified by Lemma \ref{lemma:return-time-markov}. Finally, { Assumption \ref{assume:chain-mix}} follows from Lemma \ref{lemma:markov-assume4} by setting $\chi_0=\chi$. 
\end{proof}
{
\subsection{Proof of Proposition \ref{prop:ic-control}}~\label{sec:prf-icc}

We first prove the following lemma.
\begin{lemma}
    Let $M_{min}=\min_{l} p_l$. Then, $M^{\texttt{\textbf{S}}-2\texttt{\textbf{s}}}$ is a positive matrix with elements at least $M_{min}^{\texttt{\textbf{S}}-2\texttt{\textbf{s}}}$.
\end{lemma}

\begin{proof}{}
    We only show that $M^{\texttt{\textbf{S}}-2\texttt{\textbf{s}}}_{(\texttt{\textbf{s}},0),(\texttt{\textbf{S}}-\texttt{\textbf{s}},0)}>M_{min}^{\texttt{\textbf{S}}-2\texttt{\textbf{s}}}$. The rest follows similarly. Observe by lower bounding the Chapman-Kolmogorov equation that, 
    \begin{align*}
        M^{\texttt{\textbf{S}}-2\texttt{\textbf{s}}}_{(\texttt{\textbf{s}},0),(\texttt{\textbf{S}}-\texttt{\textbf{s}},0)} & = \prob\lp X_{\texttt{\textbf{S}}-2\texttt{\textbf{s}}} = \texttt{\textbf{S}}-\texttt{\textbf{s}}, a_{\texttt{\textbf{S}}-2\texttt{\textbf{s}}} = 0 | X_0 = \texttt{\textbf{s}},a_0=0 \rp\\
        & \geq \prod_{i=1}^{\texttt{\textbf{S}}-2\texttt{\textbf{s}}} \prob\lp X_i = \texttt{\textbf{s}}+i,a_i = 0| X_{i-1}=\texttt{\textbf{s}}+i-1,a_{i-1} = 0\rp \\
        & = p_{\texttt{\textbf{s}}-1}^{\texttt{\textbf{S}}-2\texttt{\textbf{s}}} \geq M_{min}^{\texttt{\textbf{S}}-2\texttt{\textbf{s}}}.
    \end{align*}
\end{proof}
In a manner akin to the proof of Proposition \ref{prop:inhomogenous-mc}, the return time bound for $M^{\texttt{\textbf{S}}-2\texttt{\textbf{s}}}$ can be established as $1/M_{min}^{\texttt{\textbf{S}}-2\texttt{\textbf{s}}}$. Thus, $\{X_i,a_i\}$ fulfills Assumption \ref{assume:return_time}, with  $T=(\texttt{\textbf{S}}-2\texttt{\textbf{s}})/M_{min}^{\texttt{\textbf{S}}-2\texttt{\textbf{s}}}$.

Likewize, similar to the reasoning in the proof of Proposition \ref{prop:markov-mdp}, we ascertain that $\{X_{i},a_i\}$ satisfies Assumption \ref{assume:control-exp}, in this case with $\gamma_{i,j,k}=0$, and Assumption \ref{assume:chain-exp}, with $\bar \theta_{i,j} = (d/M_{min}^{\texttt{\textbf{S}}-2\texttt{\textbf{s}}})^{\lfloor(j-i)/(\texttt{\textbf{S}}-2\texttt{\textbf{s}})\rfloor}$; here, $\lfloor\cdot\rfloor$ denotes the floor function. Assumption \ref{assume:sets},  can be established with $\zeta_2 = M_{min}^{\texttt{\textbf{S}}-2\texttt{\textbf{s}}}$ and $\zeta_1=1-M_{min}^{\texttt{\textbf{S}}-2\texttt{\textbf{s}}}$ following a similar course as that in Proposition \ref{prop:inhomogenous-mc}. This completes the proof.
}
\subsection{Proof of Proposition \ref{prop:combined-chain}}~\label{sec:prf-combchn}

\begin{proof}\
\textbf{STEP 1.}
Our first objective is to find $m\pow 1$ and $m\pow 2$ respectively. We only produce $m\pow 1$. $m\pow 2$ can be found similarly. To find $m\pow 1$, we only need to find the constants $\zeta_1,\zeta_2,T$, and $\constant_\theta$. Since it is a Markov chain, it is easy to see that Assumption \ref{assume:control-exp} is satisfied with $\constant=0$. It is easily seen that the lower bound in Assumption \ref{assume:sets} is satisfied by $\zeta_2 = \iota_1$. By taking $\Scal_i=\chi\backslash t_1$, it is easy to see that the upper bound is satisfied with $\zeta_1=1-\iota_1$. Next, we produce $T$ by calculating an upper bound for the smallest probability of the invariant distribution. For that, let $\nu\pow 1$ be the invariant distribution for $M\pow 1$. We first show that $\nu_{t_1}\pow 1<\iota_1$. We do that by verification. We know from the balance equations that
\[
\nu_{t_1}\pow 1 = \sum_{s\in\chi} M_{s,t_1}\pow 1\nu_{s}\pow 1.
\]
Trivially bounding $\nu_s\pow 1$ by $1$, we get
\[
\nu_{t_1}\pow 1 < \sum_{s\in\chi} M_{s,t_1}\pow 1 = \iota_1
\]
We now prove that $\nu_{t}\pow 1>M_{min}$ for all $t\neq t_1$. To see this fact, we again start from the balance equation for $\nu_{t}\pow 1$ and then lower bound $M_{s,t}\pow 1$ by $M_{min}$. 
\[
\nu_{t}\pow 1 =  \sum_{s\in\chi} M_{s,t}\nu_{s}\pow 1>M_{min}\sum_{s\in \chi}\nu_{s}\pow 1=M_{min}.
\]
Similarly to the proof of Proposition \ref{prop:stationary-mdp} it now follows that $T=\max_s 1/\lp \nu_s\pow 1\rp$. It follows from Assumption \ref{assume:combined-chain} that $\max_s 1/\lp \nu_s\pow 1\rp>1/\iota_1$.
Finally, using $\chi_0=\chi\backslash\{t_1,t_2\}$ in Lemma \ref{lemma:markov-assume4} we get that 
\[
\constant_\theta = \frac{1}{(d-2)M_{min}}.
\]
Substituting the previous constants in Theorem \ref{thm:sample-complexity} and simplifying the terms we get that for a Markov chain with transition matrix $M\pow 1$, as long as the number of samples 
\[
m>m\pow 1 = \; \max\left\{\frac{1}{\iota_1\eps^2}\log\lp \frac{d}{\iota_1\epsilon^2\delta} \rp, \frac{1}{\lp(d-2)M_{min}\iota_1\rp^2}\log\lp \frac{d}{\delta} \rp\right\}
\]
the empirical estimate $\hat M\pow 1$ satisfies $\prob\lp\lV\hat M\pow 1-M\pow 1\rV_\infty>\epsilon\rp<\delta$. We have ignored the universal constant in the previous expression for clarity.
Similarly, for a Markov chain with transition matrix $M\pow 1$, as long as the number of samples 
\[
m>m\pow 2 = \; \max\left\{\frac{1}{\iota_2\eps^2}\log\lp \frac{d}{\iota_2\epsilon^2\delta} \rp, \frac{2}{\lp(d-2)M_{min}\iota_2\rp^2}\log\lp \frac{d}{\delta} \rp\right\}
\]
the empirical estimate $\hat M\pow 2$ satisfies $\prob\lp\lV\hat M\pow 2-M\pow 2\rV>\epsilon\rp_\infty<\delta$. 

\textbf{STEP 2.}
Now we show that it is faster to estimate the transition matrices from a controlled Markov chain with a sequence of deterministic controls, $\{a_i\}_{i\geq 0}=\{1,1,2,2,1,1,2,2,\dots \}$. By taking $\Scal_i = \chi\backslash\{t_1,t_2\}$, it can be easily seen that Assumption \ref{assume:sets} is satisfied by taking $\zeta_2 = M_{min}$ and $\zeta_1 = 1-M_{min}$. And it follows similarly to the previous part that $\constant = 0$ and 
\[
\constant_\theta  = \frac{1}{(d-2)M_{min}}.
\]
We now only need to find $T$ to upper bound the conditional expectations of $\tau_{t,l'}\pow i$.

\emph{Case 1: $(t,l')=(t_2,2)$} For integer $i$ such that $4$ divides $i$, let $\history_0^i$ be such that $X_i=s,a_i=1$. Then, 
\begin{align*}
    \prob\lp (X_{i+1},a_{i+1})\neq (t_2, 2)|\History_0^i = \history_0^i\rp & =  \prob\lp X_{i+1}\neq t_2,a_{i+1} = 2|X_i=s,a_i=1\rp\\
    & =  \prob\lp X_{i+1}\neq t_2|X_i=s,a_i=1\rp\\
    & \leq 1-M_{min}.
\end{align*}
For $a_{i+2}$,
\begin{align*}
    \prob\lp (X_{i+2},a_{i+2})\neq (t_2,2), (X_{i+1},a_{i+1})\neq (t_2,2)|\History_0^i = \history_0^i\rp & \leq\prob\lp  (X_{i+1},a_{i+1})\neq (t_2,2)|\History_0^i = \history_0^i\rp\\
    & \leq 1-M_{min}.
\end{align*}
Observe that $a_{i+3}=1$ by construction. Thus,
\begin{align*}
    & \prob\lp (X_{i+3},a_{i+3})\neq (t_2,2), (X_{i+2},a_{i+2})\neq (t_2,2), (X_{i+1},a_{i+1})\neq (t_2,2)|\History_0^i = \history_0^i\rp\\
     & \quad =  \prob\lp (X_{i+2},a_{i+2})\neq (t_2,2), (X_{i+1},a_{i+1})\neq (t_2,2)|\History_0^i = \history_0^i\rp\\
    & \quad  \leq 1-M_{min}.
\end{align*}
The case of $a_{i+4}=1$ follows similarly. It is easy to show by induction that for any integer $j\geq 1$,
\begin{align*}
    \prob\lp (X_{i+j},a_{i+j})\neq (t_2,2),\dots, (X_{i+1},a_{i+1})\neq (t_2,2)|\History_0^i = \history_0^i\rp\leq (1-M_{min})^{\lfloor\frac{j-1}{4}\rfloor+1}
\end{align*}
where $\lfloor\cdot\rfloor$ is the floor function. Observe that this quantity does not depend upon $i$. It now follows, using techniques similar to the proof of Lemma \ref{lemma:return-time-markov} that for any $i$ 
\[ \prob\lp\tau_{t_2,2}\pow i\geq j|\History_0^{\sum_{p=1}^{i-1}\tau_{t_2,2}\pow p}=\history_0^{\sum_{p=1}^{i-1}\tau_{t_2,2}\pow p}\rp\leq (1-M_{min})^{\lfloor\frac{j-1}{4}\rfloor+1}
\]
Using the layer-cake representation for expectation, we get 
\begin{align*}
    \expec\lb\tau_{t_2,2}\pow i\geq j|\History_0^{\sum_{p=1}^{i-1}\tau_{t_2,2}\pow p}=\history_0^{\sum_{p=1}^{i-1}\tau_{t_2,2}\pow p}\rb & = \sum_{j\geq 1}\prob\lp\tau_{t_2,2}\pow i\geq j|\History_0^{\sum_{p=1}^{i-1}\tau_{t_2,2}\pow p}=\history_0^{\sum_{p=1}^{i-1}\tau_{t_2,2}\pow p}\rp\\
    & \leq 4 \sum_{j\geq 0} (1-M_{min})^{j}\\
    & \leq \frac{4}{M_{min}}.
\end{align*}

\emph{Case 2: $(t,l')=(t_1,1)$} This case follows similarly.

\emph{Case 3: $(t,l')=(t_1,2)$} For $i$ such that $4$ divides $i-1$ let $\history_0^i$ be such that $X_i=s,a_i=1$.
\begin{align*}
    \prob\lp (X_{i+1},a_{i+1})\neq (t_1, 2)|\History_0^i = \history_0^i\rp & =  \prob\lp X_{i+1}\neq t_1,a_{i+1} = 2|X_i=s,a_i=1\rp\\
    & \leq 1-M_{min}.
\end{align*}
Now following the calculations similar to Case 1, we find that the upper bound to the expected return time is $4/M_{min}$. All other cases follow similarly. Now invoking Theorem \ref{thm:sample-complexity} we get that as long as the number of samples 
\begin{align*}
    m > m\pow c =\max\lc \frac{4}{M_{min}\eps^2}\log\lp\frac{4d}{M_{min}\epsilon^2\delta}\rp, \frac{32}{(d-2)^2 M_{min}^4} \log\lp \frac{d}{\delta}\rp\rc
\end{align*}
the empirical estimate satisfies $\prob\lp \underset{l\in\{1,2\}}{\sup} \lV M\pow l-\hat M\pow l \rV_{\infty} >\eps \rp<\delta$. As before, we have ignored the universal constant for the sake of brevity. Since $M_{min}>8\min\{\iota_1,\iota_2\}$, the rest of the proof follows by some simple algebra.
\end{proof}
\subsection{Proof of \Cref{prop:episodic-mdp}}~\label{prf:ep-mdp}
\begin{proof} \ 

The proof proceeds by verifying Assumptions \ref{assume:sets}, \ref{assume:return_time},  \ref{assume:control-mixing}, and \ref{assume:chain-mix}.\\
\text{Assumption \ref{assume:sets}: } This is verified similarly to that in the proof of \Cref{prop:inhomogenous-mc}.\\
\text{Assumption \ref{assume:return_time} :} Our objective is to provide an upper bound for $\sup_{i\geq 0}\expec[\tau_{s,l}\pow{i}]$.
We prove only when $i=0$. All other cases follow similarly. 
To compute the expectation, it is sufficient to bound from above the survival function $\prob\lp\tau_{s,l}\pow{0}>p\rp$ for all $p\geq 1$.
When $p\in\lc 1,\dots,H -1 \rc$, we trivially upper bound this probability by $1$.
When $p=H $, writing the expression for $\prob\lp\tau_{s,l}\pow{0}>H \rp $ we get
\begin{align*}
    & \prob\lp \lc X_j\neq s\bigcup a_i\neq l \rc \forall j\in\{1\dots,H \}|X_0=s,a_0=l \rp\\
    &\  = \prob\lp \lc X_H \neq s\bigcup a_H \neq l \rc  \rp & (E1)\\
    & \quad \times \prob\lp \lc X_j\neq s\bigcup a_i\neq l \rc \forall j\in\{1\dots,H -1\}|X_0=s,a_0=l \rp\\
    & \leq \lp1-\frac{1}{dk}\rp.& (E2)
\end{align*}
$(E1)$ follows from \cref{eq:episodic-eq1} since for all $j\neq H $, $X_H ,a_H $ is independent of $(X_j,a_j)$.
$(E2)$ follows by substituting the appropriate probability in  $\prob\lp \lc X_H \neq s\bigcup a_H \neq l \rc  \rp$ and trivially upper bounding $\prob\lp \lc X_j\neq s\bigcup a_i\neq l \rc \forall j\in\{1\dots,H -1\}|X_0=s,a_0=l \rp$ by $1$.
We can proceed similarly for $p\in\{H +1,\dots,2H -1\}$.
For $p=2H $ we can similarly decompose $\prob\lp\tau_{s,l}\pow{i}>2H \rp$ as,
\begin{align*}
    \prob\lp\tau_{s,l}\pow{i}>2H \rp & = \prob\lp \lc X_{2H }\neq s\bigcup a_{2H }\neq l \rc  \rp & \\
    & \quad \times \prob\lp \lc X_H \neq s\bigcup a_H \neq l \rc  \rp & \\
    & \quad \times \prob\lp \lc X_j\neq s\bigcup a_i\neq l \rc \forall j\in\{1\dots,2H -1\}\backslash\{H \}|X_0=s,a_0=l \rp\\
    &\leq \lp1-\frac{1}{dk}\rp^2.
\end{align*}
Proceeding similarly, we bound from above $ \prob\lp\tau_{s,l}\pow{0}>p\rp$ for each $p$.
Substituting these bounds in the expression for $\expec[\tau_{s,l}\pow{0}]$ we get
\begin{align*}
    \expec[\tau_{s,l}\pow{0}] & =  \sum_{p=1}^{H -1} \prob\lp\tau_{s,l}\pow{0}>p\rp + \sum_{p=H }^{2H -1}\prob\lp\tau_{s,l}\pow{0}>p\rp+\dots\\
    & \leq H -1+\lp1-\frac{1}{dk}\rp H +\lp1-\frac{1}{dk}\rp^2H +\dots\\
    & = \lp H +\lp1-\frac{1}{dk}\rp H +\lp1-\frac{1}{dk}\rp^2H +\dots\rp-1\\
    & = dkH -1.
\end{align*}
This completes the verification of Assumption \ref{assume:return_time}.\\
{Assumption \ref{assume:control-mixing} :} To verify Assumption \ref{assume:control-mixing} let $(p,i,j)$ be any triplet in $\naturalset^3$ such that $j>i$. 
We first consider the case when $j\leq i+H $. 
Observe from \cref{eq:episodic-eq2} that whenever $p>i+j+H $,
\begin{align*}
    & \probl\lp a_p|X_p=s_p,\History_{i+j}^{p-1}=\history_{i+j}^{p-1},\History_0^i=\history_{0}^{i}\rp-\probl\lp a_p|X_p=s_p,\History_{i+j}^{p-1}=\history_{i+j}^{p-1}\rp\numberthis\label{eq:episodic-eq3}\\
    &\  = \probl\lp a_p|X_p=s_p,\History_{H \pow{p}}^{p-1}=\history_{H \pow{p}}^{p-1}\rp-\probl\lp a_p|X_p=s_p,\History_{H \pow{p}}^{p-1}=\history_{H \pow{p}}^{p-1}\rp\\
    & \ =0.
\end{align*}
Whenever $p\leq i+j+ H $, it follows trivially that 
\begin{align*}
    -1\leq \probl\lp a_p|X_p=s_p,\History_{i+j}^{p-1}=\history_{i+j}^{p-1},\History_0^i=\history_{0}^{i}\rp-\probl\lp a_p|X_p=s_p,\History_{i+j}^{p-1}=\history_{i+j}^{p-1}\rp\leq 1.
\end{align*}
Now we consider the case when $j>i+H $. 
Using the law of iterated expectations we get
\begin{small}
\begin{align*}
    \probl\lp a_p|X_p=s_p,\History_{i+j}^{p-1} =\history_{i+j}^{p-1},\History_0^i  =\history_{0}^{i}\rp & =\expec\lb \probl\lp a_p|X_p=s_p,\History_{i+j}^{p-1}=\history_{i+j}^{p-1},\History_{H \pow{i+j}}^{i+j-1},\History_0^i=\history_{0}^{i}\rp\rb\\
    & =\expec\lb \probl\lp a_p|X_p=s_p,\History_{i+j}^{p-1}=\history_{i+j}^{p-1},\History_{H \pow{i+j}}^{i+j-1}\rp\rb\numberthis\label{eq:episodic-eq4},
\end{align*}
\end{small}
where the second equality follows from \cref{eq:episodic-eq1}. Substituting this expression in \cref{eq:episodic-eq3} we get for all $j>i+H $ 
\begin{align*}
    & \probl\lp a_p|X_p=s_p,\History_{i+j}^{p-1}=\history_{i+j}^{p-1},\History_0^i=\history_{0}^{i}\rp-\probl\lp a_p|X_p=s_p,\History_{i+j}^{p-1}=\history_{i+j}^{p-1}\rp\\
    & \ =\expec\lb \probl\lp a_p|X_p=s_p,\History_{i+j}^{p-1}=\history_{i+j}^{p-1},\History_{H \pow{i+j}}^{i+j-1}\rp\rb-\expec\lb \probl\lp a_p|X_p=s_p,\History_{i+j}^{p-1}=\history_{i+j}^{p-1},\History_{H \pow{i+j}}^{i+j-1}\rp\rb\\
    & \ =0.
\end{align*}
Substituting the previous upper bounds into the expression for $\gamma_{p,j,i}$, we get that
\begin{align*}
    \sum_{j=1}^{m} \sum_{p=i+j+1}^m \gamma_{p,j,i} & =  \sum_{j=1}^{i+H } \sum_{p=i+j+1}^{i+j+H } \gamma_{p,j,i} +\sum_{j=1}^{i+H } \sum_{p=i+j+H +1}^{m} \gamma_{p,j,i}+ \sum_{j={H \pow{p}}}^{m} \sum_{p=i+j+1}^m \gamma_{p,j,i} \\
    & = \sum_{j=1}^{i+H } \sum_{p=i+j+1}^{i+j+H } \gamma_{p,j,i}\\
    & < H ^2.
\end{align*}
This completes the verification of Assumption \ref{assume:control-mixing}.

\noindent{Assumption \ref{assume:chain-mix}: } Applying law of iterated expectation and decomposing $ \probl\lp X_j|X_i=s_1,a_i=l_1\rp$ similar to \cref{eq:episodic-eq4}, we obtain that whenever $j>i+H $ $\bar \theta_{i,j}=1$. 
It follows that  $\sum_{j>i}\bar\theta_{i,j}<H $. 
This completes the verification of Assumption \ref{assume:chain-mix}.
\end{proof}

\section{Proofs of Propositions and Lemmas}

\subsection{Proof of Lemma \ref{lemma:mixing-lemm}}~\label{sec:prf-mxlemm}

\begin{proof}  \ 

The upper bound follows easily by adding and subtracting $\prob\left((X_m,a_m,\dots, X_j,a_j)\in \mathbb{T}\right)$, and using triangle inequality. To prove the lower bound, let $(\history_0^i)'\in(\chi\times\Ibb)^{(i+1)}$. We can write
\begin{small}
\begin{align*}
    &\lv\prob\left((X_m,a_m,\dots, X_j,a_j)\in \mathbb{T}|\History_0^{i}=\history_0^{i}\right)-\prob\left((X_m,a_m,\dots, X_j,a_j)\in \mathbb{T}\right)\rv\\
    & \ = \Bigg |\sum_{(\history_0^i)'\in(\chi\times\Ibb)^{(i+1)}} (\prob\left((X_m,a_m,\dots, X_j,a_j)\in \mathbb{T}|\History_0^{i}=\history_0^{i}\right)\\ 
    &\quad \quad -\prob\left((X_m,a_m,\dots, X_j,a_j)\in \mathbb{T}|\History_0^i=(\history_0^i)'\right))\prob(\History_0^i=(\history_0^i)') \Bigg |\\
    & \ \leq \sum_{(\history_0^i)'\in(\chi\times\Ibb)^{(i+1)}}\Bigg |  \prob\left((X_m,a_m,\dots, X_j,a_j)\in \mathbb{T}|\History_0^{i}=\history_0^{i}\right)\\ 
    & \quad -\prob\left((X_m,a_m,\dots, X_j,a_j)\in \mathbb{T}|\History_0^i=(\history_0^i)'\right)\Bigg |\prob(\History_0^i=(\history_0^i)')\\
    & \ \leq \bar\eta_{i,j}\sum_{(\history_0^i)'\in(\chi\times\Ibb)^{(i+1)}}\prob(\History_0^i=(\history_0^i)')\\
    & \ =\bar\eta_{i,j}.
\end{align*}
\end{small}
The first inequality follows using the triangle inequality, and the second inequality follows from equation \ref{def:weak-mixing}. This completes the proof.

\end{proof}

\subsection{Proof of Proposition \ref{prop:err-bnd}}
\begin{proof} \ 

The analysis of the proposition term is done via the sampling scheme introduced in Section \ref{Sec:mdp-sampling}. To begin, use the sampling scheme to get $\{\Tilde{X}_0,\tilde a_0,\dots,\Tilde X_m,\tilde a_m\}$. 
We construct the estimators $\Tilde{N}_s^{(l)} :=\sum_{i} \indicator[\tilde{X_i}=s,\tilde a_i=l] \text{ and }
\tilde{N}_{s,t}^{(l)} :=\sum_{i} \indicator[\tilde{X_i}=s,\tilde{X}_{i+1}=t,\tilde a_{i}=l]$. Consequently, we define $\tildemlst := \frac{\tilde{N}_{s,t}^{(l)}}{\tilde{N}_{s}^{(l)}}$ and $\tildemls:=\lp\tilde{M}_{s,1}^{(l)},\tilde{M}_{s,2}^{(l)},\dots,\tilde{M}_{s,d}^{(l)} \rp$. We observe that $(\hatmls, N_s\pow{l})\overset{d}{=}(\tildemls,\tilde N_s\pow{l})$ by construction we have,
\begin{align*}
    \prob\Bigg(\bigg\|\hatmls  -\mls\bigg\|_1 > \epsilon, N_s^{(l)}=n \Bigg) & = \prob\Bigg(\bigg\|\tildemls  -\mls\bigg\|_1 > \epsilon, \tilde N_s^{(l)}=n \Bigg)\numberthis\label{eq:main-thmeq4}
\end{align*}
Next, we observe that given $\tilde{N}_s^{(l)}=n\leq m$, and for any $t\in\chi$, one can reduce $\tildemlst$ into sum of independent random variables. 
Recall from \Cref{Sec:mdp-sampling} that $\tilde X_{i+1}=X_{\Tilde{X_i},\Tilde{N}_{\Tilde{X_i}}^{(i,\tilde a_i)}+1}^{(\tilde a_i)}$. 
Therefore, we can write
\begin{align*}
    \tildemlst & =\frac{1}{n}\sum_{i=1}^m \indicator \left[\tilde{X_i}=s, X_{\Tilde{X_i},\Tilde{N}_{\Tilde{X_i}}^{(i,\tilde a_i)}+1}^{(\tilde a_i)}=t,\tilde a_i=l\right]\\
    & = \frac{1}{n} \sum_{i=1}^n \indicator [X_{s,i}^{(l)}=t].\numberthis\label{eq:main-thmeq2}
\end{align*}
Let $\tilde M_n\pow{l}(s,\cdot)$ be defined as the $d$ dimensional vector whose $t$-th coordinate is $\frac{1}{n} \sum_{i=1}^n \indicator [X_{s,i}^{(l)}=t]$.
Therefore, 
\begin{align*}
    \prob\left(\left\|\tildemls-\mls\right\|_1 > \epsilon, \tilde{N}_s^{(l)}=n \right) & =  \prob\left(\sum_{t\in\chi}\lv\tildemlst-M_{s,t}\pow{l}\rv > \epsilon, \tilde{N}_s^{(l)}=n \right)\\
    & = \prob\left(\left\|\tilde M_n\pow{l}(s,
    \cdot)-\mls\right\|_1 > \epsilon, \tilde{N}_s^{(l)}=n\right)\\
    & \leq \prob\lp \left\|\tilde M_n\pow{l}(s,
    \cdot)-\mls\right\|_1 > \epsilon\rp , \numberthis\label{eq:main-thmeq3}
\end{align*}
We obtain the following facts as a consequence of equations \ref{eq:main-thmeq2} and \ref{eq:main-thmeq3}. 

First,  $\sum_{i=1}^n \indicator [X_{s,i}^{(l)}=t]$ is a sum of independent Bernoulli random variables and $\expec[n\tilde M_n\pow{l}(s,
    \cdot)]=n \mls$. 
Furthermore, 
\[
\Var\left( \sum_{i=1}^n \indicator [X_{s,i}^{(l)}=t]\right)=n {M}_{s,t}^{(l)}(1-{M}_{s,t}^{(l)})\leq n{M}_{s,t}^{(l)}.
\]
Second, the mean absolute deviation satisfies
\begin{align*}
    \expec\left[\lVert\tilde M_n\pow{l}(s,
    \cdot)-{M}\pow{l}(s,\cdot)\rVert_{1}\right] & = \sum_{t\in\chi}\expec\lb \lv\sum_{i=1}^n \frac{1}{n}\indicator [X_{s,i}^{(l)}=t]-{M}_{s,t}^{(l)}\rv \rb\\
    & \leq\sum_{t\in\chi}\sqrt{ \Var\lp\frac{1}{n}\sum_{i=1}^n\indicator [X_{s,i}^{(l)}=t]\rp}\\
    & \leq \frac{ \sum_{t\in\chi}\sqrt{M_{s,t}^{(l)}}}{\sqrt{n}}\\
    & \leq \sqrt{\frac{d}{n}},\numberthis\label{eq:main-thmeq6}
\end{align*}
where the first and the last inequalities follows by Cauchy-Schwarz inequality.

Let $\Phi$ be a function such that
\[
 \Phi(X_{s,1}\pow{l},\dots,X_{s,n}\pow{l})=\begin{bmatrix}
 \frac{1}{n}\sum_{i=1}^n\indicator[X_{s,i}\pow{l}=1]\\
 \vdots\\
 \frac{1}{n}\sum_{i=1}^n\indicator[X_{s,i}\pow{l}=d]
 \end{bmatrix}-\begin{bmatrix}
 M_{s,1}\pow{l}\\
 \vdots\\
 M_{s,d}\pow{l}
 \end{bmatrix}.
\]
For a fixed $j\in\lc1,\dots,n\rc$, let  $(X_{s,1}\pow{l},\dots,X_{s,j}\pow{l},\dots,X_{s,n}\pow{l})$ and $(X_{s,1}\pow{l},\dots,\lp X_{s,j}\pow{l}\rp',\dots,X_{s,n}\pow{l})$ be two random vectors which differ only in the $j$-th coordinate. 
Thus,
$X_{s,j}\pow{l}=t_1$ and $\lp X_{s,j}\pow{l}\rp'=t_2$, with $t_1\neq t_2$ .
As a consequence, all but the $t_1$ and $t_2$-th coordinates of the vector 
\[
\Phi(X_{s,1}\pow{l},\dots,X_{s,j}\pow{l},\dots,X_{s,n}\pow{l})-\Phi\lp X_{s,1}\pow{l},\dots,\lp X_{s,j}\pow{l}\rp',\dots,X_{s,n}\pow{l}\rp
\]
are $0$, while the $t_1$-th and $t_2$-th coordinates are $\frac{1}{n}$ and $-\frac{1}{n}$ respectively. Thus,
\[
\lV\Phi(X_{s,1}\pow{l},\dots,X_{s,j}\pow{l},\dots,X_{s,n}\pow{l})-\Phi\lp X_{s,1}\pow{l},\dots,\lp X_{s,j}\pow{l}\rp',\dots,X_{s,n}\pow{l}\rp\rV_{1}\leq \frac{2}{n}.
\]
The reverse triangle inequality gives us,
\begin{align*}
    & \lv\lV\Phi(X_{s,1}\pow{l},\dots,X_{s,j}\pow{l},\dots,X_{s,n}\pow{l})\rV_1 - \lV \Phi\lp X_{s,1}\pow{l},\dots,\lp X_{s,j}\pow{l}\rp',\dots,X_{s,n}\pow{l}\rp\rV_1\rv\\
    & \quad \leq\lV\Phi(X_{s,1}\pow{l},\dots,X_{s,j}\pow{l},\dots,X_{s,n}\pow{l})-\Phi\lp X_{s,1}\pow{l},\dots,\lp X_{s,j}\pow{l}\rp',\dots,X_{s,n}\pow{l}\rp\rV_1\\
    &\quad \leq \frac{2}{n}.
\end{align*}

We can now apply McDiarmid's inequality \citep[Equation 1.3]{rio2013mcdiarmid} to the probability in \cref{eq:main-thmeq3}, which, combined with the previous facts implies,
\begin{align*}
    \prob\Bigg(\bigg\|\hatmls  -\mls\bigg\|_1 > \epsilon, N_s^{(l)}=n \Bigg) & \leq \prob\Bigg(\bigg\|\tilde M_n\pow{l}(s,\cdot) -\mls\bigg\|_1 > \epsilon \Bigg) \\
    & \leq \exp\left(-\frac{n}{2}\max\left\{0, \epsilon-\sqrt{\frac{d}{n}}\right\}^2\right).
\end{align*}
Now it follows from \cref{eq:prob_sum}, 
\begin{align*}
    & \sum_{n=n_{low,s}}^{n_{high,s}} \prob\left(\left\|\hatmls-\mls\right\|_1 > \epsilon, N_s^{(l)}=n \right) \\
    & \quad \leq \sum_{n=n_{low,s}}^{n_{high,s}}  \exp\left(-\frac{n}{2}\max\left\{0,\epsilon-\sqrt{\frac{d}{n}}\right\}^2\right)\\
    & \quad \leq \sum_{n=n_{low,s}}^{n_{high,s}} \exp\left({-\frac{n_{low,s}}{2} \max \left\{0,\epsilon-\sqrt{\frac{d}{n_{high,s}}}\right\}^2}\right).
\end{align*}
Since the term in the exponent does not depend upon $n$, it can be taken out of the summation. This yields the following upper bound,
\begin{align*}
    &  \lp n_{high,s}-n_{low,s} \rp\exp\left({-\frac{n_{low,s}}{2} \max \left\{0,\epsilon-\sqrt{\frac{d}{n_{high,s}}}\right\}^2}\right)\\
    &\leq m\exp\left({-\frac{n_{low,s}}{2} \max \left\{0,\epsilon-\sqrt{\frac{d}{n_{high,s}}}\right\}^2}\right)\numberthis~\label{eq:t1bound}
\end{align*}
This completes the analysis.
\end{proof}

\subsection{Proof of Lemma \ref{lemma:expec-bound}}~\label{sec:prf-expbd}
To find the upper bound, observe that,
\begin{align*}
    \expec[N_s\pow{l}]=\sum_{i=1}^m\expec\lb\indicator[X_i=s,a_i=l]\rb & =\sum_{i=1}^m\prob\lp X_i=s,a_i=l \rp. 
\end{align*}
From Assumption \ref{assume:sets} it follows that if $(s,l)\in\Scal_i$ then, $\prob\lp X_i=s,a_i=l \rp\leq \zeta_1$. Furthermore, if $(s,l)\notin\Scal_i$, then $(s,l)\in\Scal_i^c$, where $\Scal_i^c$ is the complement of the set $\Scal_i$. In that case, 
\begin{align*}
    \prob\lp X_i=s,a_i=l\rp & \leq \prob\lp (X_i,a_i)\in \Scal_i^c\rp\\ 
    & = 1 -  \prob\lp (X_i,a_i)\in \Scal_i\rp\\ 
    & \leq 1-\zeta_2.
\end{align*}
Thus, we have proved that $\prob\lp X_i=s,a_i=l \rp\leq \max\lc \zeta_1,1-\zeta_2 \rc$. It follows as a consequence that,
\begin{align*}
    \expec[N_s\pow{l}] & =\sum_{i=1}^m\prob\lp X_i=s,a_i=l \rp\\ 
    & \leq m\max\lc \zeta_1,1-\zeta_2 \rc.
\end{align*}
For the lower bound, define the random variable $\{Z_{s,l}\pow{p}\}$ and the filtration $\Fcal_p'$ as,
\begin{align*}
    & Z_{s,l}\pow{0} := 0\\
    & Z_{s,l}\pow{p} := \frac{\sum_{i=1}^p\tau_{s,l}\pow{i}}{T}-p\\
    & \Fcal_p':=\Fcal_{\sum_{i=1}^p\tau_{s,l}\pow{i}}.
\end{align*}
Observe that 
\begin{align*}
    \expec[ Z_{s,l}\pow{p}|\Fcal_{p-1}'] & =\frac{\expec[\sum_{i=1}^p\tau_{s,l}\pow{i}|\Fcal_{p-1}']}{T}-p\\
    & = \frac{\expec[\sum_{i=1}^{p-1}\tau_{s,l}\pow{i}|\Fcal_{p-1}']}{T}-(p-1)+\frac{\expec[\tau_{s,l}\pow{p}|\Fcal_{p-1}']}{T}-1\\
    & \leq \expec[{Z_{s,l}\pow{p-1}}|\Fcal_{p-1}']+\frac{T}{T}-1\\
    & =Z_{s,l}\pow{p-1},
\end{align*}
where the last inequality follow from Assumption \ref{assume:return_time} and the last equality follow from the fact that $Z_{s,l}\pow{p-1}$ is $\Fcal_{p-1}'$ measurable.
It follows that, $\{Z_{s,l}\pow{p}\}$ is a supermartingale.
Now, define 
\[
N:=\min\{n:\sum_{i=1}^p\tau_{s,l}\pow{i}>m\}.
\]
It can be seen easily that $N$ is a valid stopping time. 
Moreover, since the return times $\tau_{s,l}\pow{i}\geq 1$ almost everywhere, it easily follows that $\prob(N\leq m+1)=1$.
Therefore, it follows from Doob's Optional Stopping Theorem for supermartingales \cite[Theorem 7.1, page 495]{gut2005probability} that,
\begin{align*}
    \expec[Z_N]\leq \expec[Z_0].
\end{align*}
Since $Z_0=0$, we can write  
\begin{align*}
    \expec\lb\frac{\sum_{i=1}^N\tau_{s,l}\pow{i}}{T}-N\rb& \leq 0.\\
         \intertext{This in turn implies}
    \expec\lb\frac{\sum_{i=1}^N\tau_{s,l}\pow{i}}{T}\rb & \leq \expec[N].
\end{align*}
Next, we observe that $N_s\pow{l}$ can be written as 
\[
N_s\pow{l}=\max\{n:\sum_{i=1}^n\tau_{s,l}\pow{i}\leq m\}.
\]
In other words, $N_s\pow{l}=N+1$ almost everywhere. 
It follows that,
\begin{align*}
     \expec\lb\frac{\sum_{i=1}^N\tau_{s,l}\pow{i}}{T}\rb & \leq \expec[N_s\pow{l}]+1.\\ 
     \intertext{This in turn implies}
      \expec\lb\frac{\sum_{i=1}^N\tau_{s,l}\pow{i}}{T}\rb-1&\leq \expec[N_s\pow{l}].
\end{align*}
Finally, observe that by the definition of $N$, $\sum_{i=1}^N\tau_{s,l}\pow{i}> m$ almost everywhere.
Therefore,
\begin{align*}
    \frac{m}{T}-1&< \expec[N_s\pow{l}].
\end{align*}
Finally, if $m\geq 2T$
\[
\frac{m}{T}-1\geq \frac{m}{2T}.
\]
This completes the proof of the lower bound and consequently proves our lemma.
\subsection{Proof of Proposition \ref{prop:tail-bound-peligrad}}
\begin{proof} \

For ease of notation, we drop $s$, and $l$, and denote $\indicator[X_i=s,a_i=l]$ by $I_i$. Our first step is to prove that for any integer $i$,
\begin{align*}
    \lp \Var(I_i) +2\sum_{j\geq i}|\Cov(I_i,I_j)|\rp\leq 4\constant_\Delta\rho_s\pow{l}.
\end{align*}
By an application of Lemma \ref{lemma:hall-holder's}, observe that
\begin{align*}
    |\Cov(I_i,I_j)|&\leq \bar\eta_{i,j}\expec|I_j-\expec[I_j]|\mathrm{ess}\sup|I_i|\\
    & \leq \bar\eta_{i,j}\expec|I_j-\expec[I_j]|\\
    & = \bar\eta_{i,j}\lp\prob\lp I_j=1\rp\lp1-\prob\lp I_j=1\rp\rp+\lv0-\prob\lp I_j=1\rp\rv\lp1-\prob\lp I_j=1\rp\rp\rp\\
    & = 2\bar\eta_{i,j}\prob\lp I_j=1\rp\lp1-\prob\lp I_j=1\rp\rp\\
    & \leq 2\bar\eta_{i,j}\prob\lp I_j=1\rp\\
    & \leq 2\bar\eta_{i,j}\rho_s\pow{l}.
\end{align*}
Moreover,
\[
\Var(I_i)=\prob\lp I_j=1\rp\lp1-\prob\lp I_j=1\rp\rp\leq \prob\lp I_j=1\rp\leq 4\rho_s\pow{l}.
\]
Combining these facts we get,
\begin{align*}
     \lp \Var(I_i) +2\sum_{j\geq i}|\Cov(I_i,I_j)|\rp & \leq 4\lp1+\sum_{j\geq i}\bar\eta_{i,j}\rp\rho_s\pow{l}\\
     & \leq 4\constant_\Delta\rho_s\pow{l},
\end{align*}
where the final line follows from Assumption \ref{assume:eta-exp}. Next, we observe that,
\begin{align*}
    \prob(N_s^{(l)}\notin [n_{low,s} ,n_{high,s}]) & = \prob(N_s^{(l)}-\expec[N_s\pow{l}]< n_{low,s}-\expec[N_s\pow{l}])\\
    & \quad +\prob(N_s^{(l)}-\expec[N_s\pow{l}]>n_{high,s}-\expec[N_s\pow{l}]).\numberthis\label{eq:proptbp-eq1}
\end{align*}
The rest of the proof follows from an application of Lemma \ref{lemma:peligrad} and \cref{eq:proptbp-eq1}. Since  $4\constant_\Delta\rho_s\pow{l}$ is independent of $i$, our proof is complete. 
\end{proof}
\subsection{Proof of Lemma \ref{lemma:delta-bound}}
\begin{proof} \
Recall from Lemma \ref{lemma:kontorovich} the definition of $\deltam$.
\[
\|\Delta_m\|:=\underset{1\leq i\leq m}{\max} (1+ \Bar{\eta}_{i,i+1}+ \Bar{\eta}_{i,i+2}+\dots \Bar{\eta}_{i,m}),
\]
where, 
\begin{align*}
    \Bar{\eta}_{i,j}:= \underset{\mathbb{T},s_1,s_2,l_1,l_2,\history_0^{i-1}}{\sup} \eta_{i,j},
\end{align*}
and, 
\begin{align*}
    \eta_{i,j} & := \left|\prob\left((X_m,a_m,\dots, X_j,a_j)\in \mathbb{T}|X_i=s_1,a_i=l_1,\History_0^{i-1}=\history_0^{i-1}\right)\right.\\
    & \left.\qquad \qquad -\prob\left((X_m,a_m,\dots,X_j,a_j)\in\mathbb{T}|X_i=s_2,a_i=l_2,\History_0^{i-1}=\history_0^{i-1}\right)\right|.
\end{align*}
For any fixed $s_1,s_2\in \chi$, $a_1,a_2\in\Ibb$, and $\history_0^{i-1}\in (\chi\times\Ibb)^{i}$, define the total variation distance $\eta_{i,j}\pow{m}$ as, 
\begin{align*}
\eta_{i,j}\pow{m} & := \|\probl\left((X_m,a_m,\dots, X_j,a_j)|X_i=s_1,a_i=l_1,\History_1^{i-1}=\history_0^{i-1}\right)\\
    &  \qquad -\probl\left((X_m,a_m,\dots,X_j,a_j)|X_i=s_2,a_i=l_2,\History_1^{i-1}=\history_0^{i-1}\right)\|_{TV}.
\end{align*}
We note that for each fixed $m$, $\bar \eta_{i,j}$ can be recovered from the corresponding $\eta_{i,j}\pow{m}$ as 
\begin{align*}
    \bar \eta_{i,j} = \sup_{s_1,s_2,l_1,l_2,\history_0^{i-1}} \eta_{i,j}\pow{m}. \numberthis\label{eq:db-eq4}
\end{align*}
Thus is is enough to find an upper bound for $\eta_{i,j}\pow{m}$. 
For the sake of simplicity, consider the case when $m=2$. Then, 
\begin{align*}
    \eta_{1,2}\pow{2} & =\bigg\|\probl\lp(X_2,a_2)|X_1=s_1,a_1=l_1,X_0=s_0,a_0=l_0\rp\\
    & \qquad -\probl\lp(X_2,a_2)|X_1=s_2,a_1=l_2,X_0=s_0,a_0=l_0\rp\bigg\|_{TV}.
\end{align*}
Use Lemma \ref{lemma:tv-bound} by substituting $X=a_2$, $Y=X_2$ and $Z=(X_1,a_1,X_0,a_0)$ with $z_1$ and $z_2$ defined accordingly.
This reduces the $\eta$ mixing coefficient into a sum of the $\gamma$ mixing coefficients and $\theta$ mixing coefficients. 
To be precise, 
\begin{align*}
    \eta_{1,2}\pow{2}& \leq \sup_{s''}\| \probl\lp a_2|X_2=s'',\numberthis\label{eq:db-eq1} X_1=s_1,a_1=l_1,X_0=s_0,a_0=l_0\rp\\
    & \qquad -\probl\lp a_2|X_2=s'', X_1=s_2,a_1=l_2,X_0=s_0,a_0=l_0\rp\|_{TV}\\
    & + \|\probl\lp X_2\in \Tbb_2| X_1=s_1,a_1=l_1,X_0=s_0,a_0=l_0\rp\\
    & \qquad -\probl\lp X_2\in \Tbb_2| X_1=s_2,a_1=l_2,X_0=s_0,a_0=l_0\rp\|_{TV}.
\end{align*}
We can decompose the first term in \cref{eq:db-eq1} as,
\begin{align*}
    &\sup_{s''}\| \probl\lp a_2|X_2=s'', X_1=s_1,a_1=l_1,X_0=s_0,a_0=l_0\rp\\
    & \qquad -\probl\lp a_2|X_2=s'', X_1=s_2,a_1=l_2,X_0=s_0,a_0=l_0\rp\|_{TV}\\
    & \ \leq   \sup_{s''}\| \probl\lp a_2|X_2=s'' X_1=s_1,a_1=l_1,X_0=s_0,a_0=l_0\rp- \probl\lp a_2|X_2=s''\rp\|_{TV}\\
    & \qquad + \sup_{s''}\|\probl\lp a_2|X_2=s'', X_1=s_2,a_1=l_2,X_0=s_0,a_0=l_0\rp-\probl\lp a_2|X_2=s''\rp\|_{TV}\\
    & \ \leq 2\gamma_{2,1,1}\numberthis\label{eq:db-eq3},
\end{align*}
where the last step follows from taking supremum over $s_1,l_1,s_2,l_2,s_0$,and $a_0$ on both sides of the inequality. 
It follows that, when $m=2$,
\[
 \bar\eta_{1,2}\leq \gamma_{2,1,1}+\bar\theta_{1,2}.
\]
We now establish the following recursion,
\begin{align*}
 \eta_{i,j}^{(m+1)} \leq \eta_{i,j}^{(m)}+2\gamma_{m+1,j,i}.\numberthis\label{eq:db-eq5}    
\end{align*}

Fix an arbitrary $m\geq 3$ and fix $i,j\in{1,2,\dots,m}$ such that $i<j$ and
decompose $\eta_{i,j}\pow{m+1}$ using Lemma \ref{lemma:tv-bound} by taking $X=(X_{m+1},a_{m+1})$, and defining $Y$ and $Z$ accordingly to get,
\begin{align*}
    \eta_{i,j}^{(m+1)}& \leq \eta_{i,j}^{(m)}\\
    & \qquad + \|\probl(X_{m+1},a_{m+1}|X_{m}=x_{m},a_{m}=l_{m},X_i=s_1,a_i=l_1,\dots)\\
    & \qquad \qquad -\probl(X_{m+1},a_{m+1}|X_{m}=x_{m},a_{m}=l_{m},X_i=s_2,a_i=l_2,\dots\|_{TV}.
\end{align*}
where we have replaced the terms common to both the conditionals by $\dots$ for convenience of notation.

Using Lemma \ref{lemma:tv-bound} again on the second term by taking $X=a_{m+1}$, $Y=X_{m+1}$, and $Z$ as all the random variables in the conditional, we get the following upper bound,
\begin{align*}
    \eta_{i,j}^{(m+1)}& \leq \eta_{i,j}^{(m)} + \sup_{x_{m+1}\in \chi}\|\probl(a_{m+1}|X_{m+1}=x_{m+1},X_{m}=x_{m},a_{m}=l_{m},X_i=s_1,a_i=l_1,\dots)\\
    & \qquad \qquad -\probl(a_{m+1}|X_{m+1}=x_{m+1},X_{m}=x_{m},a_{m}=l_{m},X_i=s_2,a_i=l_2,\dots\|_{TV}\\
    & \qquad + \|\probl(X_{m+1}|X_{m}=x_{m},a_{m}=l_{m},X_i=s_1,a_i=l_1,\dots)\\
    & \qquad \qquad -\probl(X_{m+1}|X_{m}=x_{m},a_{m}=l_{m},X_i=s_2,a_i=l_2,\dots\|_{TV}.\numberthis\label{eq:db-eq2}
\end{align*}
Since $X_{m+1}$ is Markovian conditional on $X_{m}$ and $a_{m}$, the third term vanishes. Proceeding similarly to \cref{eq:db-eq3}, the second term can be upper bounded by $\gamma_{m+1,j,i}$.
This proves the recursion in \cref{eq:db-eq5}.

Applying the recursion multiple times we are left with, 
\[
\eta_{i,j}^{(m)}\leq \sum_{p=i+j+1}^m 2\gamma_{p,j,i}+\eta_{i,j}\pow{j}.
\]
Proceeding similarly to \cref{eq:db-eq1}, we can decompose $\eta_{i,j}\pow{j}$ to get, 
\[
\eta_{i,j}^{(j)}\leq 2\gamma_{j,i,i}+\bar\theta_{i,j}.
\]
The terms in the right hand side of the previous equation is a constant. 
Following the arguments of \cref{eq:db-eq4}, we can now take an appropriate supremum on the left hand side to recover $\bar \eta_{i,j}$. 
Therefore, 
\begin{align*}
     \bar\eta_{i,j}\leq 2\sum_{p=i+j}^m\gamma_{p,j,i}+\bar\theta_{i,j}.\numberthis\label{eq:lem-dbeq1}   
\end{align*}
By substituting the upper bound into the expression of $\deltam$, it follows that, 
\[
    \deltam \leq \sup_{1\leq i\leq m}\lc 1+2\sum_{j>i}\sum_{p=i+j}^m\gamma_{p,j,i}+\sum_{j>i}\bar\theta_{i,j}\rc,
\]
From Assumptions \ref{assume:control-mixing} and \ref{assume:chain-mix} it now follows that,
\[
     \deltam \leq 1+\constant+\constant_{\theta}, 
\]
which completes the proof.
\end{proof}
{
\subsection{Proof of Lemma \ref{lemma:eta-bound}}
\begin{proof}\ 
The proof of this Lemma follows similarly to that of Lemma \ref{lemma:delta-bound}. We begin by observing from \cref{eq:lem-dbeq1} that
\begin{align*}
    \bar\eta_{i,j}\leq 2\sum_{p=i+j}^m\gamma_{p,j,i}+\bar\theta_{i,j}.
\end{align*}
The rest of the proof follows from a direct application of Assumptions \ref{assume:control-exp} and \ref{assume:chain-exp}.
\end{proof}
}
\subsection{Proof of Lemma \ref{lemma:return-time-markov}}  
\begin{proof}\ 
As before, let $\tau_{s,l}\pow{i,\star,j}$ as the time between the $j-1$-th and $j$-th visit to control $l$ after visiting state-control pair $s,l$ for the $i$-th time.
Our next step is to represent $\tau_{s,l}\pow{i}$ in terms of $\tau_{s,l}\pow{i,\star,j}$'s. Observe that
\begin{align*}
    \tau_{s,l}\pow{i+1}=\begin{cases}
   \tau_{s,l}\pow{i,\star,1}\  \text{ if } \lc X_{\sum_{p=1}^i \tau_{s,l}\pow{p}+ \tau_{s,l}\pow{i,\star,1}}=s\rc\\
    \tau_{s,l}\pow{i,\star,1}+\tau_{s,l}\pow{i,\star,2}\  \text{ if } \lc X_{\sum_{p=1}^i \tau_{s,l}\pow{p}+ \tau_{s,l}\pow{i,\star,1}}\neq s \text{ and } \ X_{\sum_{p=1}^i \tau_{s,l}\pow{p}+ \tau_{s,l}\pow{i,\star,1}+ \tau_{s,l}\pow{i,\star,2}}=s\rc\\
    \quad \vdots \qquad \qquad \qquad \vdots
    \end{cases}
\end{align*}
Informally, $\tau_{s,l}\pow{i+1}$ is: $\tau_{s,l}\pow{i,\star,1}$ if the state at the corresponding time is $s$; it is $ \tau_{s,l}\pow{i,\star,1}+\tau_{s,l}\pow{i,\star,2}$ if the state was not $s$ after $\tau_{s,l}\pow{i,\star,1}$ time points and $s$ after $\tau_{s,l}\pow{i,\star,1}+\tau_{s,l}\pow{i,\star,2}$ time points, and so on. In other words,
\begin{small}
\begin{align*}
    \tau_{s,l}\pow{i+1} & =  \tau_{s,l}\pow{i,\star,1}\indicator\lb X_{\sum_{p=1}^i \tau_{s,l}\pow{p}+ \tau_{s,l}\pow{i,\star,1}}=s\rb\\
    & \ + \lp\tau_{s,l}\pow{i,\star,1}+\tau_{s,l}\pow{i,\star,2}\rp\indicator\lb X_{\sum_{p=1}^i \tau_{s,l}\pow{p}+ \tau_{s,l}\pow{i,\star,1}}\neq s,\ X_{\sum_{p=1}^i \tau_{s,l}\pow{p}+ \tau_{s,l}\pow{i,\star,1}+ \tau_{s,l}\pow{i,\star,2}}=s\rb\\
    & \ + \dots
\intertext{which in turn implies}
    & \expec[\tau_{s,l}\pow{i+1}|\Fcal_{\sum_{p=1}^{i-1} \tau_{s,l}\pow{p}}] \\
    & \ =  \expec\lb\tau_{s,l}\pow{i,\star,1}\indicator\lb X_{\sum_{p=1}^i \tau_{s,l}\pow{p}+ \tau_{s,l}\pow{i,\star,1}}=s\rb|\Fcal_{\sum_{p=1}^{i-1} \tau_{s,l}\pow{p}}\rb\\
    & \quad + \expec\lb\lp\tau_{s,l}\pow{i,\star,1}+\tau_{s,l}\pow{i,\star,2}\rp\indicator\lb X_{\sum_{p=1}^i \tau_{s,l}\pow{p}+ \tau_{s,l}\pow{i,\star,1}}\neq s,\ X_{\sum_{p=1}^i \tau_{s,l}\pow{p}+ \tau_{s,l}\pow{i,\star,1}+ \tau_{s,l}\pow{i,\star,2}}=s\rb|\Fcal_{\sum_{p=1}^{i-1} \tau_{s,l}\pow{p}}\rb\\
    & \quad + \dots\\
    & = \text{Term 1}+\text{Term 2}+\dots\numberthis\label{eq:ret-t-markeq4}
\end{align*}
\end{small}
To compute an upper bound to $\expec[\tau_{s,l}\pow{i}]$, it is thus sufficient to individually find an upper bound to each term of the summation in the right-hand side of the previous equation by a careful bookkeeping of the conditional expectations.

\noindent\textbf{Term 1: }Applying the law of conditional expectation to the first term we get
\begin{align*}
    & \expec\lb\tau_{s,l}\pow{i,\star,1}\indicator\lb X_{\sum_{p=1}^i \tau_{s,l}\pow{p}+ \tau_{s,l}\pow{i,\star,1}}=s\rb|\Fcal_{\sum_{p=1}^{i-1} \tau_{s,l}\pow{p}}\rb \\
    &\  =  \expec\lb \expec\lb\tau_{s,l}\pow{i,\star,1}\indicator\lb X_{\sum_{p=1}^i \tau_{s,l}\pow{p}+ \tau_{s,l}\pow{i,\star,1}}=s\rb\gn\tau_{s,l}\pow{i,\star,1} \rb|\Fcal_{\sum_{p=1}^{i-1} \tau_{s,l}\pow{p}} \rb\\
    &\ = \expec\lb\tau_{s,l}\pow{i,\star,1}\prob\lp X_{\sum_{p=1}^i \tau_{s,l}\pow{p}+ \tau_{s,l}\pow{i,\star,1}}=s| \tau_{s,l}\pow{i,\star,1} \rp|\Fcal_{\sum_{p=1}^{i-1} \tau_{s,l}\pow{p}}\rb \numberthis\label{eq:ret-t-markeq2}
\end{align*}
Recall from \cref{eq:markov-eq1} our assumption that 
\[
 \max_{s,t,l} M_{s,t}^{(l)}=M_{max}, \text{ and } \min_{s,t,l} M_{s,t}^{(l)}=M_{min}
\]
for two numbers $0<M_{min},M_{max}<1$. It follows that for any time $p$, state $s$, and history $\history_0^{p-1}$,
\begin{align*}
    &M_{min}\leq \prob\lp X_p=s\gn \History_0^{p-1}=\history_0^{p-1}\rp\leq M_{max}, \, \, \text{and } & (E1)\\
    & \prob\lp X_p\neq s\gn \History_0^{p-1}=\history_0^{p-1}\rp\leq \max\lc  M_{max},1-M_{min}\rc=:M_{opt}. & (E2)
\end{align*}
It follows from $(E1)$ that, $\prob\lp X_{\sum_{p=1}^i \tau_{s,l}\pow{p}+ \tau_{s,l}\pow{i,\star,1}}=s| \tau_{s,l}\pow{i,\star,1} \rp\leq M_{max}$. Substituting this value in the right hand side of \cref{eq:ret-t-markeq2}, we get the following upper bound to Term 1
\[
 \expec\lb\tau_{s,l}\pow{i,\star,1}\indicator\lb X_{\sum_{p=1}^i \tau_{s,l}\pow{p}+ \tau_{s,l}\pow{i,\star,1}}=s\rb|\Fcal_{\sum_{p=1}^{i-1} \tau_{s,l}\pow{p}}\rb\leq \expec\lb \tau_{s,l}\pow{i,\star,1}M_{max}|\Fcal_{\sum_{p=1}^{i-1} \tau_{s,l}\pow{p}}\rb\leq T_\star M_{max},
\]
where the last inequality follows from tower property since \[\Fcal_{\sum_{p=1}^{i-1} \tau_{s,l}\pow{p}}\subseteq \Fcal_{\sum_{p=1}^{i-1} \tau_{s,l}\pow{p}+\sum_{p=1}^{j-1}\tau_{s,l}\pow{i,\star,p}}.\]
\textbf{Term 2: }For ease of notation, we define 
\[
\expec^*[\cdot]=\expec[\cdot|\Fcal_{\sum_{p=1}^{i-1} \tau_{s,l}\pow{p}}]
\]
and proceed similarly as before to get
\begin{small}
\begin{align*}
    &\expec^*\lb\lp\tau_{s,l}\pow{i,\star,1}\rp\indicator\lb X_{\sum_{p=1}^i \tau_{s,l}\pow{p}+ \tau_{s,l}\pow{i,\star,1}}\neq s,\ X_{\sum_{p=1}^i \tau_{s,l}\pow{p}+ \tau_{s,l}\pow{i,\star,1}+ \tau_{s,l}\pow{i,\star,2}}=s\rb\rb\\ 
    & = \expec^*\lb\lp\tau_{s,l}\pow{i,\star,1}\rp\prob\lp X_{\sum_{p=1}^i \tau_{s,l}\pow{p}+ \tau_{s,l}\pow{i,\star,1}}\neq s,\ X_{\sum_{p=1}^i \tau_{s,l}\pow{p}+ \tau_{s,l}\pow{i,\star,1}+ \tau_{s,l}\pow{i,\star,2}}=s|\tau_{s,l}\pow{i,\star,1},\tau_{s,l}\pow{i,\star,2}\rp\rb.\numberthis\label{eq:ret-t-markeq3}
\end{align*}
\end{small}
We decompose $\prob\lp X_{\sum_{p=1}^i \tau_{s,l}\pow{p}+ \tau_{s,l}\pow{i,\star,1}}\neq s,\ X_{\sum_{p=1}^i \tau_{s,l}\pow{p}+ \tau_{s,l}\pow{i,\star,1}+ \tau_{s,l}\pow{i,\star,2}}=s|\tau_{s,l}\pow{i,\star,1},\tau_{s,l}\pow{i,\star,2}\rp$ into 
\begin{align*}
& \prob\lp X_{\sum_{p=1}^i \tau_{s,l}\pow{p}+ \tau_{s,l}\pow{i,\star,1}+ \tau_{s,l}\pow{i,\star,2}}=s|\tau_{s,l}\pow{i,\star,1},\tau_{s,l}\pow{i,\star,2}, X_{\sum_{p=1}^i \tau_{s,l}\pow{p}+ \tau_{s,l}\pow{i,\star,1}}\neq s\rp\\
 & \qquad \qquad \times \prob\lp X_{\sum_{p=1}^i \tau_{s,l}\pow{p}+ \tau_{s,l}\pow{i,\star,1}}\neq s|\tau_{s,l}\pow{i,\star,1},\tau_{s,l}\pow{i,\star,2}\rp
\end{align*}
We use $(E1)$ to bound the first term from above and $(E2)$ to bound the second term from above in the previous equation. 
This gives us,
\begin{small}
\begin{align*}
    \prob\lp X_{\sum_{p=1}^i \tau_{s,l}\pow{p}+ \tau_{s,l}\pow{i,\star,1}}\neq s,\ X_{\sum_{p=1}^i \tau_{s,l}\pow{p}+ \tau_{s,l}\pow{i,\star,1}+ \tau_{s,l}\pow{i,\star,2}}=s|\tau_{s,l}\pow{i,\star,1},\tau_{s,l}\pow{i,\star,2}\rp & \leq   M_{max}M_{opt},
\end{align*}
\end{small}
Substituting this value in the right hand side of \cref{eq:ret-t-markeq3} we get
\begin{align*}
   \expec^*\lb\lp\tau_{s,l}\pow{i,\star,1}\rp\indicator\lb X_{\sum_{p=1}^i \tau_{s,l}\pow{p}+ \tau_{s,l}\pow{i,\star,1}}\neq s,\ X_{\sum_{p=1}^i \tau_{s,l}\pow{p}+ \tau_{s,l}\pow{i,\star,1}+ \tau_{s,l}\pow{i,\star,2}}=s\rb\rb & \leq \expec^*[\tau_{s,l}\pow{i,\star,1}] M_{max}M_{opt} \\
    & \leq  T_\star M_{max}M_{opt},
\end{align*}
where the last inequality follows from an application of tower property. We similarly have,
\[
   \expec^*\lb\lp\tau_{s,l}\pow{i,\star,2}\rp\indicator\lb X_{\sum_{p=1}^i \tau_{s,l}\pow{p}+ \tau_{s,l}\pow{i,\star,1}}\neq s,\ X_{\sum_{p=1}^i \tau_{s,l}\pow{p}+ \tau_{s,l}\pow{i,\star,1}+ \tau_{s,l}\pow{i,\star,2}}=s\rb\rb \leq  T_\star M_{max}M_{opt}.
\]
Thus $2T_\star M_{max}M_{opt}$ is an upper bound to Term 2. Proceeding similarly, we can find an upper bound to each term. Substituting these upper bounds back into \cref{eq:ret-t-markeq4} we get
\begin{align*}
    \expec\lb\tau_{s,l}\pow{i}|\Fcal_{\sum_{p=1}^{i-1} \tau_{s,l}\pow{p}}\rb & \leq  T_\star M_{max}+ 2T_\star M_{max}M_{opt}+\dots\\
    & = \sum_{j=1}^\infty jT_\star M_{max} M_{opt}^{j-1}\\
    & = \frac{T_\star M_{max}}{1-M_{opt}} \sum_{j=1}^\infty j(1-M_{opt})M_{opt}^{j-1}\\
    & = \frac{T_\star M_{max}}{M_{opt}(1-M_{opt})}.
\end{align*}
\end{proof}
\subsection{Proof of Lemma \ref{lemma:greedy-lem1}}

\begin{proof}\
That it is a controlled Markov chain is obvious. We only need to verify ergodicity of $\tilde\mc\pow{l}$, and stationarity of controls $a_i$. 
We first show that $\tilde\mc\pow{l}$ is irreducible. 
Since $M\pow{l}$ is irreducible and aperiodic, there exists an integer $n_{ap}$ such that $\lp M\pow{l}\rp^{n_{ap}}$ is a positive stochastic matrix. Observe that $J\times J=J$, $J\times M\pow{l}=J$. Then,
\[
\lp\mc\pow{l}\rp^2 = \mc\pow{l}\times\mc\pow{l}=\begin{bmatrix}
(1-\upsilon)^2 J & (1-\upsilon)\upsilon J\\
(1-\upsilon)\upsilon J & \upsilon^2 \lp M\pow{l}\rp^2
\end{bmatrix}.
\]
Similarly, it can be shown that $\lp\mc\pow{l}\rp^{n_{ap}}$ depends only upon $\upsilon,\lp M\pow{l}\rp^{n_{ap}}$ and $J$. Thus, $\lp\mc\pow{l}\rp^{n_{ap}}$ is positive and it is irreducible.
\[
\prob(X_2=1,\omega_2=0|X_1=1,\omega_1=0)=\frac{1-\upsilon}{d}\prob(\omega_2=\xi)>0
\]
Since $\prob(\omega_2=\xi)=\min\{\upsilon,1-\upsilon\}>0$, this proves that there exists a one-step path of positive probability from $(1,0)$ to $(1,0)$. 
Thus, state $(1,0)$ is aperiodic. 
Since the matrix is also irreducible, all the states are aperiodic, and the transition matrix is ergodic. 
Next, we verify the stationarity of the controls $a_i$. 
Let $\history_0^{i-1}\in(\chi\times \{0,1\}\times\Ibb)^i$ be any history and $(s_i,\xi_i)\in\chi\times\{0,1\}$. Then,

\begin{align*}
    \probl(\tilde a_i|X_i=s_i,\omega_i=\xi_i,\History_0^{i-1}=\history_0^{i-1}) & =\begin{cases}
    \probl(a_i|X_i=s_i,\omega_i=\xi_i,\History_0^{i-1}=\history_0^{i-1}) \text{ if } \xi_i=1\\
    \probl(D_i\pow{3}|X_i=s_i,\omega_i=\xi_i,\History_0^{i-1}=\history_0^{i-1}) \text{ if } \xi_i=0\\
    \end{cases}.
\end{align*}
However, if $\xi_i=1$, then $a_i$ is drawn uniformly over $\Ibb$ independent of the history. And if $\xi_i=0$, then we draw from $D_i\pow{3}$, which is distributed uniformly over $\Ibb$ independent of the history.
Thus the controls are trivially stationary. This completes the proof.
\end{proof}

\subsection{Proof of Lemma \ref{lemma:funcmix-coefbd}}~\label{sec:prf-fmxcfd}
\begin{proof}\ 
As before, for some $i,j$ let $T_\star\in\lc0,1\rc^{m-j}$ and $\history_\star\in\lc0,1\rc^{i+1}$, and denote by $\lp\indicator[X_i=s,a_i=l],\dots \indicator[X_0=s,a_0=l]\rp$ by $\indicator(\History_0^i)$.
Define $\Scal_{\history_\star}:=\{\history_0^i\in (\chi\times\Ibb)^{i+1}:\indicator(\history_0^i)=\history_\star\}$.
Next observe that,
\begin{align*}
    & \left|\prob\left(\indicator[X_m=s,a_m=l],\dots, \indicator[X_j=s,a_j=l])\in \mathbb{T}_\star|\indicator\lp\History_0^{i}\rp=\history_\star\right)\right.\\
    & \qquad \left.-\prob\lp\indicator[X_m=s,a_m=l],\dots, \indicator[X_j=s,a_j=l])\in \mathbb{T}_\star\rp\right|\\
    & \leq \sup_{\substack{\history_0^{i}\in \Scal_{\history_\star}}} \left|\prob\left(\indicator[X_m=s,a_m=l],\dots, \indicator[X_j=s,a_j=l])\in\right. \mathbb{T}_\star|\History_0^{i}=\history_0^{i}\right)\\
    & \quad \left.-\prob\left(\indicator[X_m=s,a_m=l],\dots, \indicator[X_j=s,a_j=l])\in \mathbb{T}_\star\right)\right|\\
    &\leq \sup_{\substack{\history_0^{i}\in (\chi\times\Ibb)^{i+1}}} \left|\prob\left(\indicator[X_m=s,a_m=l],\dots, \indicator[X_j=s,a_j=l])\in\right. \mathbb{T}_\star|\History_0^{i}=\history_0^{i}\right)\\
    & \quad \left.-\prob\left(\indicator[X_m=s,a_m=l],\dots, \indicator[X_j=s,a_j=l])\in \mathbb{T}_\star\right)\right|\\
    & \leq \sup_{\substack{\Tbb,\history_0^{i-1}}}\left|\prob\left((X_m,a_m,\dots, X_j,a_j)\in \mathbb{T}|\History_0^{i}=\history_0^{i}\right)-\prob\left(X_m,a_m,\dots,X_j,a_j\in\mathbb{T}\right)\right|,
\end{align*}
where the first inequality follows because there exists at least one $\history_0^{i}$ such that $\indicator(\history_0^{i})=\history_\star$, and the last inequality follows naturally by taking inversion of $\indicator(\cdot)$ and the appropriate supremum. 
Now taking an appropriate supremum over $\Tbb_\star,\history_\star$ we get,
\begin{align*}
    & \sup_{\substack{\Tbb_\star,\history_\star}}\left|\prob\left(\indicator[X_m=s,a_m=l],\dots, \indicator[X_j=s,a_j=l])\in \mathbb{T}_\star|\indicator\lp\History_0^{i}\rp=\history_\star\right)\right.\\
    & \qquad \left.-\prob\lp\indicator[X_m=s,a_m=l],\dots, \indicator[X_j=s,a_j=l])\in \mathbb{T}_\star\rp\right|\\
    & \leq \sup_{\substack{\Tbb,\history_0^{i-1}}}\left|\prob\left((X_m,a_m,\dots, X_j,a_j)\in \mathbb{T}|\History_0^{i}=\history_0^{i}\right)-\prob\left(X_m,a_m,\dots,X_j,a_j\in\mathbb{T}\right)\right|\\
    & =\phi_{i,j}.
\end{align*}
This proves our lemma.
\end{proof}
\subsection{Proof of Lemma \ref{lemma:tv-bound}}~\label{sec:prf-tvbd}
\begin{proof}\
Let $\Xcal$ and $\Ycal$ be two events.
Consider the following term,
\begin{small}
\begin{align*}
    & \lv\prob(X\in\Xcal,Y\in\Ycal|Z=z_1)-\prob(X\in\Xcal,Y\in\Ycal|Z=z_2)\rv\\
    & \ =\lv \sum_{y\in\Ycal}\lb \prob(X\in\Xcal,Y=y|Z=z_1)-\prob(X\in\Xcal,Y= y|Z=z_2)\rc \rv\\
    & \ =\lv \sum_{y\in\Ycal}\lb \prob(X\in\Xcal|Y= y,Z=z_1)\prob(Y=y|Z=z_1)-\prob(X\in\Xcal|Y=y,Z=z_2)\prob(Y=y|Z=z_2)\rc \rv
\end{align*}
\end{small}
Adding and subtracting $\prob(X\in\Xcal|Y=y,Z=z_1)\prob(Y=y|Z=z_2)$ inside the summation and applying triangle inequality we get,
\begin{small}
\begin{align*}
    & \lv\prob(X\in\Xcal,Y\in\Ycal|Z=z_1)-\prob(X\in\Xcal,Y\in\Ycal|Z=z_2)\rv\\
    & \leq \lv \sum_{y\in\Ycal}\lb \prob(X\in\Xcal|Y=y,Z=z_1)\prob(Y=y|Z=z_2)-\prob(X\in\Xcal|Y=y,Z=z_2)\prob(Y=y|Z=z_2)\rb \rv\\
    & \ + \left| \sum_{y\in\Ycal}\lb \prob(X\in\Xcal|Y=y,Z=z_1)\prob(Y=y|Z=z_2)-\prob(X\in\Xcal|Y=y,Z=z_1)\prob(Y=y|Z=z_1)\rb \right|\\
    & =\text{TERM 1+TERM 2}
\end{align*}
\end{small}
In the first term in the previous equation, $\prob(Y=y|Z=z_2)$ is common. So, the first term can be rewritten as,
\begin{small}
\begin{align*}
    \text{TERM 1} & = \lv \sum_{y\in\Ycal}\bigg(\lb \prob(X\in\Xcal|Y=y,Z=z_1)-\prob(X\in\Xcal|Y=y,Z=z_2)\rb\times \prob(Y=y|Z=z_2)\bigg) \rv
\end{align*}
\end{small}
We take the modulus inside the summation to get,
\begin{small}
\begin{align*}
 \text{TERM 1} & \leq \sum_{y\in\Ycal}\bigg(\lv \prob(X\in\Xcal|Y=y,Z=z_1)-\prob(X\in\Xcal|Y=y,Z=z_2)\rv \prob(Y=y|Z=z_2)\bigg) \\
    & \leq \sup_{y_0\in \Ycal} \lv \prob(X\in\Xcal|Y=y_0,Z=z_1)-\prob(X\in\Xcal|Y=y_0,Z=z_2)\rv \sum_{y\in\Ycal}\prob(Y=y|Z=z_2)\\
    & = \sup_{y_0\in \Ycal} \lv \prob(X\in\Xcal|Y=y_0,Z=z_1)-\prob(X\in\Xcal|Y=y_0,Z=z_2)\rv \prob(Y=\Ycal|Z=z_2)\\
    & \leq  \sup_{y_0\in \Ycal} \lv \prob(X\in\Xcal|Y=y_0,Z=z_1)-\prob(X\in\Xcal|Y=y_0,Z=z_2)\rv.
\end{align*}
\end{small}
We upper bound $\lv\prob(X\in\Xcal|Y=y_0,Z=z_1)-\prob(X\in\Xcal|Y=y_0,Z=z_2)\rv$ by its corresponding total variation distance as,
\[
\sup_{y_0\in \Ycal} \lV \probl(X|Y=y_0,Z=z_1)-\probl(X|Y=y_0,Z=z_2)\rV_{TV}.
\]
Since $\Ycal\subseteq\Omega$ , we can further upper bound the previous term by 
\begin{align*}
    \sup_{y_0\in \Omega} \lV \probl(X|Y=y_0,Z=z_1)-\probl(X|Y=y_0,Z=z_2)\rV_{TV}.
\end{align*}
This gives us the following upper bound to the first term,
\begin{align}\label{eq:tv-eq1}
    \text{TERM 1}\leq \sup_{y_0\in \Omega} \lV \probl(X|Y=y_0,Z=z_1)-\probl(X|Y=y_0,Z=z_2)\rV_{TV}.
\end{align}
Now focusing on the second term, let $\Ycal_1\subseteq\Ycal$ be the largest set such that the probability difference $\prob\lp Y=y|Z=z_1\rp-\prob\lp Y=y|Z=z_2\rp>0$ for any $y\in\Ycal_1$. Let $\Ycal_2$ be the set difference $\Ycal\backslash\Ycal_1$.
This yields us the following decomposition of the second term,
\begin{align*}
    \text{TERM 2} & = \left| \sum_{y\in\Ycal}\prob(X\in\Xcal|Y=y,Z=z_1)\lb \prob(Y=y|Z=z_1)-\prob(Y=y|Z=z_2)\rb \right| \\
    & =\Bigg| \sum_{y\in\Ycal_1}\prob(X\in\Xcal|Y=y,Z=z_1)\lb \prob(Y=y|Z=z_1)-\prob(Y=y|Z=z_2)\rb\\
    & \ + \sum_{y\in\Ycal_2}\prob(X\in\Xcal|Y=y,Z=z_1)\lb \prob(Y=y|Z=z_1)-\prob(Y=y|Z=z_2)\rb \Bigg|, 
\end{align*}
which can be further upper bounded by,
\begin{align*}
    & \Bigg|\lp\sup_{y_1\in\Ycal_1}\prob(X=\Xcal|Y=y_1,Z=z_1)\rp \sum_{y_1\in\Ycal_1} \lp\prob(Y=y_1|Z=z_1)-\prob(Y=y_1|Z=z_2)\rp\\
    & \ +  \lp\inf_{y_2\in\Ycal_2}\prob(X=\Xcal|Y=y_2,Z=z_1)\rp \sum_{y_2\in\Ycal_2}\lp\prob(Y=y_2|Z=z_1)-\prob(Y=y_2|Z=z_2)\rp\Bigg|\\
    & = \Bigg|\lp\sup_{y_1\in\Ycal_1}\prob(X=\Xcal|Y=y_1,Z=z_1)\rp  \lp\prob(Y=\Ycal_1|Z=z_1)-\prob(Y=\Ycal_1|Z=z_2)\rp\\
    & \ +  \lp\inf_{y_2\in\Ycal_2}\prob(X=\Xcal|Y=y_2,Z=z_1)\rp\lp\prob(Y=\Ycal_2|Z=z_1)-\prob(Y=\Ycal_2|Z=z_2)\rp\Bigg|.
     \intertext{Recall that if $\mathcal{P}_1$ and $\mathcal{P}_2$ are real numbers with different signs, $|\mathcal{P}_1+\mathcal{P}_2|\leq \max\{|\mathcal{P}_1|,|\mathcal{P}_2|\}$. Using this fact, we can get the following upper bound to the previous equation.}
    &\max\Bigg\{\lv\lp\sup_{y_1\in\Ycal_1}\prob(X\in\Xcal|Y=y_1,Z=z_1)\rp\lp\prob(Y\in\Ycal_1|Z=z_1)-\prob(Y\in\Ycal_1|Z=z_2)\rp\rv,\\
    & \ \qquad \lv\lp\inf_{y_2\in\Ycal_2}\prob(X=\Xcal|Y=y_2,Z=z_1)\rp\lp\prob(Y=\Ycal_2|Z=z_1)-\prob(Y=\Ycal_2|Z=z_2)\rp\rv\Bigg\}.
\end{align*}
We can upper bound the common probabilities in the previous terms by $1$. This gives us the following upper bound to the previous term.
\begin{small}
\begin{align*}
    \max\bigg\{\lv\prob(Y=\Ycal_1|Z=z_1)-\prob(Y=\Ycal_1|Z=z_2)\rv,\lv\prob(Y=\Ycal_2|Z=z_1)-\prob(Y=\Ycal_2|Z=z_2)\rv\bigg\},
\end{align*}
\end{small}
Since $\lV\probl(Y|Z=z_1)-\probl(Y|Z=z_2)\rV_{TV}$ is an upper bound to both of the terms inside the maximum, we find the following upper bound to the second term.
\begin{align}\label{eq:tv-eq2}
    \text{TERM 2}\leq \lV\probl(Y|Z=z_1)-\probl(Y|Z=z_2)\rV_{TV}.
\end{align}
Combining equations \ref{eq:tv-eq1} and \ref{eq:tv-eq2} we find
\begin{align*}
    |\prob(X\in\Xcal,Y\in\Ycal|Z=z_1) & -\prob(X\in\Xcal,Y\in\Ycal|Z=z_2)|  \leq \lV\probl(Y|Z=z_1)-\probl(Y|Z=z_2)\rV_{TV}\\
    & \ + \sup_{y_0\in \Omega} \lV \probl(X|Y=y_0,Z=z_1)-\probl(X|Y=y_0,Z=z_2)\rV_{TV}.
\end{align*}
Taking supremum over $\Xcal$ and $\Ycal$,
\begin{small}
\begin{align*}
\lV\probl(X,Y|Z=z_1) -\probl(X,Y|Z=z_2)\rV_{TV}  & \leq \lV\probl(Y|Z=z_1)-\probl(Y|Z=z_2)\rV_{TV}\\
    & \ + \sup_{y_0\in \Omega} \lV \probl(X|Y=y_0,Z=z_1)-\probl(X|Y=y_0,Z=z_2)\rV_{TV}.
\end{align*}
\end{small}
This completes the proof.
\end{proof}
\subsection{Proof of Lemma \ref{lemma:markov-assume4}}~\label{sec:prf-mkvass4}
\begin{proof}\ 
Define the paired process $Y_i:=(X_i,a_i)$ on the paired state space $\chi\times \Ibb$. It follows from the definition of $\bar \theta_{i,j}$ in \cref{eq:def-theta} that,
\begin{align*}
\bar\theta_{i,j} & = \sup_{s_1,s_2\in\chi,l_1,l_2\in \Ibb}\| \probl\lp X_j|X_i=s_1,a_i=l_1\rp-\probl\lp X_j|X_i=s_2,a_i=l_2\rp \|_{TV}\\
&\leq  \sup_{s_1,s_2\in\chi,l_1,l_2\in \Ibb}\| \probl\lp X_j,a_j|X_i=s_1,a_i=l_1\rp-\probl\lp X_j,a_j|X_i=s_2,a_i=l_2\rp \|_{TV}
\end{align*}
As seen in section \ref{sec:exam-Markov}, $(X_i,a_i)$ forms an inhomogenous Markov chain with the probability of transition from $(s,l)$ to $(t,l')$ at time point $i$ is $M_{s,t}^{(l)}\times P_{t,l'}^{(i)}$. 
It follows from \citet[Theorem 2]{hajnal1958weak} that
\begin{align*}
& \sup_{s_1,s_2\in\chi,l_1,l_2\in \Ibb}\| \probl\lp X_j,a_j|X_i=s_1,a_i=l_1\rp-\probl\lp X_j,a_j|X_i=s_2,a_i=l_2\rp \|_{TV}\\
& \qquad \leq \prod_{p=i}^{j-1}\lp 1-\min_{(s_1,l_1),(s_2,l_2)\in\chi\times\Ibb}\sum_{(t,l')\in \chi\times\Ibb} \min\lc \lp M_{s_1,t}^{(l_1)}\times P_{t,l'}\pow{i}\rp,\lp M_{s_2,t}^{(l_2)}\times P_{t,l'}\pow{i}\rp \rc\rp.\numberthis\label{eq:mark-asseq1}
\end{align*}
Recall that by hypothesis
\[
\min_{s\in\chi,l\in\Ibb} M_{s,t}\pow{l}>M_{min},
\]
for any $t\in\chi_0$.
This implies that for all $t\in\chi_0$,
\[
\min\lc \lp M_{s_1,t}^{(l_1)}\times P_{t,l'}\pow{i}\rp,\lp M_{s_2,t}^{(l_2)}\times P_{t,l'}\pow{i}\rp \rc\geq M_{min} P_{t,l'}\pow{i}.
\]
Decomposing the summation over $(t,l)\in\chi\times\Ibb$ in \cref{eq:mark-asseq1} into a summation over $(t,l)\in(\chi\backslash\chi_0)\times\Ibb$ and $(t,l)\in \chi_0\times\Ibb$ and substituting $ M_{min} P_{t,l'}\pow{i}$ as the appropriate lower bound we get,
\begin{align*}
    \sum_{(t,l')\in \chi\times\Ibb} & \min\lc \lp M_{s_1,t}^{(l_1)}\times P_{t,l'}\pow{i}\rp,\lp M_{s_2,t}^{(l_2)}\times P_{t,l'}\pow{i}\rp \rc \\
    & \geq \sum_{t\in\chi\backslash\chi_0} \sum_{l'\in \Ibb}\min\lc \lp M_{s_1,t}^{(l_1)}\times P_{t,l'}\pow{i}\rp,\lp M_{s_2,t}^{(l_2)}\times P_{t,l'}\pow{i}\rp \rc+\sum_{(t,l')\in\chi_0\times \Ibb}  M_{min} P_{t,l'}\pow{i}\\
    & \geq \sum_{t\in\chi_0}\sum_{l'\in \Ibb}  M_{min} P_{t,l'}\pow{i}\\
    & = \sum_{t\in\chi_0}M_{min}\sum_{l'\in \Ibb}P_{t,l'}\pow{i} \\
    & = |\chi_0|M_{min}.
\end{align*}
It follows that
\begin{align*}
    & \prod_{p=i}^{j-1}\lp 1-\min_{(s_1,l_1),(s_2,l_2)\in\chi\times\Ibb}\sum_{(t,l')\in \chi\times\Ibb} \min\lc \lp M_{s_1,t}^{(l_1)}\times P_{t,l'}\pow{i}\rp,\lp M_{s_2,t}^{(l_2)}\times P_{t,l'}\pow{i}\rp \rc\rp \\
    &\ \leq  \prod_{p=i}^{j-1}\lp 1-|\chi_0|M_{min}\rp\\
    &\ = \lp1-|\chi_0|M_{min}\rp^{j-i-1}.
\end{align*}
Thus we prove that,
\[
\bar\theta_{i,j} \leq \lp1-|\chi_0|M_{min}\rp^{j-i-1}.
\]
Therefore, 
\[
\sum_{j=i+1}^{\infty}\bar\theta_{i,j}\leq \frac{1}{|\chi_0|M_{min}}.
\]
Choosing $\constant_\theta=1/(|\chi_0|M_{min})$ now completes the proof.
\end{proof}
\subsection{Proof of Proposition \ref{prop:mod-sampling-scheme}}~\label{prf:sampling-scheme}
\begin{proof}\ 
We prove this fact by induction. 
Obviously, $(X_0,a_0)\overset{d}{=}(\tilde X_0,\tilde{a_0})$.
Now, for some $i\geq 1$, let $\lp X_0,a_0,\dots,X_i,a_i\rp$ be identically distributed to $\lp \tilde X_0,\tilde a_0,\dots,\tilde X_i,\tilde a_i\rp$. Then, for $i+1$, we note that,
\begin{small}
\begin{align*}
    & \prob\lp \tilde X_{i+1}=s_{i+1},\tilde a_{i+1}=l_{i+1},\dots,\tilde X_0=s_0,\tilde a_0=l_0 \rp\\
    & \ =\prob\lp \tilde a_{i+1}=l_{i+1}| \tilde X_{i+1}=s_{i+1},\dots,\tilde X_0=s_0\rp\\ 
    &\quad \times\prob\lp \tilde X_{i+1}=s_{i+1}|\tilde X_i=s_i,\tilde a_{i}=l_i,\dots,\tilde a_0=l_0,\tilde X_0=s_0\rp\\
    &\quad \times \prob(\tilde X_i=s_i,\tilde a_i=l_i,\dots,\tilde X_0=s_0,\tilde a_0=l_0)\\
    & \ =\prob\lp  \alpha_i\pow{\tilde X_0,\tilde{a_0},\dots,\tilde X_{i+1}}=l_{i+1}| \tilde X_{i+1}=s_{i+1},\dots,\tilde X_0=s_0\rp\\
    &\quad \times \prob\lp X_{\Tilde{X_i},\Tilde{N}_{\Tilde{X_i}}^{(i,\tilde a_i)}+1}\pow{\tilde a_i}=s_{i+1}|\tilde X_i=s_i,\tilde a_{i}=l_i,\dots,\tilde a_0=l_0,\tilde X_0=s_0\rp\\
    &\quad \times \prob(\tilde X_i=s_i,\tilde a_i=l_i,\dots,\tilde X_0=s_0,\tilde a_0=l_0)\\
    & \ =\prob\lp  \alpha_i\pow{s_0,l_0,\dots,s_{i+1}}=l_{i+1}| \tilde X_{i+1}=s_{i+1},\dots,\tilde X_0=s_0\rp\\
    &\quad \times \prob\lp X_{s_i,\Tilde{N}_{s_i}^{(i,l_i)}+1}\pow{l_i}=s_{i+1}|\tilde X_i=s_i,\tilde a_{i}=l_i,\dots,\tilde a_0=l_0,\tilde X_0=s_0\rp\\
    &\quad \times \prob(\tilde X_i=s_i,\tilde a_i=l_i,\dots,\tilde X_0=s_0,\tilde a_0=l_0),
\end{align*}
\end{small}
where the equalities follow by substituting in the corresponding value of each quantity. Observe that under the given conditional, such that $\Tilde{N}_{s_i}^{(i,l_i)}+1$ is some fixed integer $n$. 
Then, the right-hand side of the previous equation can be further decomposed into,
\begin{small}
\begin{align*}
    & \ \prob\lp  \alpha_i\pow{s_0,l_0,\dots,s_{i+1}}=l_{i+1}| \tilde X_{i+1}=s_{i+1},\dots,\tilde X_0=s_0\rp\\
    &\quad \times \prob\lp X_{s_i,\Tilde{N}_{s_i}^{(i,l_i)}+1}\pow{l_i}=s_{i+1}|\tilde X_i=s_i,\tilde a_{i}=l_i,\dots,\tilde a_0=l_0,\tilde X_0=s_0\rp\\
    &\quad \times \prob(\tilde X_i=s_i,\tilde a_i=l_i,\dots,\tilde X_0=s_0,\tilde a_0=l_0)\\
    &  \ =\prob\lp \alpha_i\pow{s_0,a_0,\dots,s_{i+1}}=l_{i+1}| \tilde X_{i+1}=s_{i+1},\dots,\tilde X_0=s_0\rp\\
    &\quad \times \prob\lp X_{s_i,n}\pow{l_i}=s_{i+1}|\tilde X_i=s_i,\tilde a_{i}=l_i,\dots,\tilde a_0=l_0,\tilde X_0=s_0\rp\\
    &\quad \times \prob(\tilde X_i=s_i,\tilde a_i=l_i,\dots,\tilde X_0=s_0,\tilde a_0=l_0)\\
    &  \ =\prob\lp \alpha_i\pow{s_0,a_0,\dots,s_{i+1}}=l_{i+1}\rp\times \prob\lp X_{s_i,n}\pow{l_i}=s_{i+1}\rp\\
    &\quad \times \prob(\tilde X_i=s_i,\tilde a_i=l_i,\dots,\tilde X_0=s_0,\tilde a_0=l_0)\numberthis\\
    &\ = P_l\pow{s_i,l_{i-1},\dots,l_0,s_0} M_{s,t}\pow{l_i} \times \prob(\tilde X_i=s_i,\tilde a_i=l_i,\dots,\tilde X_0=s_0,\tilde a_0=l_0)\\
    &\ = P_l\pow{s_i,l_{i-1},\dots,l_0,s_0} M_{s,t}\pow{l_i} \times \prob(X_i=s_i, a_i=l_i,\dots,X_0=s_0, a_0=l_0),
\end{align*}
\end{small}
where the last equality follows by induction hypothesis. It follows easily from the last equality that 
\begin{align*}
   &\prob\lp \tilde X_{i+1}=s_{i+1},\tilde a_{i+1}=l_{i+1},\dots,\tilde X_0=s_0,\tilde a_0=l_0 \rp \\ 
   &\ = \prob\lp  X_{i+1}=s_{i+1}, a_{i+1}=l_{i+1},\dots, X_0=s_0, a_0=l_0 \rp. 
\end{align*}
This completes the proof.
\end{proof}

\subsection{Proof of Lemma \ref{lemma:cover-time}}~\label{sec:prf-cvrtm}
\begin{proof}\ 
We introduce the notation $\chi'$ to denote $\lc(1,1),\dots,(d/3,1),(2,1),\dots,(d/3,k) \rc$. Observe that $ \Tbb$ can be written as,
\begin{align*}
    \Tbb:=\sum_{\Upsilon=0}^{dk/3-1} U_\Upsilon\numberthis\label{eq:lem-covt1}
\end{align*}
where $U_\Upsilon$ is the time spent between the $\Upsilon$-th and the $\Upsilon+1$-th unique state-control pair visited in $\chi'$.
Next, we observe two facts. Firstly, observe that for any element $(t,l')$ belonging to $\chi'$ we have
\begin{align*}
    \prob\lp X_i=t,a_i=l'|X_{i-1}=s,a_{i-1}=l \rp= \prob\lp X_i=t,a_i=l' \rp=\frac{3\iota}{dk}
\end{align*}
independent of any $(s,l)\in\chi\times\Ibb$. 
Secondly, observe that the probability of visiting a new state-control pair in $\chi_\Ibb$ when $\Upsilon$ unique states have already been visited is ${3\iota\lp{dk}/{3}-\Upsilon\rp}/dk$.
Together, these facts imply that
\begin{align*}
    U_\Upsilon\overset{d}{=}X_\Upsilon \text{ where } X_\Upsilon\sim Geometric\lp\lp\frac{dk}{3}-\Upsilon\rp\frac{3\iota}{dk} \rp.\numberthis~\label{eq:lem-covt2}
\end{align*}
It follows from \cref{eq:lem-covt2} that, 
\begin{align*}
    \expec[\Tbb] & = \lp \frac{dk}{3\iota}\sum_{ \Upsilon=0}^{dk/3-1}\frac{1}{dk/3-\Upsilon}\rp\\
    \intertext{where we have dropped the superscript $l$ from $\Upsilon\pow{l}$ for convenience. Rewriting the previous equation we get, }
     \expec[\Tbb] & = \frac{dk}{3\iota}\sum_{ \Upsilon=1}^{dk/3}\frac{1}{\Upsilon}\\
    & > \frac{dk}{3\iota}\log\lp dk/3+1\rp~\numberthis\label{eq:lem-covt4}.
\end{align*}
where the last inequality follows from the Euler-Maclaurin (see for example, \citet{apostol1999elementary}) approximation of a sum by its integral.
We also observe that,
\begin{align*}
    \Var(U_\Upsilon)=\frac{d^2k^2}{9\iota^2}\lp\frac{dk}{3}-\Upsilon\rp^{-2}\lb1-\lp\frac{dk}{3}-\Upsilon\rp\frac{3\iota}{dk} \rb.
\end{align*}
The term inside the square brackets is a probability, and can be upper bounded by $1$. 
Observe that when $\Upsilon\leq dk/3-1$ we can upper bound $\Var(\Tbb)$ as
\begin{align*}
    \Var( \Tbb) & \leq \sum_{\Upsilon=0}^{dk/3-1} \frac{d^2k^2}{9\iota^2}\lp\frac{dk}{3}-\Upsilon\rp^{-2}\\
    & = \sum_{\Upsilon=1}^{dk/3} \frac{d^2k^2}{9\iota^2}\frac{1}{\Upsilon^2}\\
    & < \frac{d^2k^2}{9\iota^2}\frac{\pi^2}{6}\\
    & < \frac{d^2k^2}{9\iota^2}\frac{\pi^2}{4}.~\numberthis\label{eq:lem-covt3}
\end{align*}
where the second inequality follows from the fact that $\sum_{\Upsilon\geq 1}1/\Upsilon^2=\pi^2/6$. 
Using Cantelli's inequality \citep[Equation 5]{ghosh2002probability}, we obtain, for all $0<\theta<{\expec[ \Tbb]}/{\sqrt{\Var( \Tbb)}}$,
\[
\prob\lp \Tbb>  \frac{dk}{3\iota}\log\lp \frac{dk}{3}+1\rp -\theta \frac{dk}{3\iota}\frac{\pi}{2}\rp\geq \frac{\theta^2}{1+\theta^2}.
\]
From the equations \ref{eq:lem-covt4} and \ref{eq:lem-covt3}, we get that ${\expec[ \Tbb]}/\lp{\sqrt{\Var( \Tbb)}}\rp>\lp\log(dk/3)+1 \rp/\pi$. Substituting $\theta={(\log(dk/3)+1)/\pi}$ we get
\begin{align*}
    \prob\lp\Tbb>  \frac{dk}{6\iota}\lp\log\lp\frac{dk}{3}\rp+1\rp \rp\geq \frac{1}{1+\lp\frac{\pi}{\log(dk/3)+1}\rp^2}> \frac{1}{1+\pi^2}.
\end{align*}
This proves the lemma.

\end{proof}

\end{APPENDICES}

\printbibliography 

\end{document}